\documentclass{article} %

\newcommand{\loose}{\looseness=-1}

\usepackage[utf8]{inputenc} %
\usepackage[T1]{fontenc}    %
\usepackage{url}            %
\usepackage{booktabs}       %
\usepackage{amsfonts}       %
\usepackage{nicefrac}       %
\usepackage{microtype}      %

\usepackage{tocloft}            %

\usepackage{enumitem}

\usepackage{breakcites}

\usepackage{etoolbox}
\usepackage{comment}
\newcommand{\arxiv}[1]{\iftoggle{arxiv}{#1}{}}
\newcommand{\colt}[1]{\iftoggle{colt}{#1}{}}
\newtoggle{colt}
\global\togglefalse{colt}
\newtoggle{arxiv}
\global\toggletrue{arxiv}

\colt{
  \usepackage{times}
}

\usepackage{mathrsfs}

\usepackage{algorithm}
\usepackage{verbatim}
\usepackage[noend]{algpseudocode}
\newcommand{\multiline}[1]{\parbox[t]{\dimexpr\linewidth-\algorithmicindent}{#1}}

\usepackage{multicol}

\usepackage{colortbl}

\usepackage{setspace}

\usepackage{transparent}

\usepackage{inconsolata}
\usepackage[scaled=.90]{helvet}
\usepackage{xspace}

\usepackage{pifont}

\colt{
  \PassOptionsToPackage{dvipsnames}{xcolor}
  }

\arxiv{
\usepackage[letterpaper, left=1in, right=1in, top=1in, bottom=1in]{geometry}
\PassOptionsToPackage{hypertexnames=false}{hyperref}  %
\usepackage{parskip}
\usepackage[dvipsnames]{xcolor}
\usepackage[colorlinks=true, linkcolor=blue!70!black,
citecolor=blue!70!black,urlcolor=black,breaklinks=true]{hyperref}
}

\usepackage{microtype}
\usepackage{hhline}

\makeatletter
\newcommand{\neutralize}[1]{\expandafter\let\csname c@#1\endcsname\count@}
\makeatother

\usepackage{algorithm}

\usepackage{natbib}
\arxiv{\bibliographystyle{plainnat}
\bibpunct{(}{)}{;}{a}{,}{,}}

\usepackage{amsthm}
\usepackage{mathtools}
\usepackage{amsmath}
\usepackage{bbm}
\usepackage{amsfonts}
\usepackage{amssymb}

\usepackage{xpatch}

\usepackage{thmtools}
\usepackage{thm-restate}
\declaretheorem[name=Theorem,parent=section]{theorem}
\declaretheorem[name=Lemma,parent=section]{lemma}
\declaretheorem[name=Assumption, parent=section]{assumption}
\declaretheorem[name=Condition, parent=section]{condition}

\declaretheorem[qed=$\triangleleft$,name=Remark,parent=section]{remark}
\declaretheorem[name=Proposition, parent=section]{proposition}

\declaretheorem[name=Corollary, parent=section]{corollary}

\makeatletter
  \renewenvironment{proof}[1][Proof]%
  {%
   \par\noindent{\bfseries\upshape {#1.}\ }%
  }%
  {\qed\newline}
  \makeatother

\theoremstyle{definition}  %

\theoremstyle{plain}
\newtheorem{definition}{Definition}[section]

\xpatchcmd{\proof}{\itshape}{\normalfont\proofnameformat}{}{}
\newcommand{\proofnameformat}{\bfseries}

\usepackage[nameinlink,capitalize]{cleveref}

\newcommand{\pref}[1]{\cref{#1}}
\newcommand{\pfref}[1]{Proof of \pref{#1}}
\newcommand{\savehyperref}[2]{\texorpdfstring{\hyperref[#1]{#2}}{#2}}

\crefformat{equation}{#2Eq. (#1)#3}
\Crefformat{equation}{#2Eq. (#1)#3}

\Crefformat{figure}{#2Figure #1#3}
\Crefname{assumption}{Assumption}{Assumptions}
\Crefformat{assumption}{#2Assumption #1#3}

\crefformat{subsubsection}{#2Appendix #1#3}

\usepackage{crossreftools}
\newcommand{\creftitle}[1]{\crtcref{#1}}

\usepackage{xparse}

\ExplSyntaxOn
\DeclareDocumentCommand{\XDeclarePairedDelimiter}{mm}
 {
  \__egreg_delimiter_clear_keys: %
  \keys_set:nn { egreg/delimiters } { #2 }
  \use:x %
   {
    \exp_not:n {\NewDocumentCommand{#1}{sO{}m} }
     {
      \exp_not:n { \IfBooleanTF{##1} }
       {
        \exp_not:N \egreg_paired_delimiter_expand:nnnn
         { \exp_not:V \l_egreg_delimiter_left_tl }
         { \exp_not:V \l_egreg_delimiter_right_tl }
         { \exp_not:n { ##3 } }
         { \exp_not:V \l_egreg_delimiter_subscript_tl }
       }
       {
        \exp_not:N \egreg_paired_delimiter_fixed:nnnnn 
         { \exp_not:n { ##2 } }
         { \exp_not:V \l_egreg_delimiter_left_tl }
         { \exp_not:V \l_egreg_delimiter_right_tl }
         { \exp_not:n { ##3 } }
         { \exp_not:V \l_egreg_delimiter_subscript_tl }
       }
     }
   }
 }

\keys_define:nn { egreg/delimiters }
 {
  left      .tl_set:N = \l_egreg_delimiter_left_tl,
  right     .tl_set:N = \l_egreg_delimiter_right_tl,
  subscript .tl_set:N = \l_egreg_delimiter_subscript_tl,
 }

\cs_new_protected:Npn \__egreg_delimiter_clear_keys:
 {
  \keys_set:nn { egreg/delimiters } { left=.,right=.,subscript={} }
 }

\cs_new_protected:Npn \egreg_paired_delimiter_expand:nnnn #1 #2 #3 #4
 {%
  \mathopen{}
  \mathclose\c_group_begin_token
   \left#1
   #3
   \group_insert_after:N \c_group_end_token
   \right#2
   \tl_if_empty:nF {#4} { \c_math_subscript_token {#4} }
 }
\cs_new_protected:Npn \egreg_paired_delimiter_fixed:nnnnn #1 #2 #3 #4 #5
 {
  \mathopen{#1#2}#4\mathclose{#1#3}
  \tl_if_empty:nF {#5} { \c_math_subscript_token {#5} }
 }
\ExplSyntaxOff

\XDeclarePairedDelimiter{\supnorm}{
  left=\lVert,
  right=\rVert,
  subscript=\infty
  }

\DeclarePairedDelimiter{\abs}{\lvert}{\rvert} %
\DeclarePairedDelimiter{\brk}{[}{]}
\DeclarePairedDelimiter{\crl}{\{}{\}}
\DeclarePairedDelimiter{\prn}{(}{)}
\DeclarePairedDelimiter{\nrm}{\|}{\|}

\DeclarePairedDelimiter{\ceil}{\lceil}{\rceil}

\DeclareMathOperator{\En}{\mathbb{E}}

\DeclareMathOperator*{\argmin}{arg\,min} %
\DeclareMathOperator*{\argmax}{arg\,max}

\newcommand{\wt}[1]{\widetilde{#1}}
\newcommand{\wh}[1]{\widehat{#1}}
\newcommand{\wb}[1]{\widebar{#1}}

\def\ddefloop#1{\ifx\ddefloop#1\else\ddef{#1}\expandafter\ddefloop\fi}
\def\ddef#1{\expandafter\def\csname bb#1\endcsname{\ensuremath{\mathbb{#1}}}}
\ddefloop ABCDEFGHIJKLMNOPQRSTUVWXYZ\ddefloop
\def\ddefloop#1{\ifx\ddefloop#1\else\ddef{#1}\expandafter\ddefloop\fi}
\def\ddef#1{\expandafter\def\csname b#1\endcsname{\ensuremath{\mathbf{#1}}}}
\ddefloop ABCDEFGHIJKLMNOPQRSTUVWXYZ\ddefloop
\def\ddef#1{\expandafter\def\csname sf#1\endcsname{\ensuremath{\mathsf{#1}}}}
\ddefloop ABCDEFGHIJKLMNOPQRSTUVWXYZ\ddefloop
\def\ddef#1{\expandafter\def\csname c#1\endcsname{\ensuremath{\mathcal{#1}}}}
\ddefloop ABCDEFGHIJKLMNOPQRSTUVWXYZ\ddefloop
\def\ddef#1{\expandafter\def\csname h#1\endcsname{\ensuremath{\widehat{#1}}}}
\ddefloop ABCDEFGHIJKLMNOPQRSTUVWXYZ\ddefloop
\def\ddef#1{\expandafter\def\csname hc#1\endcsname{\ensuremath{\widehat{\mathcal{#1}}}}}
\ddefloop ABCDEFGHIJKLMNOPQRSTUVWXYZ\ddefloop
\def\ddef#1{\expandafter\def\csname t#1\endcsname{\ensuremath{\widetilde{#1}}}}
\ddefloop ABCDEFGHIJKLMNOPQRSTUVWXYZ\ddefloop
\def\ddef#1{\expandafter\def\csname tc#1\endcsname{\ensuremath{\widetilde{\mathcal{#1}}}}}
\ddefloop ABCDEFGHIJKLMNOPQRSTUVWXYZ\ddefloop
\def\ddefloop#1{\ifx\ddefloop#1\else\ddef{#1}\expandafter\ddefloop\fi}
\def\ddef#1{\expandafter\def\csname scr#1\endcsname{\ensuremath{\mathscr{#1}}}}
\ddefloop ABCDEFGHIJKLMNOPQRSTUVWXYZ\ddefloop

\newcommand{\ind}{\mathbbm{1}}    %

\newcommand{\veps}{\varepsilon}

\newcommand{\ldef}{\vcentcolon=}

\newcommand{\Ccov}{C_\cov}

\newcommand{\Time}{\mathsf{Time}}

\newcommand{\desclength}{\mathsf{len}}

\newcommand{\halving}{halving\xspace}

\newcommand{\mainalg}{Version Space Averaging\xspace}

\newcommand{\sMstar}{\sss{\Mstar}}

\newcommand{\Closs}{C_{\divSymbol}}

\newcommand{\framework}{OEOE\xspace}

\DeclareFontFamily{U}{mathx}{\hyphenchar\font45}
\DeclareFontShape{U}{mathx}{m}{n}{
      <5> <6> <7> <8> <9> <10>
      <10.95> <12> <14.4> <17.28> <20.74> <24.88>
      mathx10
      }{}
\DeclareSymbolFont{mathx}{U}{mathx}{m}{n}
\DeclareFontSubstitution{U}{mathx}{m}{n}
\DeclareMathAccent{\widecheck}{0}{mathx}{"71}
\DeclareMathAccent{\wideparen}{0}{mathx}{"75}

\def\cs#1{\texttt{\char`\\#1}}

\newcommand{\Dsq}[2]{\divSymbol_{\mathsf{sq}}\prn*{#1,#2}}
\newcommand{\DsqX}[3]{\divSymbol_{\mathsf{sq}}\prn[#1]{#2,#3}}
\newcommand{\Dbin}[2]{\divSymbol_{0/1}\prn*{#1,#2}}
\newcommand{\Dbinshort}{\divSymbol_{0/1}}
\newcommand{\DbinX}[3]{\divSymbol_{0/1}\prn[#1]{#2,#3}}

\newcommand{\Dsqshort}{\divSymbol_{\mathsf{sq}}}

\newcommand{\Dkl}[2]{\divSymbol_{\mathsf{KL}}\prn*{#1\,\|\,#2}}
\renewcommand{\DsqX}[3]{\divSymbol_{\mathsf{sq}}\prn[#1]{#2,#3}}
\newcommand{\DhelsX}[3]{\divSymbol^{2}_{\mathsf{H}}\prn[#1]{#2,#3}}
\newcommand{\Dhels}[2]{\divSymbol^{2}_{\mathsf{H}}\prn*{#1,#2}}

\newcommand{\Dhelshort}{\divSymbol_{\mathsf{H}}}

\newcommand{\Dsbeshort}{D_{\mathsf{sbe}}}

\newcommand{\Dgen}{\div}

\newcommand{\compgen}[1][\divSymbol]{\comp^{#1}}

\newcommand{\compsbe}[1][D]{\comp^{\Dsbeshort}}

\newcommand{\RegDM}{\Reg_{\mathsf{DM}}}

\renewcommand{\c}{\mathrm{c}}

\newcommand{\AlgOff}{\mathrm{\mathbf{Alg}}_{\mathsf{Off}}}

\renewcommand{\emptyset}{\varnothing}

\newcommand{\filt}{\mathscr{F}}

\newcommand{\Framework}{Decision Making with Structured Observations\xspace}
\newcommand{\FrameworkShort}{DMSO\xspace}

\newcommand{\learner}{learner\xspace}

\newcommand{\RewardSpace}{\cR}
\newcommand{\RSpace}{\RewardSpace}

\newcommand{\comp}[1][\gamma]{\mathsf{dec}_{#1}}

\newcommand{\etdoff}{\textsf{E\protect\scalebox{1.04}{2}D.Off}\xspace}

\newcommand{\etd}{\textsf{E\protect\scalebox{1.04}{2}D}\xspace}

\newcommand{\M}[1]{^{{\scriptscriptstyle M}}}  %
\newcommand{\sM}{\sss{M}}

\newcommand{\sups}[1]{^{{\scriptscriptstyle#1}}}
\newcommand{\subs}[1]{_{{\scriptscriptstyle#1}}}

\newcommand{\sss}[1]{{\scriptscriptstyle#1}}

\newcommand{\pim}[1][M]{\pi_{\sss{#1}}}

\newcommand{\gm}{g\sups{M}}

\newcommand{\pimstar}{\pi\subs{\Mstar}}

\newcommand{\fstar}{f^{\star}}

\newcommand{\Mbar}{\wb{M}}

\newcommand{\Rm}[1][M]{R\sups{#1}}
\newcommand{\Pm}[1][M]{P\sups{#1}}

\newcommand{\PiRNS}{\Pi_{\mathrm{RNS}}} %

\newcommand{\Reg}{\mathrm{\mathbf{Reg}}}

\newcommand{\EstHel}{\mathrm{\mathbf{Est}}_{\mathsf{H}}}

  \newcommand{\AlgEst}{\mathrm{\mathbf{Alg}}_{\mathsf{Est}}}

\newcommand{\Mhat}{\wh{M}}
\newcommand{\Mstar}{M^{\star}}

\newcommand{\algcommentlight}[1]{\textcolor{blue!70!black}{\transparent{0.5}\footnotesize{\texttt{\textbf{//\hspace{2pt}#1}}}}}
\newcommand{\algcommentlighttiny}[1]{\textcolor{blue!70!black}{\transparent{0.5}\scriptsize{\texttt{\textbf{//\hspace{2pt}#1}}}}}

\newcommand{\fhat}{\wh{f}}

\newcommand{\fbar}{\bar{f}}

\renewcommand{\ind}[1]{^{{\scriptscriptstyle #1}}}

\newcommand{\bigoh}{O}
\newcommand{\bigoht}{\wt{O}}
\newcommand{\bigom}{\Omega}

\newcommand{\indic}{\mathbbm{1}}

\newcommand{\poly}{\mathrm{poly}}

\newcommand{\Ber}{\mathrm{Ber}}

\newcommand{\dmid}{\;\|\;}

\newcommand{\unif}{\mathrm{Unif}}

\newcommand{\ftil}{\tilde{f}}

\def\multiset#1#2{\ensuremath{\left(\kern-.3em\left(\genfrac{}{}{0pt}{}{#1}{#2}\right)\kern-.3em\right)}}

\renewcommand{\emptyset}{\varnothing}

\newcommand{\Orc}{\AlgOff}
\newcommand{\AlgOn}{\mathrm{\mathbf{Alg}}_{\mathsf{On}}}

\newcommand{\Bloss}[2]{\Dbin{#1}{#2}}
\newcommand{\Major}{\mathrm{Majority}}
\newcommand{\Ldim}{\mathrm{Ldim}}
\newcommand{\istar}{i^{\star}}

\newcommand{\cov}{\mathsf{cov}}
\newcommand{\dtil}{\wt{d}}
\newcommand{\mustar}{\mu^\star}

\newcommand{\divSymbol}{\mathsf{D}}
\newcommand{\divhSymbol}{\divSymbol_h}
\newcommand{\divRLSymbol}{\divSymbol^{\mathsf{RL}}}

\newcommand{\loss}[2]{\divSymbol\prn*{#1,#2}}
\renewcommand{\div}[2]{\divSymbol\prn*{#1,#2}}
\newcommand{\divX}[3]{\divSymbol\prn[#1]{#2,#3}}
\newcommand{\divh}[2]{\divSymbol_h\prn*{#1\|#2}}
\newcommand{\divRL}[2]{\divSymbol^{\mathsf{RL}}\prn*{#1\|#2}}
\newcommand{\divRLHels}[2]{\divSymbol_{\mathsf{H}}^{\mathsf{RL}}\prn*{#1\|#2}}

\newcommand{\divCBSymbol}{\divSymbol_\mathsf{CB}}
\newcommand{\divCB}[2]{\divCBSymbol\prn*{#1,#2}}

\newcommand{\estmem}{oracle-efficient\xspace}
\newcommand{\memless}{memoryless\xspace}

\newcommand{\memlessOE}{memoryless oracle-efficient\xspace}

\newcommand{\AlgDOL}{\cA_{\mathsf{DOL}}}
\newcommand{\AlgCDE}{\cA_{\mathsf{CDE}}}
\newcommand{\AlgCDEwRO}{\cA_{\mathsf{CDEwRO}}}

\newcommand{\AlgCDEwRP}{\cA_{\mathsf{CDEwRP}}}
\newcommand{\AlgCDEwDRP}{\cA_{\mathsf{ODEwDRP}}}

\newcommand{\AlgOL}{\cA_{\mathsf{OL}}}

\newcommand{\RCDE}{R_{\mathsf{CDE}}}
\newcommand{\RDOL}{R_{\mathsf{DOL}}}
\newcommand{\RCDEwRO}{R_{\mathsf{CDEwRO}}}

\newcommand{\RCDEwRP}{R_{\mathsf{CDEwRP}}}
\newcommand{\RCDEwDRP}{R_{\mathsf{CDEwDRP}}}

\newcommand{\ROL}{R_{\mathsf{OL}}}

\newcommand{\EstOnD}[1][\mathsf{\divSymbol}]{\mathrm{\mathbf{Est}}_{#1}^{\mathsf{On}}}
\newcommand{\EstOffD}[1][\mathsf{\divSymbol}]{\mathrm{\mathbf{Est}}_{#1}^{\mathsf{Off}}}
\newcommand{\EstOffHels}{\EstOffD[\mathsf{H}]}

\newcommand{\EstOnHels}{\EstOnD[\mathsf{H}]}
\newcommand{\EstOnSq}{\EstOnD[\mathsf{sq}]}
\newcommand{\EstOffSq}{\EstOffD[\mathsf{sq}]}
\newcommand{\EstOnBin}{\EstOnD[0/1]}

\newcommand{\AlgRed}{\mathrm{\mathbf{Alg}}_{\mathsf{OEOE}}}

\newcommand{\act}{\pi}
\newcommand{\Act}{\Pi}
\newcommand{\ActSpace}{\Pi}

\newcommand{\mlike}{metric-like\xspace}
\newcommand{\Mlike}{Metric-like\xspace}

\newcommand{\Lrn}{F}

\newcommand{\Kernel}{\cK}

\newcommand{\FunSpace}{\cF}
\newcommand{\fun}{f}
\newcommand{\funName}{parameter\xspace}
\newcommand{\funNames}{parameters\xspace}
\newcommand{\FunNames}{Parameters\xspace}
\newcommand{\FunSpaceName}{parameter space\xspace}

\newcommand{\truefunName}{target parameter\xspace}

\newcommand{\ValSpace}{\mathcal{\cZ}}
\newcommand{\val}{z}
\newcommand{\valName}{value\xspace}
\newcommand{\valNames}{values\xspace}

\newcommand{\ObsSpace}{\cY}
\newcommand{\obs}{y}
\newcommand{\obsName}{outcome\xspace}
\newcommand{\obsNames}{outcomes\xspace}

\newcommand{\CovarSpace}{\cX}
\newcommand{\covar}{x}

\newcommand{\covarNames}{covariates\xspace}

\newcommand{\decn}{\pi}

\newcommand{\ModelSpace}{\cM}
\newcommand{\modl}{M}
\newcommand{\modlName}{model\xspace}

\newcommand{\offgan}{{\beta_{\mathsf{Off}}}}
\newcommand{\ongan}{{\beta_{\mathsf{On}}}}

\newcommand{\OObsSpace}{\cO}
\newcommand{\oobs}{o}

\newcommand{\Dabs}[2]{\divSymbol_{\mathsf{abs}}\prn*{#1,#2}}

\newcommand{\CovarSpacebar}{\widebar{\CovarSpace}}

\usepackage{scalerel}

\newcommand{\set}[1]{\left\{#1\right\}}

\renewcommand{\d}{\textnormal{d}}

\usepackage{caption}
\usepackage{graphicx}

\input{widebar}

\usepackage[suppress]{color-edits}
 \addauthor{df}{ForestGreen}
  \addauthor{cut}{purple}
 \addauthor{jq}{yellow!60!black}
 \addauthor{sr}{BurntOrange}
 \addauthor{yh}{magenta}
 
\newcommand{\dfe}[1]{\dfe{#1}}

\newcommand{\fakepar}[1]{\paragraph{#1}}

\makeatletter
\let\OldStatex\Statex
\renewcommand{\Statex}[1][3]{%
  \setlength\@tempdima{\algorithmicindent}%
  \OldStatex\hskip\dimexpr#1\@tempdima\relax}
\makeatother

\colt{
  \algrenewcommand\algorithmicindent{0.5em}
}

 \usepackage{accents}

\addtocontents{toc}{\protect\setcounter{tocdepth}{0}}

\let\oldparagraph\paragraph
\renewcommand{\paragraph}[1]{\oldparagraph{#1.}}

\newcommand{\paragraphi}[1]{\par\noindent\emph{#1.}}

\makeatletter
\newenvironment{protocol}[1][htb]{%
  \renewcommand{\ALG@name}{Protocol}%
  \begin{algorithm}[#1]%
  }{\end{algorithm}
}
\makeatother

\arxiv{
\title{Online Estimation via Offline Estimation:\\An
  Information-Theoretic Framework}
\author{%
  Dylan J. Foster\\
{\small\texttt{dylanfoster@microsoft.com}}
\and
Yanjun Han\\
{\small\texttt{yanjunhan@nyu.edu}}
\and
Jian Qian\\
{\small\texttt{jianqian@mit.edu}}
\and
Alexander Rakhlin\\
{\small\texttt{rakhlin@mit.edu}}
}
\date{}
}

\colt{
  \title[Online Estimation via Offline Estimation]{Online Estimation via Offline Estimation:\\An
  Information-Theoretic Framework}

  \coltauthor{%
 \Name{Author Name1} \Email{abc@sample.com}\\
 \addr Address 1
 \AND
 \Name{Author Name2} \Email{xyz@sample.com}\\
 \addr Address 2%
}
}

\begin{document}

\maketitle

\begin{abstract}
The classical theory of statistical estimation %
aims to estimate a parameter of interest under data
generated from a fixed design (``offline estimation''), while the contemporary theory of online learning provides
algorithms for estimation under adaptively
  chosen covariates (``online estimation''). Motivated by connections
between %
estimation and
interactive decision making, we ask: \emph{is it possible to convert offline estimation algorithms into online estimation algorithms in a black-box fashion?}
We investigate this question from an information-theoretic
perspective by introducing a new
  framework, \emph{Oracle-Efficient Online Estimation} (\framework), where the learner can only interact with the
  data stream indirectly through a sequence of \emph{offline
    estimators} produced by a black-box algorithm operating on the stream. Our main results settle the statistical
  and computational complexity of online estimation in this
  framework.\loose
  \colt{\begin{enumerate}[leftmargin=3.5em,rightmargin=3em,topsep=2pt,itemsep=2pt]
    }
    \arxiv{\begin{enumerate}}
  \item \emph{Statistical complexity.} We show that
    information-theoretically, there exist algorithms that
    achieve near-optimal online estimation error via black-box
    offline estimation oracles, and give a nearly-tight
    characterization for minimax rates in the \framework framework.
  \item \emph{Computational complexity.} We show that the guarantees
    above cannot be achieved in a computationally efficient fashion in
    general, but give a refined characterization for the special case of conditional density
    estimation: computationally
    efficient online estimation via black-box offline estimation is possible
    whenever it is possible via
    unrestricted algorithms.\loose
  \end{enumerate}
    Finally, we apply our results
    to give offline oracle-efficient algorithms for interactive
    decision making.

\end{abstract}

\colt{
  \begin{keywords}%
    Online learning, interactive decision making, oracle-efficient, regression,
    classification, conditional density estimation
  \end{keywords}
}

\section{Introduction}
\label{sec:intro}

\dfedit{Consider a general framework for statistical estimation specified by a tuple $(\CovarSpace,\ObsSpace, \ValSpace,\Kernel,\FunSpace)$, which we will show encompasses classification, regression, and conditional density estimation.} The learner is given a \FunSpaceName $\FunSpace$ (typically a function class), where each \funName $\fun\in \FunSpace$ is a map from the space of \emph{covariates} $\cX$ to the space of \emph{values} $\ValSpace$. 
For an integer $T\geq 1$, the learner is given a dataset $(\covar\ind{1},y\ind{1}),\dots,(\covar\ind{T},y\ind{T})$, where $x\ind{1},\ldots,x\ind{T}$ are \emph{covariates} and $y\ind{1},\ldots,y\ind{T}$ are \emph{\obsNames}  generated via $y\ind{t}\sim\Kernel(\cdot\mid{}\fstar(\covar\ind{t}))$, where $\fstar\in \FunSpace$ is an unknown \truefunName that the learner wishes to estimate; here $\Kernel$ is a probability kernel that assigns to each \valName  $\val\in \ValSpace$ a distribution $\Kernel(\cdot\mid{}\val)$ on the space of \obsNames $\ObsSpace$.
We adopt the shorthand $\Kernel(\val) =  \Kernel(\cdot\mid{}\val)$ throughout.\loose

The classical theory of statistical estimation typically assumes that the \emph{covariates} $x\ind{1},\ldots,x\ind{T}$ are an arbitrary fixed design, and is concerned with estimating the target parameter $\fstar\in\cF$ well \emph{in-distribution}
\citep{Sara00,lehmann2006theory,tsybakov2008introduction}. Formally, for a \emph{loss function}  $\divSymbol:\ValSpace\times \ValSpace\to \bbR_{\geq 0}$ on the space of \valNames $\ValSpace$, the goal of the learner is to output an estimator $\fhat$ based on $(\covar\ind{1},y\ind{1}),\dots,(\covar\ind{T},y\ind{T})$ such that the in-distribution error\loose
\begin{align}
  \EstOffD(T)\ldef \sum\arxiv{\limits}\colt{\nolimits}_{t=1}^{T}\divX{\big}{\fhat(\covar\ind{t})}{\fstar(\covar\ind{t})}\label{eq:offline}
\end{align}
is small; we refer to this as an \emph{offline estimation} guarantee. Canonical examples include:
\colt{\begin{itemize}[leftmargin=*]}
  \arxiv{\begin{itemize}}
  \colt{\setlength{\parskip}{2pt}}
\item Classification (i.e., distribution-free PAC learning \citep{kearns1987learnability,valiant1984theory}),
  where $\cZ=\cY=\crl{0,1}$, $\Kernel(\fstar(x))=\indic_{\fstar(x)}$,\footnote{We use $\indic_{y}$ to indicate the delta distribution that places probability mass $1$ on $y$.} and 
  $\DbinX{\big}{\fhat(x)}{\fstar(x)}=\indic\crl{\fhat(x)\neq\fstar(x)}$ is the \colt{$0/1$-loss}\arxiv{indicator loss}.\arxiv{\footnote{We focus on noiseless binary classification. For noisy binary classification, one we can set $\ValSpace=[0,1]$, $\Kernel(\fstar(x)) = \Ber(\fstar(x))$, and take $\Dabs{\fhat(x)}{\fstar(x)} = \abs{\fhat(x) - \fstar(x) }$ to be the absolute loss.}}\loose
\item Regression with a well-specified model \citep{tsybakov2008introduction,wainwright2019high}, where $\cZ=\cY=\bbR$, $\Kernel(\fstar(x))=\cN(\fstar(x),\sigma^2)$, and $\DsqX{\big}{\fhat(x)}{\fstar(x)}=(\fhat(x)-\fstar(x))^2$ is the square loss.
\item Conditional density estimation
  \citep{bilodeau2023minimax}, where $\cY$ is an arbitrary alphabet, $\cZ=\Delta(\cY)$, $\Kernel(\fstar(x))=\fstar(x)$, and $\Dhels{\cdot}{\cdot}$ is squared Hellinger distance; see \cref{sec:examples} for details.

\end{itemize}

In parallel to statistical estimation, the contemporary theory of online learning \citep{cesa2006prediction,StatNotes2012} provides estimation error algorithms that support \emph{adaptively chosen} sequences of covariates, a meaningful form of \emph{out-of-distribution} guarantee. Here, the examples $(x\ind{t},y\ind{t})$ arrive one at a time. For each step $t\in\brk{T}$, the learner produces an estimator $\fhat\ind{t}:\cX\to\cZ$ based on the data $(x\ind{1},y\ind{1}),\ldots,(x\ind{t-1},y\ind{t-1})$ observed so far. The covariate $x\ind{t}$ is then chosen in an arbitrary fashion, and the \obsName is generated via $\obs\ind{t}\sim \Kernel(\fstar(\covar\ind{t}))$ and revealed to the learner. The quality of the estimators produced by the learner is measured via\footnote{For technical reasons, it is also common to consider randomized estimators where $\fhat\ind{t}\sim\mu\ind{t}$, and measure error by $  \EstOnD(T)\ldef \sum_{t=1}^{T}\En_{\fhat\ind{t}\sim\mu\ind{t}}\brk[\big]{\divX{\big}{\fhat\ind{t}(\covar\ind{t})}{\fstar(\covar\ind{t})}}$.\loose}
\begin{align}
  \EstOnD(T)\ldef \sum\arxiv{\limits}\colt{\nolimits}_{t=1}^{T}\divX{\big}{\fhat\ind{t}(\covar\ind{t})}{\fstar(\covar\ind{t})}.\label{eq:online}
\end{align}
We refer to this as an \emph{online estimation} guarantee; classical examples include online classification in the mistake-bound model \citep{littlestone1988learning}, online regression \citep{rakhlin2014nonparametric}, and online conditional density estimation \citep{bilodeau2020tight}. Online estimation provides a non-trivial out-of-distribution guarantee, as it requires (on average) that the learner achieves non-trivial estimation performance on covariates $x\ind{t}$ that can be arbitrarily far from the previous examples $x\ind{1},\ldots,x\ind{t-1}$.
This property has many applications in algorithm design, notably in the context of \emph{interactive decision making}, where it has recently found extensive use for problems including contextual bandits \citep{foster2020beyond,simchi2020bypassing,foster2021efficient}, reinforcement learning  \citep{foster2021statistical,foster2023tight}, and imitation learning \citep{ross2011reduction,ross2014reinforcement}.\loose

In this paper, we investigate the relative power of online and offline estimation through a new information-theoretic perspective.
It is well known that any algorithm for online estimation can be used \emph{as-is} to solve offline estimation through \emph{online-to-batch conversion}, a standard technique in learning theory and statistics \citep{aizerman1964theoretical,barron1987bayes,catoni1997mixture,tsybakov2003optimal,juditsky2008learning,audibert2009fast}. The converse is less apparent, as online estimation requires non-trivial algorithm design techniques that go well beyond classical estimators like least-squares or maximum likelihood \colt{\citep{cesa2006prediction}}\arxiv{\citep{kivinen1997exponentiated,cesa1997analysis, vovk1998competitive,  vovk2001competitive,azoury2001relative,auer2002adaptive,vovk2006metric,hazan2007online,gerchinovitz2013sparsity,rakhlin2014online,hazan2015online,rakhlin2015sequential,orseau2017soft,foster2018logistic,luo2018efficient}}. In the case of regression with a finite class $\cF$, least squares achieves optimal offline estimation error $\EstOffD(T)\leq\bigoh\prn*{\log\abs{\cF}}$,\footnote{We consider \emph{unnormalized} estimation error in \cref{eq:offline}, following the convention of online learning. For normalized estimation error, we have $\frac{1}{T}\EstOffD(T)\leq\frac{\log\abs{\cF}}{T}$, in line with the classical convention of statistical estimation.}
and while it is possible to achieve a similar rate $\EstOnD(T)\leq\bigoh\prn*{\log\abs{\cF}}$ for online estimation, this requires Vovk's aggregating algorithm or exponential weights \citep{vovk1998competitive}; directly applying least squares or other standard offline estimators leads to vacuous guarantees.
This leads us to ask: \colt{\emph{Is it possible to convert offline estimation algorithms into online estimation algorithms in a black-box fashion?}\loose}
\arxiv{\begin{center}
  \textit{Is it possible to convert offline estimation algorithms into online estimation algorithms in a black-box fashion?}
\end{center}
}

Computationally speaking, this question has practical significance, since online estimation algorithms are typically far less efficient than their offline counterparts (the classical exponential weights algorithm maintains a separate weight for every $f\in\cF$, which is exponentially less memory-efficient than empirical risk minimization). In fact, at first glance this seems like a \emph{purely} computational question: if the learner has access to an offline estimator, nothing is stopping them (information-theoretically) from throwing the estimator away and using the data to run an online estimation algorithm.\footnote{\arxiv{Closely related}\colt{Related} computational questions have already been studied, with negative results \citep{blum1994separating,hazan2016computational}.\loose} Yet, for aforementioned applications in interactive decision making  \citep{foster2020beyond,simchi2020bypassing,foster2021efficient,foster2021statistical,foster2023tight,ross2011reduction,ross2014reinforcement}, estimation algorithms---particularly online estimators---play a deeper information-theoretic role, and can be viewed as compressing the data stream into a succinct, operational representation that directly informs downstream decision making. With these applications in mind, the first contribution of this paper is to introduce a new protocol, \emph{Oracle-Efficient Online Estimation}, which provides an \emph{information-theoretic} abstraction of the role of online versus offline estimation, \dfedit{analogous to the framework of information-based complexity in optimization \citep{nemirovski1983problem,traub1988information,raginsky2011information,agarwal2012information} and statistical query complexity in theoretical computer science \citep{blum1994weakly,kearns1998efficient,feldman2012complete,feldman2017general}.\loose}%

\subsection{Our Protocol: Oracle-Efficient Online Estimation}
\label{sec:protocol}

In the Oracle-Efficient Online Estimation (\framework) framework, the aim is to perform \emph{online} estimation in the sense of \cref{eq:online}, with the twist that the learner does not directly observe the outcomes $y\ind{1},\ldots,y\ind{T}$; rather, they interact with the environment \emph{indirectly} through a sequence of \emph{offline} estimators produced by a black-box algorithm operating on the historical data. We formalize this black-box algorithm as an \emph{estimation oracle} $\Orc = \crl*{\Orc\ind{t}}_{t=1}^{T}$ (e.g., \citet{foster2023foundations}), which is a mapping from histories to estimators that enjoy bounded offline estimation error.

\begin{definition}[Offline estimation oracle]
  \label{def:offline-oracle}
  An \emph{offline estimation oracle} $\Orc=\crl*{\Orc\ind{t}}_{t=1}^{T}$ for any instance $(\CovarSpace, \ObsSpace,\ValSpace,\Kernel, \FunSpace)$ and loss $\divSymbol$ is a mapping $\Orc\ind{t}:(\cX\times\cY)^{t-1}\to(\cX\to\cZ)$ such that for any sequence $(x\ind{1},y\ind{1}),\ldots,(x\ind{T},y\ind{T})$ with $y\ind{t}\sim{}\Kernel(\fstar(x\ind{t}))$, the sequence of estimators $\fhat\ind{t}=\Orc\ind{t}(x\ind{1},\ldots,x\ind{t-1},y\ind{1},\ldots,y\ind{t-1})$ satisfies
  \colt{$\EstOffD(t) \ldef \sum_{s=1}^{t-1}\divX{\big}{\fhat\ind{t}(\covar\ind{s})}{\fstar(\covar\ind{s})} \leq \offgan$}
  \arxiv{
  \begin{align}
    \label{eq:offline_constraint_simulation}
    \EstOffD(t) = \sum_{s=1}^{t-1}\divX{\big}{\fhat\ind{t}(\covar\ind{s})}{\fstar(\covar\ind{s})} \leq \offgan
  \end{align}
  }
  for all $1\leq{}t\leq{}T$ almost surely; \dfedit{we allow $x\ind{t}$ to be selected adaptively based on $y\ind{1},\ldots,y\ind{t-1}$ and $\fhat\ind{1},\ldots,\fhat\ind{t-1}$.} We refer to $\offgan\geq{}0$ as the \emph{offline estimation parameter}.
\end{definition}
This definition simply asserts that the estimators $\fhat\ind{t}$ produced by the offline estimation oracle satisfy the guarantee in \cref{eq:offline}, \dfedit{even when the covariates are selected adaptively}.
Examples include standard algorithms like least-squares for regression and maximum likelihood for conditional density estimation, which guarantee $\offgan\leq\bigoh(\log\abs{\cF})$ with high probability whenever $\cF$ is a finite class; see \pref{app:background} for further background.\footnote{Most algorithms only ensure that the guarantee in \cref{def:offline-oracle} holds with high probability. We assume an almost sure bound to simplify exposition, but our results trivially extend. Likewise, our results immediately extend to handle the case in which $\offgan$ is allowed to grow as a (sublinear) function of $t$.} \colt{Throughout the paper, we assume for simplicity that $\offgan>0$ is known in advance.}

\arxiv{\begin{remark}
  Throughout the paper, we assume for simplicity that the parameter $\offgan$ is known in advance.
\end{remark}}

\begin{protocol}[tp]
  \caption{Oracle-Efficient Online Estimation (OEOE)}
\label{alg:OE2-protocol}
\begin{algorithmic}[1]
\For{$t=1,\dots,T$}
\State \multiline{Learner receives estimator $\fhat\ind{t} = \Orc\ind{t}(\covar\ind{1},\dots,\covar\ind{t-1},\obs\ind{1},\dots,\obs\ind{t-1})$ from offline \arxiv{estimation}\colt{est.} oracle.} %
\State \multiline{Based on \dfedit{$x\ind{1},\ldots,x\ind{t-1}$} and $\fhat\ind{1},\ldots,\fhat\ind{t}$, learner produces an estimator $\fbar\ind{t}\in\cZ^{\cX}$, which may be randomized according to a distribution $\mu\ind{t}$.}
\State \multiline{Based on $\mu\ind{t}$, nature selects covariate $\covar\ind{t}\in\cX$ and \obsName $y\ind{t}\sim \Kernel(\fstar(\covar\ind{t}))$, but does not directly reveal them to the learner.}
\EndFor
\end{algorithmic}
\end{protocol}
\setcounter{algorithm}{0}

\arxiv{Equipped with}\colt{With} 
this definition, we present the Oracle-Efficient Online Estimation protocol in \savehyperref{alg:OE2-protocol}{Protocol \ref*{alg:OE2-protocol}}.
In the protocol, a learner aims to perform online estimation, but at each step $t$, the only information available\arxiv{ to them} is the covariates $x\ind{1},\ldots,x\ind{t-1}$ and
\arxiv{a sequence of}\colt{the}
estimators $\fhat\ind{1},\ldots,\fhat\ind{t}$ generated by an offline estimation oracle satisfying \cref{def:offline-oracle}; the outcomes $y\ind{1},\ldots,y\ind{T}$ are not directly observed. Based on this information, the learner 
\arxiv{must produce}\colt{produces} 
a new estimator $\fbar\ind{t}$ such that the online estimation error
\colt{ $\EstOnD(T) = \sum_{t=1}^{T}\En_{\fbar\sim\mu\ind{t}}\brk*{\divX{\big}{\fbar(\covar\ind{t})}{\fstar(\covar\ind{t})}}$ }
\arxiv{
\begin{align}
  \label{eq:online_protocol}
  \EstOnD(T) = \sum\limits_{t=1}^{T}\En_{\fbar\sim\mu\ind{t}}\brk*{\divX{\big}{\fbar(\covar\ind{t})}{\fstar(\covar\ind{t})}}
\end{align}
}
in \cref{eq:online} is minimized.\footnote{For technical reasons, we allow the learner to randomize the estimator $\fbar\ind{t}$ via a distribution $\mu\ind{t}$.}
\arxiv{We refer to an algorithm as \emph{oracle-efficient}}\colt{An algorithm is termed \emph{oracle-efficient}}
if it attains low online estimation error \eqref{eq:online} in the \framework framework.
Note that while the learner cannot directly observe the outcomes $y\ind{1},\ldots,y\ind{T}$, 
the covariates $x\ind{1},\ldots,x\ind{T}$ are observed; we prove that without this ability, it is impossible to achieve non-trivial estimation performance 
(\cref{sec:stats-results}) 
. \loose

The \framework framework abstracts away the property that oracle-efficient algorithms implicitly interact with the environment through a compressed, potentially lossy channel (the estimation oracle $\AlgEst$).
We believe this property merits deeper investigation: it is shared by essentially all algorithms from recent research that reduces interactive decision making and reinforcement learning to estimation oracles \citep{foster2020beyond,simchi2020bypassing,foster2021statistical,foster2023tight,ross2011reduction,ross2014reinforcement}, yet the relative power of offline oracles and analogously defined online oracles is poorly understood in this context. \dfedit{By providing an information-theoretic abstraction to study oracle-efficiency, the \framework framework plays a role similar to information-based complexity in optimization \citep{nemirovski1983problem,traub1988information,raginsky2011information,agarwal2012information} and statistical query complexity in theoretical computer science \citep{blum1994weakly,kearns1998efficient,feldman2012complete,feldman2017general}, both of which provide rich frameworks for designing and evaluating iterative algorithms that interact with the environment in a structured fashion.} We expect that this abstraction will find broader use for more complex domains (e.g., decision making and active learning) as a means to guide algorithm design and prove lower bounds against natural classes of algorithms.\loose

Let us first build some intuition. Familiar readers may recognize that the classical \emph{\halving algorithm} for binary classification (e.g., \citet{cesa2006prediction}) can be viewed as oracle-efficient in our framework. Specifically, for binary classification with $\cY=\cZ=\crl*{0,1}$ and loss function $\DbinX{\big}{\fhat(x)}{\fstar(x)}=\indic\crl[\big]{\fhat(x)\neq\fstar(x)}$, the \halving algorithm can use any offline oracle with $\offgan=0$ to achieve $\EstOnD(T) = \bigoh(\log\abs{\cF})$, which is optimal. However, little is known for noisy oracles with $\offgan>0$, or more general \obsName spaces and loss functions (e.g.,  regression or density estimation). 
In addition, the \halving algorithm---while oracle-efficient---is computationally inefficient, as it requires maintaining an explicit version space. 
This leads us to restate our central question formally, in two parts:
\colt{\begin{enumerate}[leftmargin=*]}
\arxiv{\begin{enumerate}}
\item     \textit{Can we design oracle-efficient algorithms with near-optimal online estimation error \eqref{eq:online}, up to polynomial factors (for general instances $(\CovarSpace,\ObsSpace, \ValSpace,\Kernel,\FunSpace)$ and $\offgan>0$)?}\label{item:q1}
\item \textit{Can we do so in a computationally efficient fashion?}\label{item:q2}
\end{enumerate}

\subsection{Contributions}

\dfedit{For a general class of losses $\divSymbol$, referred to as \emph{metric-like},} we settle the statistical and computational complexity of performing online estimation via black-box offline estimation oracles up to mild gaps, answering questions \savehyperref{item:q1}{(1)} and \savehyperref{item:q2}{(2)} above.

\paragraph{Statistical complexity}
Our first result concerning statistical complexity focuses on finite classes $\cF$, where the optimal rates for offline and online estimation with standard losses $\div{\cdot}{\cdot}$ both scale as $\Theta(\log\abs{\cF})$. For this setting, we show (\cref{thm:woe-upper-bound}) that there exists an oracle-efficient online estimation algorithm that achieves $  \EstOnD(T) = \bigoh((\offgan+1)\min\crl*{\log\abs{\cF},\abs{\cX}})$ in the \framework framework, and that this is optimal (\cref{thm:woe-lower-bound}). This provides an affirmative answer to question \savehyperref{item:q1}{(1)}, and characterizes the statistical complexity of oracle-efficient online estimation with finite classes $\cF$.\loose

In the general \framework framework, the learner can use the entire history of offline estimators $\fhat\ind{1},\ldots,\fhat\ind{t}$ and covariates $x\ind{1},\ldots,x\ind{t-1}$ to produce the online estimator $\fbar\ind{t}$ for step $t$. As a secondary result, we study a restricted class of \emph{memoryless} oracle-efficient algorithms that choose $\fbar\ind{t}$ only based on the most recent offline estimator $\fhat\ind{t}$, and show (\cref{prop:separation_memoryless_improper}) that it is impossible for such algorithms to achieve low online estimation error.

Lastly, we give a more general approach to deriving oracle-efficient reductions (\cref{thm:delayed_reduction}) that is based on a connection to \emph{delayed online learning} \citep{weinberger2002delayed,mesterharm2007improving,joulani2013online,quanrud2015online}. Using this result, we give a characterization of learnability with \emph{infinite classes} for binary classification in the \framework framework (\cref{thm:binary-loss-learnability}), proving that finite Littlestone dimension is necessary and sufficient for oracle-efficient learnability.

\paragraph{Computational complexity}   On the computational side, we provide a negative answer to question \savehyperref{item:q2}{(2)}, showing (\cref{thm:computation-lower-bound}) that under standard conjectures in computational complexity, there do not exist polynomial-time algorithms with non-trivial online estimation error in \framework framework. In spite of this negative result, we provide a fine-grained perspective for the statistical problem of \emph{conditional density estimation}, a general task that subsumes classification and regression and has immediate applications to reinforcement learning and interactive decision making \citep{foster2021statistical,foster2023tight}. Here we show, perhaps surprisingly (\cref{thm:reduction-to-ODE}), that online estimation in the \framework framework is no harder computationally than online estimation with arbitrary, unrestricted algorithms. This result is salient in light of the applications we discuss below.

\paragraph{Implications for interactive decision making}
As the preceding discussion has alluded to, our interest in studying oracle-efficient online estimation is largely motivated by a connection to the problem of \emph{interactive decision making}. \citet{foster2021statistical,foster2023tight,foster2023foundations} propose a general framework for interactive decision making called \emph{Decision Making with Structured Observations} (DMSO), which subsumes contextual bandits, bandit problems with structured rewards, and reinforcement learning with general function approximation. They show that for any decision making problem in the DMSO framework, there exists an algorithm that, given access to an online estimation algorithm (or, ``oracle'') for conditional density estimation for an appropriate class $\cF$, it is possible to achieve \emph{near-optimal regret}. The results above critically make use of \emph{online estimation} oracles, as they require achieving low estimation error for adaptively chosen sequences of covariates, and it is natural to ask whether similar guarantees can be achieved using only offline estimation oracles. However, positive results are only known for certain special cases \citep{cesa2004generalization,daskalakis2016learning,dudik2020oracle, simchi2020bypassing}, with scant results for reinforcement learning in particular. In this context, our results have the following implication (\cref{cor:application-statistical}): \emph{Information-theoretically, it is possible to achieve near-optimal regret for any interactive decision making problem using an algorithm that accesses the data stream only through offline estimation oracles}.\loose

\section{Preliminaries}
\label{sec:prelim}

\arxiv{\paragraph{Regularity assumptions}}
Unless otherwise stated, our results assume the loss function $\divSymbol$ has \emph{\mlike} structure.\loose
\begin{definition}[\Mlike loss]
A loss function $\divSymbol:\ValSpace\times\ValSpace\to [0,1]$ is \mlike on the set $\ValSpace$ if it is symmetric and satisfies (i) $\loss{\val_1}{\val_2} \geq 0$ for any $\val_1,\val_2\in \ValSpace$ and $\loss{\val}{\val} = 0$ for all $\val\in \ValSpace$; and (ii) $\loss{\val_1}{\val_2} \leq C_\divSymbol \cdot (\loss{\val_1}{\val_3} + \loss{\val_3}{\val_2})$ for all $\val_1,\val_2,\val_3\in \ValSpace$, for an absolute constant $C_\divSymbol\geq{}1$.\loose
\end{definition}

Throughout the paper, we focus on three canonical applications, outlined in the introduction: Classification with the indicator loss $\Dbinshort$ ($\Closs=1$), regression with the square loss $\Dsqshort$ ($\Closs=2$), and conditional density estimation with squared Hellinger distance $\Dhelshort^2$ ($\Closs=2$)\colt{. See \cref{sec:examples} for detailed examples and discussion. (omitted for space).}\arxiv{; additional background on these applications is given in the sequel.\loose}

\arxiv{\paragraph{Finite versus infinite classes}}
\colt{\noindent\textbf{Finite versus infinite classes.}}
The majority of our results focus on finite classes $\cF$. We believe this captures the essential difficulty of the problem, but we expect that most of our sample complexity results (which typically scale with $\log |\cF|$) can be extended to infinite classes by combining our techniques with appropriate notions of complexity for the function class (Littlestone dimension for classification, sequential Rademacher complexity, and sequential covering numbers \citep{rakhlin2012statistical,littlestone1988learning,rakhlin2015online}). For the canonical settings of classification, regression, and conditional density estimation, there exist algorithms that achieve $\EstOffD(T)=\bigoh(\log\abs{\cF})$ and $\EstOnD(T)=\bigoh(\log\abs{\cF})$ for arbitrary finite classes\arxiv{.}\colt{; see \cref{sec:examples} for details.\loose}

\arxiv{
\subsection{Examples of Estimation Problems and Loss Functions}
\label{sec:examples}

In what follows, we give detailed background on three canonical examples
of the general estimation framework discussed in \cref{sec:intro}:
Binary classification, square loss regression, and conditional density
estimation.

\paragraph{Classification \citep{kearns1987learnability,valiant1984theory}} 
For binary classification, we take $\ValSpace=\cY = \set{0,1}$ with the binary loss
$\Dbin{\val_1}{\val_2} = \indic(\val_1\neq \val_2)$ for
$\val_1,\val_2\in \ValSpace$ and kernel $\Kernel(\val) = \indic_z$,
which is noiseless. The binary loss is \mlike with $C_\divSymbol=1$.

For offline estimation, observe that with \covarNames
$\covar\ind{1},\dots,\covar\ind{T}$ and \obsNames $\obs\ind{t} =
{}\fstar(\covar\ind{t})$ for all $t\in \set{1,\dots,T}$, any empirical
risk minimizer $\fhat$ that sets $\fhat(\covar\ind{t}) = y\ind{t}$
obtains
\begin{equation}
\sum_{t=1}^{T}\Dbin{\fhat(\covar\ind{t})}{\fstar(\covar\ind{t})}
=0.
\end{equation}
For online estimation, the \halving algorithm
\citep{cesa2006prediction} achieves
    \begin{equation}
      \label{eq:online-classification}
\EstOnBin(T)    =      \sum_{t=1}^{T}\Dbin{\fhat\ind{t}(\covar\ind{t})}{\fstar(\covar\ind{t})}
        \leq \log(\abs*{\cF}).
      \end{equation}

We mention in passing that another natural classification setting we do not explore in
detail in this paper is \emph{noisy classification}, where the setting
is as above, except that we set $\cZ=\brk{0,1}$,
$\Kernel(\fstar(x))=\Ber(\fstar(x))$, and take $\Dabs{z_1}{z_2} = \abs{z_1-z_2}$ as the absolute loss for all $z_1,z_2\in \cZ$.
      
 \paragraph{Square loss regression \citep{tsybakov2008introduction,wainwright2019high}} For real-valued regression, we
 take $\ValSpace=\cY = \bbR$ with the square loss
 $\Dsq{z_1}{z_2} = (z_1-z_2)^2$ for $z_1,z_2\in \ValSpace$ and the
 kernel $\Kernel(\fstar(x)) = \cN(\fstar(x),1)$ or another subGaussian
 distribution. Note that the square loss is a \mlike loss with
 $C_\divSymbol =2$.

 For offline estimation, with \covarNames $\covar\ind{1},\dots,\covar\ind{T}$ and \obsNames $\obs\ind{t}\sim  {}\fstar(\covar\ind{t})+\veps\ind{t}$ for all $t\in \set{1,\dots,T}$, 
    the classical Empirical Risk Minimization (ERM) $\fhat \ldef \argmin_{f\in\cF} \sum_{t=1}^{T}\prn*{f(\covar\ind{t})-\obs\ind{t}}^2$ gives 
    \begin{equation}
      \label{eq:offline-regression}
\EstOffSq(T)    = \sum_{t=1}^{T} \Dsq{\fhat(\covar\ind{t})}{\fstar(\covar\ind{t})}
      \leq \log(\abs*{\cF}\delta^{-1}),
    \end{equation}
    with probability at least $1-\delta$ (cf. \cref{lem:erm-guarantee}\colt{below}).

    For online estimation, the exponential weights algorithm
    \citep{cesa2006prediction}, with decision space
    $\cF$ and the loss at each round chosen to be $\ell\ind{t}(f) =
    (f(\covar\ind{t}) - \obs\ind{t})^2$, achieves
    \begin{equation}
      \label{eq:online-regression}
\EstOnSq(T) =   \sum_{t=1}^{T} \Dsq{\fhat\ind{t}(\covar\ind{t})}{\fstar(\covar\ind{t})}
      \leq \log(\abs*{\cF}\delta^{-1}),
    \end{equation}
    with probability at least $1-\delta$ (cf. \citet{foster2020beyond}
    for a proof).

    \paragraph{Conditional density estimation
      \citep{bilodeau2023minimax}} For conditional density estimation,
    we consider an arbitrary \obsName space
    $\ObsSpace$ and take $\ValSpace = \Delta(\ObsSpace)$ with \emph{squared
    Hellinger distance} $\Dhelshort^2$ given by\footnote{More generally, if $\nu$ is a common dominating measure, then
  $\DhelsX{\big}{\bbP}{\bbQ}=\frac{1}{2}\int\prn[\Big]{\sqrt{\frac{d\bbP}{d\nu}}-\sqrt{\frac{d\bbQ}{d\nu}}}^{2}d\nu$,
where $\frac{d\bbP}{d\nu}$ and $\frac{d\bbQ}{d\nu}$ are
Radon-Nikodym derivatives. The notation in \pref{eq:hellinger} reflects that this quantity is
invariant under the choice of $\nu$.}
\begin{equation}
\label{eq:hellinger}
\DhelsX{\big}{\fhat(x)}{\fstar(x)}= \frac{1}{2} \int\prn*{\sqrt{\fhat(y\mid{}x)}-\sqrt{\fstar(y\mid{}x)}}^{2} \d y.
\end{equation}
and $\Kernel(\val)=\val$ for all $\val\in \ValSpace$. Note that squared Hellinger
distance is a \mlike loss with $C_\divSymbol =2$. 

For offline estimation, with \covarNames $\covar\ind{1},\dots,\covar\ind{T}$ and \obsNames $\obs\ind{t}\sim  {}\fstar(\covar\ind{t})$ for all $t\in \set{1,\dots,T}$, 
    the classical Maximum Likelihood Estimator (MLE) $\fhat \ldef \argmax_{f\in\cF} \sum_{t=1}^{T} \log f(\obs\ind{t}\mid{}\covar\ind{t}) $ gives 
    \begin{equation}
      \label{eq:offline-mle}
\EstOffHels(T) =     \sum_{t=1}^{T} \Dhels{\fhat(\covar\ind{t})}{\fstar(\covar\ind{t})}
      \leq \log(\abs*{\cF}\delta^{-1}),
    \end{equation}
    with probability at least $1-\delta$ (cf. \cref{lem:mle-guarantee}\colt{below}).

    For online estimation, the exponential weights algorithm
    \citep{cesa2006prediction}, with decision space $\cF$ and the loss at each round chosen to be $\ell\ind{t}(f) = -\log f(\obs\ind{t}\mid{}\covar\ind{t} )$, achieves
    \begin{equation}
      \label{eq:online-mle}
\EstOnHels(T)    =  \sum_{t=1}^{T} \Dhels{\fhat\ind{t}(\covar\ind{t})}{\fstar(\covar\ind{t})}
      \leq \log(\abs*{\cF}\delta^{-1}),
    \end{equation}
    with probability at least $1-\delta$ (cf. \citet{foster2021statistical}
    for a proof).

}

\arxiv{\subsection{Additional Notation}}
\colt{\vspace{3pt}\noindent\textbf{Additional notation}} We denote $\bbR_{\geq 0} = [0,\infty)$. For any $a, b\in \bbR$, $a\wedge b \ldef \min \set{a,b}$ and $a\vee b\ldef \max\set{a,b}$. For any integer $N>0$, $[N] = \crl{1,\dots,N}$. For any set $\cX$, $\Delta(\cX)$ is the space of all distributions on $\cX$. For any integer $T$, the notation $x\ind{1:T}$ will be the shorthand notation for the sequence $x\ind{1},\dots,x\ind{T}$.
For any real number $x\in \bbR$, denote by $\lfloor x \rfloor$ the largest integer that is smaller than or equal to $x$ and by $\lceil x\rceil$ the smallest integer that is greater than or equal to $x$. The indicator function is denoted by $\mathbbm{1}(\cdot)$. We define $O(\cdot)$, $\Omega(\cdot)$, $o(\cdot)$, $\Theta(\cdot)$, $\wt{O}(\cdot)$, $\wt{\Omega}(\cdot)$, $\wt{\Theta}(\cdot)$ following standard non-asymptotic big-oh notation. We use the binary relation $x \lesssim y$ to indicate that $x\leq O(y)$. 

\arxiv{\subsection{Paper Organization}}
\colt{\paragraph{Organization}}

\colt{\cref{sec:stats-results} presents our results on the statistical complexity of the \framework framework and \cref{sec:comp-results} presents our main computational results. We conclude with a discussion in \cref{sec:discussion}.\loose\\ \emph{Due to space constraints, the following results are deferred to the appendix: (i) detailed examples (classification, regression, and density estimation) for our statistical estimation framework (\cref{sec:examples}); (ii) additional results concerning statistical complexity of the \framework framework (\cref{sec:delayed}); and (iii) detailed results for our application to interactive decision making (\cref{sec:application}).}\loose
  }

\arxiv{\cref{sec:stats-results} presents our main results concerning the statistical complexity of the \framework framework---both upper bounds and impossibility results. \cref{sec:comp-results} presents our main computational results, and \cref{sec:application} gives implications for reinforcement learning and decision making. We conclude with a discussion and open problems in \cref{sec:discussion}. Proofs and additional results are deferred to the appendix.
  }

\section{Statistical Complexity of Oracle-Efficient Online Estimation}
\label{sec:stats-results}

This section presents our main results concerning the statistical
complexity of oracle-efficient online estimation. In \cref{subsec:opt-err-bound}, we focus on finite
  classes $\cF$ and present an oracle-efficient algorithm that
  achieves near-optimal online estimation error
  (\cref{thm:woe-upper-bound}). We then provide a lower bound that shows that our reduction
  is near optimal (\cref{thm:woe-lower-bound}). In
  \cref{sec:memoryless}, we turn our attention to memoryless
  oracle-efficient algorithms, proving strong impossibility results (\cref{prop:separation_memoryless_improper}).
  \arxiv{Finally, in \cref{sec:delayed}, we provide a more general
  approach to designing oracle-efficient algorithms based on delayed
  online learning (\cref{thm:delayed_reduction}), and use it to derive
  a characterization of oracle-efficient learnability for
  classification with infinite classes $\cF$ (\cref{thm:binary-loss-learnability}).}

\subsection{Minimax Sample Complexity for Oracle-Efficient Algorithms}
\label{subsec:opt-err-bound}

In this section, we present our main statistical conclusion for the
\framework framework: \emph{For any finite class $\cF$, it is possible to
transform any black-box offline estimation algorithm into an online
estimation algorithm with near-optimal error} (up to a logarithmic
factor that we show is unavoidable).

\paragraph{Algorithm and minimax upper bound}
Our results are achieved through a new algorithm, \emph{\mainalg}, described in
\cref{alg:weighted-maj-vote}. At each round $t$, the algorithm
uses estimators $\fhat\ind{1},\ldots,\fhat\ind{t}$ produced by an
offline estimation oracle $\Orc$, along with the previous covariates
$x\ind{1},\ldots,x\ind{t-1}$, to construct a \emph{version space}
$\cF_t\subseteq\cF$ in \cref{eq:version_space}.
Informally, $\FunSpace_t$ consists of all $f\in\cF$ that are
\emph{consistent} with the estimators
$\fhat\ind{1},\ldots,\fhat\ind{t}$ in the sense that for all
$s\in\brk{t}$, the offline estimation error relative to $\fhat\ind{s}$
is small; as long as the offline estimation oracle $\Orc$ has offline
estimation error $\offgan$ (\cref{def:offline-oracle}), it follows
immediately that the construction in \cref{eq:version_space} satisfies $\fstar\in\cF_t$. Given the
version space $\cF_t$, \cref{alg:weighted-maj-vote} predicts by uniformly sampling:
$\fbar\ind{t}\sim\mu\ind{t} \ldef \unif(\FunSpace_t)$, then proceeds
to the next round.\footnote{
      For realizable binary classification with
  $\DbinX{\big}{\fhat(x)}{\fstar(x)} = \mathbbm{1}(\fhat(x)\neq
  \fstar(x)) $ and $\offgan=0$, the version space construction in
  \cref{eq:version_space} coincides with that of the \halving
  algorithm (e.g., \citet{cesa2006prediction}), and its estimation error bound
(\cref{thm:woe-upper-bound})
  matches the halving algorithm up to
  absolute constants. As such, \cref{alg:weighted-maj-vote} can be
  viewed as a noisy/error-tolerant generalization of the \halving
  algorithm, which may find broader use.
} The main guarantee for \cref{alg:weighted-maj-vote} is stated in
\cref{thm:woe-upper-bound}.
\loose

\begin{algorithm}[tp]
\caption{\mainalg} 
\label{alg:weighted-maj-vote}
\begin{algorithmic}[1]
  \State \textbf{input:} \FunSpaceName $\FunSpace$, offline
  estimation oracle $\Orc$ with parameter $\offgan\geq{}0$.
\For{$t=1,2,\dots,T$}
\State Receive $\fhat\ind{t}=\Orc\ind{t}(x\ind{1},\ldots,x\ind{t-1},y\ind{1},\ldots,y\ind{t-1})$.%
\State Calculate version space:
\vspace{-3pt}
\begin{align}
  \FunSpace_{t} = \crl*{ f\in\cF \;\;\middle|\;\; \forall s\in [t],~ \sum\arxiv{\limits_{\tau<s}}\colt{\nolimits_{\tau<s}} \divX{\big}{\fhat\ind{s}(\covar\ind{\tau})}{f(\covar\ind{\tau})} \leq \offgan}  . \label{eq:version_space}
\end{align}
\State Predict $\fbar\ind{t}\sim\mu\ind{t} \ldef \unif(\FunSpace_t)$ and receive
$\covar\ind{t}$.\hfill\algcommentlighttiny{Nature draws
  $y\ind{t}\sim{}\Kernel(\fstar(x\ind{t}))$ and passes to $\Orc$.\loose}
\EndFor
\end{algorithmic}
\end{algorithm}

\begin{restatable}[Main upper bound for \framework]{theorem}{EstMemUpperBound}
  \label{thm:woe-upper-bound}
For any instance $(\CovarSpace,\ObsSpace,
\ValSpace,\Kernel,\FunSpace)$, any \mlike loss $\divSymbol$, and any
offline estimator $\Orc$ with parameter $\offgan\geq{}0$, 
\cref{alg:weighted-maj-vote} is oracle-efficient and
achieves\loose
\begin{align*}
\EstOnD(T)\leq  
\bigoh\prn*{ C_\divSymbol \cdot{}(\offgan+1)\cdot \min \set{ \log |\FunSpace|,  |\CovarSpace|\log T}}.
\end{align*}
\end{restatable}
Most notably, \cref{alg:weighted-maj-vote} achieves
$\EstOnD(T)\leq  
\bigoh\prn*{ C_\divSymbol \cdot{}(\offgan+1)\cdot \log |\FunSpace|}$;
that is, up to a $\bigoh(\log\abs{\cF})$ factor, the
reduction achieves online estimation rates in the \framework
framework that are no worse than the minimax rate for offline
estimation.
For classification, regression, and density estimation with generic finite classes $\cF$ (\cref{sec:examples}),
the best possible offline estimation error rate is
$\offgan=\bigoh(\log\abs{\cF})$, so this shows that price of oracle-efficiency is at most quadratic. %

\paragraph{Minimax lower bound}
Next, we show that the upper bound in \cref{thm:woe-upper-bound} is
nearly tight, giving a lower bound that matches up to logarithmic factors.

\begin{restatable}[Main lower bound for \framework]{theorem}{EstMemLowerBound}
  \label{thm:woe-lower-bound}
  Consider the binary classification setting with $\cZ=\cY=\crl{0,1}$
  and loss $\Dbin{\cdot}{\cdot}$. For any $N\in\bbN$
  \dfedit{and $\offgan>0$}, there
  exists an instance $(\CovarSpace,\ObsSpace,
  \ValSpace,\Kernel,\FunSpace)$ with $\log|\FunSpace|=|\CovarSpace| =
  N$ such that for any \estmem algorithm, there \arxiv{exists}\colt{is} a 
  sequence of covariates $(\covar\ind{1},\dots,\covar\ind{T})$ and 
  offline oracle \arxiv{$\Orc$ }with parameter $\offgan$ such that  \colt{$    \En\brk[\big]{\EstOnD(T)} 
    \geq \Omega (\min \set{ (\offgan+1)N,T})$.\loose}
\arxiv{  \begin{align*}
    \En\brk[\big]{\EstOnD(T)} 
    \geq \Omega (\min \set{ (\offgan+1)N,T}).
         \end{align*}
         }
\end{restatable}
This result states that for a generic finite class $\cF$ and offline
estimation oracle $\Orc$, any oracle-efficient online estimator must
have
\[
      \En\brk[\big]{\EstOnD(T)} 
      \geq \Omega (\min \set{ (\offgan+1)\log |\FunSpace|,
  (\offgan+1)|\CovarSpace|, T})
\]
in the worst case. This implies that the $\log\abs{\cF}$ factor we pay
for offline to online conversion is unavoidable, and that \cref{thm:woe-upper-bound} is
optimal up to a $\log{}T$ factor, giving a near-optimal
characterization for the minimax rate for online estimation in the
\framework framework. 
We conclude with two remarks: (i) The $(\offgan+1)$ scaling (as
  opposed to say, $\offgan$) in \cref{thm:woe-upper-bound} is
  unavoidable, as witnessed by the optimality of the halving algorithm
  for noiseless binary classification \citep{cesa2006prediction}; (ii) if the space $\ValSpace$ and the loss $\divSymbol$ are convex, then
  we can change \cref{alg:weighted-maj-vote} to output a deterministic prediction by
  using the average of all \funNames in $\FunSpace_t$ rather than the
  uniform distribution on $\FunSpace_t$. See
  \cref{lem:general-version-space} for details.

  \paragraph{General reductions and infinite classes}
\cref{alg:weighted-maj-vote} is somewhat
specialized to finite classes. In \cref{sec:delayed} (deferred to the
appendix for space), we provide a more general
  approach to designing oracle-efficient algorithms based on \emph{delayed
  online learning} (\cref{thm:delayed_reduction}), and use it to derive
  a characterization of oracle-efficient learnability for
  classification with infinite classes $\cF$
  (\cref{thm:binary-loss-learnability}).\loose

\subsection{Impossibility of Memoryless Oracle-Efficient Algorithms}
\label{sec:memoryless}
\label{subsec:impossibility-results}

In the general \framework framework, the learner can use the entire
history of estimators $\fhat\ind{1},\ldots,\fhat\ind{t}$ and
covariates $x\ind{1},\ldots,x\ind{t-1}$ to produce the online
estimator $\fbar\ind{t}$ for step $t$; notably the \mainalg algorithm with which our upper bounds in the prequel are
derived uses the entire history. In this section, we show that for
\emph{memoryless} oracle-efficient algorithms
(\cref{def:mem-protocol}) that select the
estimator $\fbar\ind{t}$ only as a function of the most recent offline
estimator $\fhat\ind{t}$, similar guarantees are impossible.\loose

\begin{definition}[Memoryless \arxiv{oracle-efficient }algorithm]
  \label{def:mem-protocol}
An online estimation algorithm is \arxiv{referred to as}
  \emph{\memless} if there exists a map $\Lrn\ind{t}(\cdot)$ such that
  we can write $\mu\ind{t} = \Lrn\ind{t}(\fhat\ind{t})$, where
  $\fhat\ind{t} =
  \Orc\ind{t}(\covar\ind{1},\dots,\covar\ind{t-1},\obs\ind{1},\dots,\obs\ind{t-1})$
  and $\mu\ind{t}$ is the randomization distribution for the online estimator $\fbar\ind{t}$.\footnote{
    There are many other natural variants of this protocol. For example, one could consider algorithms that select $\mu\ind{t}$ based on $\fhat\ind{t}$ and $B$ bits of arbitrary auxiliary memory. We hope to see such variants studied in future work.}\loose
\end{definition}
Memoryless algorithms are more practical than arbitrary algorithms,
since they do not require storing past estimators or covariates in
memory. Our motivation for studying memoryless algorithms arises from recent work in interactive decision making \citep{foster2020beyond,foster2021statistical}, which shows that there exist near-optimal algorithms for contextual bandits and reinforcement learning that use estimation oracles in memoryless fashion.
We show that unfortunately, it is not possible to convert offline estimators into memoryless online estimation algorithms with non-trivial error.\loose

\begin{restatable}[Impossibility of memoryless algorithms for
  \framework]{theorem}{SOELowerBoundImp}
  \label{prop:separation_memoryless_improper}
    Consider the binary classification setting with $\cZ=\cY=\crl{0,1}$
  and loss $\Dbin{\cdot}{\cdot}$.
  \dfedit{For any $N\in\bbN$ and $\offgan\geq{}0$,} there exists an instance $(\CovarSpace,\ObsSpace, \ValSpace,\Kernel,\FunSpace)$
  with $|\FunSpace|=|\CovarSpace|=N$ such that for any \memlessOE
  algorithm, there exists a sequence of covariates
  $(\covar\ind{1},\dots,\covar\ind{T})$ and a (potentially improper)
  offline oracle $\Orc$ with parameter $\offgan$ such that \colt{$    \En\brk[\big]{\EstOnD(T)} 
    \geq \Omega ( \min \set{N(\offgan+1) ,T})$.}
\arxiv{  
  \begin{align*}
    \En\brk[\big]{\EstOnD(T)} 
    \geq \Omega ( \min \set{N(\offgan+1) ,T}).
  \end{align*}}
This conclusion continues to hold when the online estimation algorithm
\arxiv{is allowed to}\colt{can} remember $\fhat\ind{1},\ldots,\fhat\ind{t-1}$, but not $x\ind{1},\ldots,x\ind{t-1}$.
\end{restatable}
This result shows that in the worst case, any \memlessOE algorithm
must have
\begin{align*}
  \En\brk[\big]{\EstOnD(T)}
  \geq \bigom((\offgan+1)\min\set{|\CovarSpace|,|\FunSpace|}).
\end{align*}
This precludes\arxiv{ the possibility of} an online estimation error bound
scaling with $(\offgan+1)\log\abs{\cF}$ as in
\cref{thm:woe-upper-bound}, and shows that the gap between general and memoryless
oracle-efficient algorithms can be
exponential\arxiv{ in general}.\loose

Interestingly, the lower bound in
\cref{prop:separation_memoryless_improper} holds even if the online estimation algorithm is allowed to remember
$\fhat\ind{1},\ldots,\fhat\ind{t-1}$, but not
$x\ind{1},\ldots,x\ind{t-1}$. The intuition here is that
without covariate information, it is not possible to aggregate the
predictions of previous estimators or otherwise use them to reduce
uncertainty. This provides post-hoc motivation for our decision to
incorporate covariate memory into the \framework protocol in \cref{sec:protocol}.

The proof of \cref{prop:separation_memoryless_improper} uses that the estimators $\fhat\ind{t}$ produced by the offline
estimation oracle may be \emph{improper} (i.e.,
$\fhat\ind{t}\notin\cF$). We next provide a variant of the result
that holds even if the estimation oracle is \emph{proper}, under the additional
assumptions that (i) the learner is itself proper in the sense that
$\mu\ind{t}\in\Delta(\cF)$, and (ii) the learner is
\emph{time-invariant} (i.e., the learner sets
$\mu\ind{t}=F(\fhat\ind{t})$ for all $t$).\loose
\newtheorem*{thm:memoryless_var}{Theorem
  \ref*{prop:separation_memoryless_improper}$'$}
\newcommand{\refmemorylessvar}{\savehyperref{thm:soe-lower-bound-proper}{Theorem
    \ref*{prop:separation_memoryless_improper}$'$}\xspace}
\begin{restatable}[Impossibility of memoryless algorithms for
  \framework; proper variant]{thm:memoryless_var}{SOELowerBound}
  \label{thm:soe-lower-bound-proper}
    Consider the binary classification setting with $\cZ=\cY=\crl{0,1}$
  and loss $\Dbin{\cdot}{\cdot}$. \dfedit{For any $N\in\bbN$ and $\offgan\geq{}0$,} there exists an instance $(\CovarSpace,\ObsSpace, \ValSpace,\Kernel,\FunSpace)$
  with $|\FunSpace|=|\CovarSpace|=N$ such that for any \memlessOE
  algorithm that is (i) \textbf{proper}\arxiv{ ($\mu\ind{t}\in\Delta(\cF)$ is
  supported on $\FunSpace$ for all $t\in [T]$)}, and (ii)
  \textbf{time-invariant}\arxiv{ ($\Lrn\ind{1} = \dots = \Lrn\ind{T}$ in the \memless OEOE protocol)}, there exists a 
sequence of covariates $(\covar\ind{1},\dots,\covar\ind{T})$ and a
\textbf{proper} offline oracle $\Orc$ with parameter $\offgan$ such
that \colt{$\En\brk[\big]{\EstOnD(T)}  
\geq \Omega ( \min \set{N (\offgan + 1),T})$.}
\arxiv{\begin{align*}
\En\brk[\big]{\EstOnD(T)}  
\geq \Omega ( \min \set{N (\offgan + 1),T}).
\end{align*}}
\end{restatable}
\fakepar{A complementary upper bound}
For completeness, we conclude by showing that the (large)
lower bound in \cref{prop:separation_memoryless_improper} can be
achieved with a memoryless oracle-efficient algorithm. We consider the ``trivial'' algorithm that outputs the estimators produced by the
offline oracle as-is.
\begin{restatable}[Upper bound for memoryless \framework]{proposition}{soeupperbound}
  \label{prop:soe-upper-bound}
  For any instance $(\CovarSpace,\ObsSpace,
  \ValSpace,\Kernel,\FunSpace)$, \mlike loss $\divSymbol$, and offline
  oracle $\Orc$ with parameter $\offgan$, the algorithm that returns
  $\fbar\ind{t}=\fhat\ind{t}\ldef{}\Orc\ind{t}(x\ind{1},\ldots,x\ind{t-1},y\ind{1},\ldots,y\ind{t-1})$
  has online estimation error \colt{$  \EstOnD(T) %
  \leq  \bigoh\prn*{(\offgan+1) |\CovarSpace|\log T}$.}
\arxiv{  \begin{align*}
  \EstOnD(T) %
  \leq  \bigoh\prn*{(\offgan+1) |\CovarSpace|\log T}.
         \end{align*}
         }
       \end{restatable}

\arxiv{
\subsection{General Reductions for Oracle-Efficient Online Estimation and\\
  Characterization of Oracle-Efficient Learnability for Classification}
\label{sec:delayed}

The oracle-efficient online estimation algorithm in
\cref{subsec:opt-err-bound}, \cref{alg:weighted-maj-vote}, is somewhat
specialized to finite classes. In this section, we provide a more general
  approach to designing oracle-efficient algorithms based on \emph{delayed
  online learning}, and use it to derive
  a characterization of oracle-efficient learnability for
  classification with infinite classes $\cF$.

For the results in this section, we assume that $\cZ$ and \jqedit{$\divSymbol$ are convex}, which covers regression
and conditional density estimation; variants of our result for classification are
given in \cref{app:delayed}.
  
\paragraph{Online learning with delayed feedback}
\label{subsubsec:delayed-reference}

Before introducing our algorithm, we first introduce an abstract
\emph{delayed online learning} framework
\citep{weinberger2002delayed,mesterharm2007improving,joulani2013online,quanrud2015online}. In
our framework, the learner is given a class
$\cF\subseteq\cZ^{\cX}$. Their goal is to choose a sequence of
\funNames $\fbar\ind{1},\ldots,\fbar\ind{T}$ that minimizes regret
against the class $\cF$ for an adversarially chosen sequence of loss
functions $\ell\ind{1},\ldots,\ell\ind{T}$, with the twist being that
the loss $\ell\ind{t}$ is not revealed immediately at step $t$, and
instead becomes
available at step $t+N$ for a \emph{delay} parameter $N\in\bbN$.

In more detail, the interaction between the learner and the
environment proceeds as follows:
\begin{enumerate}[label=$\bullet$,leftmargin=5mm]
    \item For $t=1,\dots,T$:
    \begin{enumerate}[label=$\bullet$,leftmargin=5mm]
      \item The learner picks
        $\fbar\ind{t}\sim\mu\ind{t}\in\Delta(\cZ^{\cX})$.
      \item Learner incurs loss $\ell\ind{t}(\fbar\ind{t})$ and adversary reveals loss function
        $\ell\ind{t-N}:\cZ^{\cX}\to\brk{0,1}$ for round $t-N$ (if
        $t< N+1$, nothing is revealed).
    \end{enumerate}
\end{enumerate}
The goal of the learner is to minimize regret in the sense that
\begin{align*}
\RDOL(T,N,\gamma) \ldef  \sum\limits_{t=1}^{T} \En_{\fbar\ind{t}\sim \mu\ind{t}}\brk*{\ell\ind{t}(\fbar\ind{t})}
  -  \gamma\cdot{}\min_{f\in\FunSpace}\sum\limits_{t=1}^{T}\ell\ind{t}(f)
\end{align*}
is small, where $\gamma\geq{}1$ is a parameter. For $\gamma=1$, this definition coincides with the standard notion of regret in online learning (e.g., \citet{cesa2006prediction}), but allowing for $\gamma>1$ will prove useful for our technical results.

\paragraph{Algorithm and online estimation error bound}

\begin{algorithm}[tp]
\caption{Reduction from \framework to Online Learning with Delayed Feedback}
\label{alg:reduction-to-delayed-OL}
\begin{algorithmic}[1]
  \State\textbf{input:} Offline estimation oracle $\Orc$ with
  parameter $\offgan\geq{}0$, delay parameter $N\in\bbN$, delayed
  online learning algorithm $\AlgDOL$ for class $\cF$.
\For{$t=1,2,\dots,T$}
\State Receive $\fhat\ind{t}=\Orc\ind{t}(x\ind{1},\ldots,x\ind{t-1},y\ind{1},\ldots,y\ind{t-1})$.
\If{$t> N$}
\State Let $\ftil\ind{t-N} \ldef \frac{1}{N} \sum\limits_{i=t-N+1}^{t}
\fhat\ind{i} $. 
\State \multiline{Let
$\ell\ind{t-N}(f) \ldef
\loss{\ftil\ind{t-N}(\covar\ind{t-N})}{f(\covar\ind{t-N})}$ and pass
$\ell\ind{t-N}(\cdot)$ to $\AlgDOL$ as the delayed feedback.}
\EndIf
\State \multiline{Let
$\mu\ind{t}=\AlgDOL\ind{t}(\ell\ind{1},..,\ell\ind{t-N})$ be the
delayed online learner's prediction distribution.}
\State Predict with $\dfedit{\fbar\ind{t}\sim}\mu\ind{t}$ and receive $\covar\ind{t}$.
\EndFor
\end{algorithmic}
\end{algorithm}

\cref{alg:reduction-to-delayed-OL} describes our reduction from
oracle-efficient online estimation to delayed online learning. In
addition to an offline oracle $\Orc$, the algorithm takes as input a
\emph{delay parameter} $N\in\bbN$ and a delayed online learning
algorithm $\AlgDOL$ for the class $\cF$
(\cref{alg:reduction-to-delayed-OL} does not explicitly take the class
$\cF$ as an argument, as the algorithm only implicitly makes use of $\cF$
through $\AlgDOL$).

The basic premise behind \cref{alg:reduction-to-delayed-OL} is that
for any sequence of consistent offline estimators, we can average to
improve the predictions. Consider a step $t\in\brk{T}$. Suppose
$\fhat\ind{1},\ldots,\fhat\ind{T}$ are produced by an offline oracle
$\Orc$ with parameter $\offgan\geq{}0$, \jqedit{where we augment the sequence by setting $\fhat\ind{T+s} = \fhat\ind{T}$ for all $s\in \bbN$.} 
Then we can use an argument based
on convexity (cf. proof of \cref{thm:delayed_reduction}) to show that for any
$N\in\bbN$, the averaged \funNames
\begin{equation}
  \label{eq:avg}
\ftil\ind{t}\ldef{}\frac{1}{N}\sum_{i=t+1}^{t+N}\fhat\ind{i}
\end{equation}
satisfies
\jqedit{
\begin{align*}
  \sum_{t=1}^{T}\loss{\ftil\ind{t}(\covar\ind{t})}{\fstar(\covar\ind{t})} &\leq{}N+
\frac{1}{N}\sum_{t=1}^{T-N}\sum_{i=t+1}^{t+N}\loss{\fhat\ind{i}(\covar\ind{t})}{\fstar(\covar\ind{t})}\\
  &=N+
    \frac{1}{N}\sum_{t=2}^{T}\sum_{i<t}\loss{\fhat\ind{t}(\covar\ind{i})}{\fstar(\covar\ind{i})}\leq{} N+ \offgan T/N .
\end{align*}}%
In particular, as we increase $N$, the quality of the predictions
increases, and we achieve sublinear estimation error as soon as
$N=\omega(T)$. Of course, the catch here is that $\ftil\ind{t}$
depends on the predictions of future estimators, and cannot be
computed at step $t$. However, $\ftil\ind{t}$ \emph{can} be computed
at step $t+N+1$, with a delay of $N$. This leads us to appeal to delayed
online learning. In particular, at each step $t\geq{}N+1$,
\cref{alg:reduction-to-delayed-OL} proceeds as follows. Using the new
offline estimator $\fhat\ind{t}$ from $\Orc$, the algorithm computes
the averaged estimator $\ftil\ind{t-N} \ldef \frac{1}{N} \sum\limits_{i=t-N+1}^{t}
\fhat\ind{i}$ corresponding to the estimator in \cref{eq:avg} for step
$t-N$. The algorithm then defines a loss function
\[
\ell\ind{t-N}(f) \ldef
\loss{\ftil\ind{t-N}(\covar\ind{t-N})}{f(\covar\ind{t-N})}
\]
and feeds it into the delayed online learning algorithm $\AlgDOL$ as
the feedback for step $t-N$. Finally,
\cref{alg:reduction-to-delayed-OL} uses the prediction distribution
$\mu\ind{t}$ produced by $\AlgDOL$ to sample the final estimator
$\fbar\ind{t}$, then proceeds to the next step. Our main theorem shows
that as long $\AlgDOL$ achieves low regret for delayed online
learning, this strategy leads to low online estimation error.

\begin{restatable}[Reduction from oracle-efficient online estimation
  to delayed online learning]{theorem}{OELtoDOL}
  \label{thm:delayed_reduction}
  Let $\divSymbol$ be any convex, metric-like loss. Suppose we run \cref{alg:reduction-to-delayed-OL} with delay
  parameter $N\in\bbN$ and a delayed online learning algorithm
  $\AlgDOL$ for the class $\cF$. Then for all $\gamma\geq{}1$,
  \cref{alg:reduction-to-delayed-OL} ensures that
  \begin{align}
    \label{eq:delayed_regret}
    \En\brk[\big]{\EstOnD(T)} 
    \leq \bigoh\prn*{
    C_\divSymbol\gamma(N+\offgan{}T/N) +\RDOL(T,N,\gamma)},
  \end{align}
    with any offline oracle $\Orc$ with parameter $\offgan\geq{}0$, 
  where
    \begin{align}
      \label{def:rdol}
      \RDOL(T,N,\gamma) \ldef{} \sum\limits_{t=1}^{T}\En_{\fbar\ind{t}\sim{}\mu\ind{t}}\brk*{\ell\ind{t}(\fbar\ind{t})}
    - \gamma\cdot{}\min_{f\in\FunSpace}\sum\limits_{t=1}^{T}\ell\ind{t}(f)
    \end{align}
    is the regret of $\AlgDOL$ for the sequence of losses constructed
    in \cref{alg:reduction-to-delayed-OL}.
\end{restatable}

The parameter $N$ controls a sort of bias-variance tradeoff in \cref{thm:delayed_reduction}.
The first term in
\cref{eq:delayed_regret} (corresponding to the bias of the averaged
estimators) decreases with the delay $N$, while the
second term (corresponding to the regret of $\AlgDOL$) increases; the optimal choice for $N$ will balance
these terms. To make this concrete, we revisit finite classes as a warmup.

\paragraph{Example: Finite classes}

Delayed online learning is well-studied, and optimal algorithms are
known for many classes of interest
\citep{weinberger2002delayed,mesterharm2007improving,joulani2013online,quanrud2015online}. The
following standard result (a proof is given in \cref{app:delayed} for
completeness) gives a delayed regret bound for arbitrary finite classes.

\begin{restatable}{lemma}{DOLToOL}
  \label{prop:delayed}
  Consider the delayed online learning setting with a delay parameter
  $N$. There exists an algorithm that achieves
    \begin{align*}
\RDOL(T,N,2) =  \sum\limits_{t=1}^{T}\En_{\fbar\ind{t}\sim{}\mu\ind{t}}\brk*{\ell\ind{t}(\fbar\ind{t})}
    - 2\cdot{}\min_{f\in\FunSpace}\sum\limits_{t=1}^{T}\ell\ind{t}(f) \leq{} 2N\cdot{} \log |\FunSpace|   
    \end{align*}
    for any sequences of losses $\ell\ind{1},\ldots,\ell\ind{T}\in\brk{0,1}$.
\end{restatable}

Combining \cref{thm:delayed_reduction} with \cref{prop:delayed}, we
can obtain the following upper bound for oracle-efficient online estimation.
\begin{restatable}[Oracle-efficient online estimation for finite
  classes via delayed online learning]{corollary}{RedToOLUpperBound}
  \label{cor:reduction-to-OL-upper-bound}
  Consider an arbitrary instance $(\CovarSpace,\ObsSpace,
  \ValSpace,\Kernel,\FunSpace)$ and \mlike loss $\divSymbol$, and
  assume $\cZ$ is convex.  By choosing $\AlgDOL$ as in
  \cref{prop:delayed}, \cref{alg:reduction-to-delayed-OL} ensures that
  for any $N\geq{}1$,
  \begin{align*}
  \En\brk[\big]{\EstOnD(T)}  
  \leq \bigoh\prn*{
    C_\divSymbol(N+\offgan{}T/N)+N\cdot \log |\FunSpace|}.
  \end{align*}
  with any offline oracle $\Orc$ with parameter $\offgan$.
\end{restatable}

By choosing $N = \sqrt{ \frac{C_\divSymbol\offgan \cdot T}{ C_\divSymbol+\log |\FunSpace|}  } \vee 1$,  \cref{cor:reduction-to-OL-upper-bound} gives an upper bound of $\bigoh\prn*{\sqrt{C_\divSymbol\offgan{}(C_\divSymbol+\log\abs{\FunSpace})\cdot T} + \log |\FunSpace|}$.
While the rate in \cref{cor:reduction-to-OL-upper-bound}, is worse
than \cref{thm:woe-upper-bound} in terms of dependence on $T$, the
reduction has two advantages:
(1) It does require a-priori knowledge of the offline estimation
parameter $\offgan$: If we choose $N=\sqrt{ \frac{C_\divSymbol \cdot
    T}{C_\divSymbol+\log |\FunSpace|}  } \vee 1$,
\cref{cor:reduction-to-OL-upper-bound} obtains $\bigoh ( (\offgan+1) (
\sqrt{C_\divSymbol(C_\divSymbol+\log\abs{\FunSpace})\cdot T} + C_\divSymbol+\log |\FunSpace| ) )$;
(2) The reduction by \cref{alg:reduction-to-delayed-OL} is more
flexible, and allows for guarantees beyond finite classes, as we now illustrate.

\arxiv{\subsubsection{Application: Characterization of Oracle-Efficient
    Learnability for Classification}}
\colt{\subsection{Characterization of Oracle-Efficient
    Learnability for Classification}}
\label{sec:infinite}

As an application of \cref{alg:reduction-to-delayed-OL} and
\cref{thm:delayed_reduction}, we give a characterization for
oracle-efficient learnability in the \framework framework. To state
the result, we define $\Ldim(\cF)$ as the Littlestone dimension for a
binary function class $\cF$ (e.g., \citet{ben2009agnostic}).

\begin{restatable}[Characterization of oracle-efficient learnability
  for binary classification]{theorem}{BinaryLearnability}
  \label{thm:binary-loss-learnability}
Consider a binary classification instance $(\CovarSpace,\ObsSpace,
\ValSpace,\Kernel,\FunSpace)$ with $\ValSpace =\ObsSpace =\set{0,1}$,
$\divSymbol = \Dbinshort$ and $\Kernel(\val)=\indic_\val$. For any class
$\cF$ and $\offgan\geq{}0$, there exists an \estmem algorithm that
achieves online estimation error \jqedit{$\bigoh( \sqrt{\offgan \Ldim(\cF)\cdot T\log T} + \Ldim(\cF)\log T )$. On the other hand, in the worst-case any algorithm must suffer at least $\Omega(\Ldim(\cF))$ online estimation error. }
\end{restatable}
The main idea behind this result is to show that we can create a
delayed online learner for the reduction in
\cref{alg:reduction-to-delayed-OL} that achieves low regret for Littlestone
classes.\footnote{Formally, to handle the fact that $\cZ=\crl{0,1}$ is
  not convex, this result requires a slight modification to Line 4 in
  \cref{alg:reduction-to-delayed-OL} that replaces the average with a
  majority vote. See the proof for details.}

}

\section{Computational Complexity of Oracle-Efficient Online Estimation}
\label{sec:comp-results}

In this section, we turn our attention to the computational complexity of oracle-efficient online estimation in the \framework framework. In \cref{subsec:comp-hardness}, we show (\cref{thm:computation-lower-bound}) that in general, it is not possible to transform black-box offline estimation algorithms into online estimation algorithms in a computationally efficient fashion. Then, in \cref{subsec:reductions}, we provide a more fine-grained perspective, showing (\cref{thm:reduction-to-ODE}) that for conditional density estimation, online estimation in the \framework framework is no harder computationally than online estimation with unrestricted algorithms.

\subsection{Computational Hardness of Oracle-Efficient Estimation}
\label{subsec:comp-hardness}

Our main upper bounds (\cref{subsec:opt-err-bound}) show that online estimation error $\EstOnD(T)\leq{}O((\offgan+1)\log |\FunSpace|)$ can be achieved in an oracle-efficient fashion for any finite class $\cF$, but the algorithm (\cref{alg:weighted-maj-vote}) is not computationally efficient. We now show that this is fundamental: There exist classes $\cF$ for which offline estimation can be performed in polynomial time, yet no oracle-efficient algorithm running in polynomial-time algorithm can achieve sublinear online estimation error.

\paragraph{Computational model}
To present our results, we must formalize a computational model for oracle-efficient online estimation, and in particular, define a notion of \emph{input length} for oracle-efficient online algorithms. To do so, we restrict our attention to noiseless binary classification, and consider a sequence of classification instances indexed by $n\in\bbN$, with $\cZ=\cY=\crl{0,1}$, $\cX_n\ldef{}\crl{0,1}^{n}$, $\Kernel(z)=\indic_z$, and indicator loss $\Dbin{\cdot}{\cdot}$. We consider a sequence of classes $\cF_n$ that have polynomial description length\arxiv{ in the sense that}\colt{, i.e.} $\log\abs{\cF_n}$ \jqedit{is polynomial in $n$}, so that $f\in\cF_n$ can be described in $\poly(n)$ bits. In particular, we assume that $f\in \cF_n$ is represented as a Boolean circuit of size $\poly(n)$ so that $f(x)$ can be computed in $\poly(n)$ time for $x\in\cX_n$; we refer to such sequences as \emph{polynomially computable}.\loose

To allow for offline estimators that are improper, we assume that for all $t$ and all sequences $(x\ind{1},y\ind{1}),\ldots,(x\ind{T},y\ind{T})$, the output $\fhat\ind{t}:\crl{0,1}^{n}\to\crl{0,1}$ returned by $\Orc\ind{t}(x\ind{1},\ldots,x\ind{t-1},y\ind{1},\ldots,y\ind{t-1})$ is a Boolean circuit of size $\poly(n)$; we refer to such oracles as having $\poly(n)$-\emph{output description length}.\footnote{See, e.g., \citet{arora2009computational}, Definition 6.1} Likewise, to allow the online estimation algorithm itself to be improper and randomized, we restrict to algorithms for which computing $\fbar\ind{t}(x)$ for $\fbar\ind{t}\sim\mu\ind{t}$ can be implemented as $\fbar\ind{t}(x,r)$ for a random bit string $r\sim\unif(\crl{0,1}^{B})$, where $B=\poly(n)$; we refer to the online estimator as having $\poly(n)$-output description length if $\fbar\ind{t}(\cdot,\cdot)$ is itself a Boolean circuit of size $\poly(n)$. We refer to the online estimation algorithm \emph{polynomial time} if it runs in time $\poly(n)$ for any sequence of inputs $t$,
 $x\ind{1},\ldots,x\ind{t-1}$, and $\fhat\ind{1},\ldots,\fhat\ind{t-1}$, and has $\poly(n)$-output description length.\footnote{The precise computational model for runtime under consideration (e.g., Turing machines or Boolean circuits) does not change the nature of our results.}

\fakepar{Main lower bound} Our main computational lower bound is as follows.%

\newcommand{\refcomplbvar}{\cref{thm:computation-lower-bound}\xspace}

\begin{restatable}[Computational lower bound for \framework]{theorem}{complbimproper}
  \label{thm:computation-lower-bound}
  Assume the existence of one-way functions.\footnote{Existence of one-way functions is a standard and widely believed complexity-theoretic assumptions, which forms the basis of modern cryptography \citep{goldreich2005foundations}.} 
  There exists a sequence of polynomially computable classes $(\cF_1,\cF_2,\dots,\cF_n,\dots)$, along with a sequence of $\poly(n)$-output description length offline oracles with $\offgan=0$ associated with each $\cF_n$, such that for any fixed polynomials $p,q:\bbN\to\bbN$ and all $n\in \bbN$ sufficiently large, any oracle-efficient online estimation algorithm with runtime bounded by $p(n)$ must have \arxiv{\[\En[\EstOnD(T)]\geq{}T/4\]}\colt{$\En[\EstOnD(T)]\geq{}T/4$} for all $1\leq{}T\leq q(n)$.
At the same time, there exists an inefficient algorithm that achieves $\En[\EstOnD(T)]\leq{}\bigoh(\sqrt{n})$ for all $T\in\bbN$.

\end{restatable}

Informally, \cref{thm:computation-lower-bound} shows that there exist a class $\cF$ and offline estimation oracles $\Orc$ for which no oracle-efficient online estimation algorithm that runs in time
\[
\poly(\desclength(\cX), \desclength(\cF), \mathrm{max}_t\,\desclength(\fhat\ind{t}), T)
\]
can achieve sublinear estimation error, where $\desclength(\cX)$ and $\desclength(\cF)$ denote the number of bits required to describe $x\in\cX$ and $f\in\cF$, and $\desclength(\fhat\ind{t})$ denotes the size of the circuit required to compute $\fhat\ind{t}(x)$. Yet, low online estimation error \emph{can} be achieved by an inefficient algorithm, and there exist efficient offline estimators with $\offgan=0$ as well.  The result is essentially a corollary of \citet{blum1994separating}; we refer to \cref{app:comp-hardness} for the detailed proof\colt{.}\arxiv{, which shows how to formally map the construction of \citet{blum1994separating} onto the \framework framework.} We mention in passing that \citet{hazan2016computational} also give lower bounds against reducing online learning to offline oracles, but in a somewhat different computational model; see \cref{sec:related} for detailed discussion.

\cref{thm:computation-lower-bound} is slightly disappointing, since one of the main motivations for studying oracle-efficiency is to leverage offline estimators as a computational primitive. Combined with our results in \cref{sec:stats-results}, \cref{thm:computation-lower-bound} shows that even though it is possible to be oracle-efficient \arxiv{in the \framework framework} information-theoretically, it is not possible to achieve this computationally. Nonetheless, we are optimistic that our abstraction can (i) aid in designing computationally efficient algorithms for learning settings beyond online estimation, and (ii) continue to serve as a tool to formalize lower bounds against natural classes of algorithms, as we have done here; see \cref{sec:discussion} for further discussion.%

\begin{remark}
\dfedit{\refcomplbvar relies on the fact that the offline estimator may be improper (i.e., $\fhat\ind{t}\notin\cF$).} An interesting open problem is whether one can attain $\poly (\log |\FunSpace_n|)\cdot o(T)$ online estimation error with runtime $\poly (\log |\FunSpace_n|)$ given a \emph{proper} offline estimation oracle with parameter $\offgan=0$.
\end{remark}

\subsection{Conditional Density Estimation: Computationally Efficient Algorithms}
\label{subsec:reductions}

In spite of the negative result in the prequel, which shows that efficient computation in the \framework framework is not possible in general, we can provide a more fine-grained perspective on the computational complexity of oracle-efficient estimation for the problem \emph{conditional density estimation}, a general task which subsumes classification and regression, and has immediate applications to reinforcement learning and interactive decision making \citep{foster2021statistical,foster2023tight}.

\fakepar{Conditional density estimation}
Recall that conditional density estimation \citep{yang1999information,bilodeau2023minimax} is the special case of the online estimation framework in \cref{sec:intro} in which $\cX$ and $\cY$ are arbitrary, $\cZ=\Delta(\ObsSpace)$, and the kernel is\arxiv{ given by} $\Kernel(z) = z$; that is sampling $y\sim{}\Kernel(\fstar(x))$ is equivalent to sampling $y\sim{}\fstar(x)$. For the loss, we use squared Hellinger distance: $\DhelsX{\big}{f(x)}{\fstar(x)}= \frac{1}{2} \int\prn[\big]{\sqrt{f(y\mid{}x)}-\sqrt{\fstar(y\mid{}x)}}^{2}$; with online estimation error given by $\EstOnHels(T) = \sum_{t=1}^{T}\DhelsX{\big}{\fhat\ind{t}(\covar\ind{t})}{\fstar(\covar\ind{t})}$.

\noindent\textbf{Base algorithm.}
Our result is based on a reduction. We assume access to a base algorithm $\AlgCDE$ for online estimation in the Conditional Density Estimation (CDE) framework, which is \emph{unrestricted} in the sense that it is not necessarily oracle-efficient. That is, $\AlgCDE$ can directly use the full data stream $(x\ind{1},y\ind{1}),\ldots,(x\ind{t-1},y\ind{t-1})$ at step $t$. For parameters $\RCDE(T)$ and $C_{\FunSpace}\geq{}1$, we assume that for any \arxiv{target parameter }$\fstar\in\cF$, the base algorithm $\AlgCDE$ ensures that \dfedit{for all $\delta\in(0,e^{-1})$}, with probability at least $1-\delta$,\loose
\begin{align}
  \label{eq:online_base}
  \EstOnHels(T)
  \leq{}\RCDE(T) + C_{\FunSpace}\cdot \log (1/\delta).
\end{align}
 We define the total runtime for $\AlgCDE$ across all rounds as $\Time(\cF,T)$.

\paragraph{Main result}
Our main result shows that any algorithm $\AlgCDE$ satisfying the guarantee above can be transformed into an oracle-efficient algorithm with only polynomial blowup in runtime. For technical reasons, we assume that $V\ldef e~\vee~\sup_{\fun,\fun'\in\cF,\covar\in\cX,\obs\in\cY} \frac{\fun(\obs|\covar)}{\fun'(\obs|\covar)}$ is bounded; our sample complexity bounds scale only logarithmically with respect to this parameter. \dfedit{In addition, we assume that all $f\in\cF$ and $x\in\cX$ have $\bigoh(1)$ description length, and that one can sample $y\sim{}f(x)$ in $\bigoh(1)$ time.\loose}

\begin{restatable}{theorem}{RedToODE}
  \label{thm:reduction-to-ODE}
  Let $\AlgCDE$ be an arbitrary (unrestricted) online estimation algorithm that satisfies \cref{eq:online_base} and has runtime $\Time(\cF,T)$. Then for any $N\in\bbN$, there exists an oracle-efficient online estimation algorithm that achieves estimation error
    \begin{align*}
      \En\brk[\big]{\EstOnHels(T)}  \leq \bigoht \prn*{C_{\FunSpace}\log V\cdot \offgan T/N +  N \cdot\prn*{ \RCDE(T)+ C_{\FunSpace}\log V}}
    \end{align*}
  with runtime $\poly(\Time(\cF,T),\log|\FunSpace|, \log |\CovarSpace|,T)$, where $\offgan\geq{}0$ is the offline estimation parameter. The \arxiv{randomization }distributions $\mu\ind{1},\ldots,\mu\ind{T}$ produced by the algorithm have support size $\poly(\log|\FunSpace|,\log|\CovarSpace|, T)$. \arxiv{\\}
  As a special case, if the online estimation guarantee for the base algorithm holds with $\RCDE(T) \leq  C_{\FunSpace}' \log T$ for some problem-dependent constant $C_{\FunSpace}'\geq{}1$, then by choosing $N$ appropriately, we
  \colt{achieve $\En\brk[\big]{\EstOnHels(T)}\leq\bigoht\prn*{(C_{\FunSpace}(C_{\FunSpace} + C_{\FunSpace}') \offgan)^{1/2} \log V \cdot  T^{1/2} + (C_{\FunSpace} + C_{\FunSpace}') \log V  }$.}\arxiv{achieve
  \begin{align*}   
    \En\brk[\big]{\EstOnHels(T)} \leq \bigoht\prn*{(C_{\FunSpace}(C_{\FunSpace} + C_{\FunSpace}') \offgan)^{1/2} \log V \cdot  T^{1/2} + (C_{\FunSpace} + C_{\FunSpace}') \log V  }.
  \end{align*}}
\end{restatable}
Note that the estimation error bound in \cref{thm:reduction-to-ODE} is sublinear whenever the rate $\RCDE(T)$ is. This implies that for squared Hellinger distance, online estimation in the \framework framework is no harder computationally than online estimation with arbitrary, unrestricted algorithms.

\arxiv{
Let us make a few technical remarks. First, A limitation in \cref{thm:reduction-to-ODE} is that we require the base algorithm to satisfy a high-probability online estimation error bound, with the dependence on the failure probability $\delta$ in the estimation error bound scaling as $\log(\delta^{-1})$. It would be interesting to understand if this can be relaxed. Second, we expect that the dependence on the problem parameters $T$, $C_{\cF}$, $C_{\cF}'$, and so on in the estimation error bound can be improved.}

\colt{
  The proof of \cref{thm:reduction-to-ODE} is algorithmic, and is based on several layers of reductions. The main reason why the result is specialized to conditional density estimation is as follows: If we have an estimator $\fhat$ for which $\DhelsX{\big}{\fhat(x)}{\fstar(x)}$ is small for some $x$, we can simulate $y\sim{}\fstar(x)$ up to low statistical error by sampling $y\sim{}\fhat(x)$ instead, as $\fhat$ and $\fstar$ are close in distribution. This allows us to implement a scheme based on simulating outcomes and feeding them to the base algorithm.
  }

\arxiv{
\paragraph{Overview of proof for \creftitle{thm:reduction-to-ODE}}

The proof of \cref{thm:reduction-to-ODE} is algorithmic, and is based on several layers of reductions.
\begin{itemize}
\item First, using the scheme in \cref{sec:delayed}, we reduce the problem of oracle-efficient online estimation to delayed online learning with the loss function $\ell\ind{t-N}(f) =   \divX{\big}{\ftil\ind{t-N}(\covar\ind{t-N})}{f(\covar\ind{t-N})}$ defined in \cref{alg:reduction-to-delayed-OL}, where $\ftil\ind{t-N} = \frac{1}{N} \sum_{i=t-N+1}^{t} \fhat\ind{i}$ is an average of offline estimators and $N\in\bbN$ is a delay parameter.
\item Then, using a standard reduction \citep{weinberger2002delayed,mesterharm2007improving,joulani2013online,quanrud2015online}, we reduce the delayed online learning problem above to a sequence of $N$ non-delayed online learning problems, with the same sequence of loss functions; both this and the preceding step are computationally efficient.
\item To complete the reduction, we argue that the base algorithm can be used to solve the online learning problem above in an oracle-efficient fashion. To do this, we simulate interaction with the environment by sampling fictitious outcomes $y\ind{t}\sim{}\ftil\ind{t}(x\ind{t})$ from the averaged offline estimators and passing them into the base algorithm. We the fictitious outcomes approximate the true outcomes well through a change-of-measure argument.
\end{itemize}
Combining the above, we conclude that given any base algorithm that efficiently performs online estimation with outcomes sampled from the true model $\fstar$, we can efficiently construct a computationally efficient and oracle-efficient algorithm.
We defer the details of the proof to \cref{app:hellinger_reduction}.

The main reason why this result is specialized to conditional density estimation is as follows: If we have an estimator $\fhat$ for which $\DhelsX{\big}{\fhat(x)}{\fstar(x)}$ is small for some $x$, we can simulate $y\sim{}\fstar(x)$ up to low statistical error by sampling $y\sim{}\fhat(x)$ instead, because $\fhat$ and $\fstar$ are close in distribution. This allows us to implement a scheme based on simulating outcomes and feeding them to the base algorithm. We expect the argument can be extended to more general settings in which (i) $\div{\cdot}{\cdot}$ guarantees closeness in distribution, and (ii) the kernel $\Kernel(\cdot)$ is known to the learner.
}

\arxiv{
\section{Application to Interactive Decision Making}
\label{sec:application}

In this section, we apply our techniques for oracle-efficient online
estimation to the \emph{\Framework} (\FrameworkShort) framework for
interactive decision making introduced by
\cite{foster2021statistical}. First, in
\cref{subsec:application-to-decision-making}, we use our reductions to
provide offline oracle-efficient algorithms for interactive decision
making. Then, in \cref{subsec:positive-results-w-struct}, we focus on
reinforcement learning and show that
it is possible to bypass the impossibility results for memoryless
oracle-efficient algorithms
(\cref{prop:separation_memoryless_improper}) for instances
$(\CovarSpace,\ObsSpace, \ValSpace,\Kernel,\FunSpace)$ corresponding
to Markov decision processes that satisfy a structural property known
as coverability.

\subsection{Offline Oracle-Efficient Algorithms for Interactive Decision Making}
\label{subsec:application-to-decision-making}
\newcommand{\gmstar}{g\sups{\Mstar}}

In this section, we introduce the setting of \Framework (\FrameworkShort) and the applications of our results to this setting.
\paragraph{\Framework (\FrameworkShort)} The \FrameworkShort framework
\citep{foster2021statistical} captures a large class of interactive
decision making problems (e.g. contextual bandits and reinforcement learning). In this framework, the learner is given access to a \emph{model class} $\cM$ that contains an unknown true model $\Mstar:\ActSpace\to\Delta(\cR\times\OObsSpace)$, where $\ActSpace$ is the \emph{decision space}, $\RSpace\subseteq\bbR$ is the \emph{reward space} and $\OObsSpace$ is the \emph{observation space}. 
Then the interaction between the learner and the environment proceeds in $T$
rounds, where for each round $t=1,\ldots,T$:
\begin{enumerate}
\item The \learner selects a \emph{decision} $\act\ind{t}\in\Act$.
  \item Nature selects a \emph{reward} $r\ind{t}\in\RewardSpace$ and
  \emph{observation} $\oobs\ind{t}\in\OObsSpace$ based on the decision,
  where the pair
  $(r\ind{t}, \oobs\ind{t})$ is drawn independently from the unknown
  distribution $\Mstar(\act\ind{t})$. The reward and
  observation is then observed by the learner. \looseness=-1
\end{enumerate}
Let $\gm(\act)\ldef{}\En^{\sss{M},\act}\brk*{r}$ denote the
mean reward function and $\pim\ldef{}\argmax_{\act\in\Act}\gm(\act)$ denote the decision with
the greatest expected reward for $M$. 
The \learner's performance is evaluated in terms of \emph{regret} to the optimal
decision for $\Mstar$:
\begin{equation}
  \label{eq:regret}
  \RegDM(T)
  \ldef \sum\limits_{t=1}^{T}\En_{\act\ind{t}\sim{}p\ind{t}}\brk*{\gmstar(\pimstar) - \gmstar(\act\ind{t})},
\end{equation}
where $p\ind{t}\in\Delta(\Act)$ is the
learner's distribution over decisions at round $t$.

\paragraph{Background: Reducing \FrameworkShort to online estimation}
Any DMSO class $(\cM,\Pi,\cO)$ induces an instance $(\CovarSpace,\ObsSpace, \ValSpace,\Kernel,\FunSpace)$ of
the estimation framework in \cref{sec:intro} as follows. We associate
$\cF=\cM$, $\cX=\Pi$, $\cY=\cO\times\cR$, $\cZ=\Delta(\cO\times\cR)$, and
$\Kernel(\Mstar(\pi))=\Mstar(\pi)$. That is, we have a conditional
density estimation problem in which the covariates are decisions
$\pi\in\Pi$ and the outcomes are observation-reward pairs drawn from
the underlying model
$\Mstar(\pi)$. In particular for a sequence of decisions
$\pi\ind{1},\ldots,\pi\ind{T}$ and a sequence of estimators
$\Mhat\ind{1},\ldots,\Mhat\ind{T}$, we define the online estimation
error for a loss $\divSymbol$ as
\begin{align}
  \label{eq:online_est_dmso}
  \EstOnD(T)
  = \sum_{t=1}^{T}\divX{\big}{\Mhat\ind{t}(\pi\ind{t})}{\Mstar(\pi\ind{t})}.
\end{align}
We refer to any algorithm $\AlgOn$ that ensures that
$\EstOnD(T)\leq\ongan$ almost surely given access to
$\crl*{(\pi\ind{t},o\ind{t},r\ind{t})}_{t=1}^{T}$ with
$(o\ind{t},r\ind{t})\sim\Mstar(\pi\ind{t})$ as an \emph{online
  estimation oracle} with parameter $\ongan$.

\citet{foster2021statistical,foster2023tight} give an algorithm,
\emph{Estimation-to-Decisions} (\etd), that provides bounds on the regret in
\cref{eq:regret} given access to an online estimation oracle
$\AlgOn$. The algorithm is (online) oracle-efficient and memoryless,
in the sense that the decision $\pi\ind{t}$ at each step $t$ is a
measurable function of the oracle's output $\Mhat\ind{t}$. To restate
their result, we define the \emph{Decision-Estimation Coeficient} for
the class $\cM$ as 
\begin{align}
  \label{def:dec}
  \compgen(\cM,\Mbar) =
  \inf_{p\in\Delta(\Act)}\sup_{M\in\cM}\En_{\act\sim{}p}\biggl[\gm(\pim)-\gm(\pi)
  -\gamma\cdot\Dgen{M(\act)}{\Mbar(\act)}
  \biggr]
\end{align}
for a reference model $\Mbar$ and any losses
$\divSymbol$. We further define $\compgen(\cM) = \sup_{\Mbar\in
  \cM} \compgen(\cM,\Mbar)$.
With this notation, the main regret bound for \etd is as follows.

\begin{proposition}[Theorem 4.3 of \cite{foster2021statistical}]
  \label{thm:upper_general_distance}
For any model class $(\cM,\Pi,\cO)$ and \mlike loss $\divSymbol$, any
$\gamma>0$, and any online estimation oracle $\AlgOn$ with parameter
$\ongan>0$, the \etd algorithm ensures that
\begin{equation}
  \label{eq:upper_general_distance}
\RegDM(T) \lesssim{}
\sup_{\mu\in\Delta(\cM)}\compgen(\cM,\mu)\cdot{}T + \gamma\cdot\ongan.
\end{equation}
\end{proposition}

For the case of squared Hellinger distance where $\divSymbol = \Dhelshort^2$,
the Decision-Estimation Coefficient
\jqedit{$\compgen(\cM)$} was shown to be a
\emph{lower bound} on the minimax optimal regret for any class
$\cM$. Hence, \cref{thm:upper_general_distance} shows that it is
possible to achieve near-optimal regret for any interactive decision
making problems whenever an online estimation oracle is
available. However, it was unclear whether similar results could be
achieved based on offline estimation oracles.

\begin{algorithm}[tp]
    \setstretch{1.3}
     \begin{algorithmic}[1]
       \State \textbf{parameters}:
       \Statex[1] Offline estimation oracle $\Orc$ with parameter $\offgan$.
       \Statex[1] Oracle-efficient online estimation algorithm $\AlgRed$.
       \Statex[1] Exploration parameter $\gamma>0$.
  \For{$t=1, 2, \cdots, T$}
  \State Compute estimate $\Mhat\ind{t} = \Orc\ind{t}\prn[\Big]{ \act\ind{1},\dots,\act\ind{t-1},\oobs\ind{1},...,\oobs\ind{t-1} }$.
  \State Feed $\Mhat\ind{t}$ to \framework algorithm $\AlgRed$ and obtain $\mu\ind{t}$.
  \State Let 
  \begin{equation}
    \label{eq:dec-policy}
    p\ind{t}=\argmin_{p\in\Delta(\Act)}\sup_{M\in\cM} \En_{\pi\sim p,\Mhat\sim \mu\ind{t}}\brk*{\gm(\pim) - \gm(\pi) -\gamma\cdot \div{\Mhat(\pi)}{M(\pi)} }.
  \end{equation}
\State{}Sample decision $\act\ind{t}\sim{}p\ind{t}$ and
$\oobs\ind{t}\sim{}\Mstar(\pi\ind{t})$ and feed $\pi\ind{t}$ to \framework algorithm $\AlgRed$.
\EndFor
\end{algorithmic}
\caption{Estimation to Decisions Meta-Algorithm with Offline Oracles (\etdoff)}
\label{alg:etd-off}
\end{algorithm}

\paragraph{Making $\etd$ offline oracle-efficient}
\cref{alg:etd-off} ($\etdoff$) invokes the $\etd$ algorithm of
\citet{foster2021statistical} with any oracle-efficient online
estimation algorithm $\AlgRed$ (which can be any of the algorithms we provide,
e.g. \cref{thm:woe-upper-bound,thm:reduction-to-ODE}), along with an
offline estimation oracle $\Orc$, to provide offline oracle-efficient
guarantees for interactive decision making. Invoking the algorithm
with Version Space Averaging (via \cref{thm:woe-upper-bound}) leads to
the following corollary.

\begin{corollary}
  \label{cor:application-statistical}
Consider any \FrameworkShort class $(\cM, \ActSpace, \cO)$ and \mlike loss $\divSymbol$. \cref{alg:etd-off}, with 
exploration parameter $\gamma >0$ and 
$\AlgRed$ chosen to be \cref{alg:weighted-maj-vote}, ensures that
\begin{align*}
\En[\RegDM] \leq  \bigoh(\log{}T)\cdot{}\max\crl[\bigg]{\compgen(\cM)\cdot{}T,\;
\gamma\cdot{}(\offgan+1)\log|\cM|},
\end{align*}
for any offline estimation oracle $\Orc$ with parameter $\offgan$.
\end{corollary}
This result shows that information-theoretically, it is possible to
achieve low regret in the DMSO framework with offline oracles, though
the result is not computationally efficient.

As an example, in the case of square loss regression (\cref{sec:examples}) which is used for contextual bandits, an offline guarantee of $\offgan = \bigoh(\log |\cM|)$ is achievable. Meanwhile, it is known that $\compgen[\Dsqshort](\cM) \lesssim |\cA|/\gamma$ \citep{foster2021statistical}. Thus \cref{cor:application-statistical} achieves a bound of $\bigoht(\sqrt{|\cA|T}\cdot \log |\cM|)$ with an appropriate choice of $\gamma$.  The best know regret guarantee for contextual bandit is $\bigoht(\sqrt{|\cA|T\log|\cM|})$ \citep{simchi2020bypassing,xu2020upper}. The bound from \cref{cor:application-statistical} matches the state-of-the-art result up to a factor $\bigoht(\sqrt{\log |\cM|})$. How to remove this suboptimality is an interesting direction for future work.

Naturally, the other reductions for oracle-efficient online estimation developed in
this paper can be combined with \cref{alg:etd-off} as well. In
particular, by combining with \cref{thm:reduction-to-ODE} we derive
the following corollary for squared Hellinger distance.

\newtheorem*{cor:hellinger_reduction_informal}{Corollary (Informal)}
  \begin{cor:hellinger_reduction_informal}
    Whenever online conditional density estimation can be performed
    efficiently with access to the full history, and whenever the
    minimax problem in \cref{def:dec} can be solved efficiently, there
    exists a computationally efficient and offline oracle-efficient algorithm with near-optimal regret in the DMSO framework.
  \end{cor:hellinger_reduction_informal}

\subsection{Bypassing Impossibility of Memoryless Algorithms via Coverability}
\label{subsec:positive-results-w-struct}
\newcommand{\PbarM}[1][M]{\wb{P}^{\sss{#1}}}

Recall that our results in \cref{sec:memoryless} show that in general,
it is impossible to obtain low online estimation error through
memoryless oracle-efficient algorithms. In this section, we revisit
memoryless algorithms for the \emph{Markov decision processes} a
particular type of class $(\cM,\Pi,\cO)$ (or equivalently,
$(\CovarSpace,\ObsSpace, \ValSpace,\Kernel,\FunSpace)$). We prove that
for any class of Markov decision processes for which a structural
parameter called \emph{coverability} \citep{xie2022role} is small, any
offline estimator can be directly converted into an online estimator.

  \paragraph{Markov decision processes} We consider classes
  $(\cM,\Pi,\cO)$ that correspond to an episodic finite-horizon reinforcement
  learning setting, following \citet{foster2021statistical}. With $H\in \bbN$ denoting the horizon, each \modlName $\modl\in\ModelSpace$ specifies a non-stationary Markov decision process as a tuple
  $M=\crl*{\crl{\cS_h}_{h=1}^{H}, \cA, \crl{\Pm_h}_{h=1}^{H}, \crl{\Rm_h}_{h=1}^{H},
    d_1}$, where $\cS_h$ is the state space for layer $h$, $\cA$ is the action space,
  $\Pm_h:\cS_h\times\cA\to\Delta(\cS_{h+1})$ is the probability transition
  kernel for layer $h$, $\Rm_h:\cS_h\times\cA\to\Delta(\bbR)$ is
  the reward distribution for layer $h$, and $d_1\in\Delta(\cS_1)$ is the initial
  state distribution. We allow
 the reward distribution and transition
  kernel to vary across models in $\ModelSpace$ and assume that the initial
  state distribution is fixed. \looseness=-1

  We set $\Pi\subset\PiRNS$, which denotes the set of all randomized,
  non-stationary policies $\act=(\act_1,\ldots,\act_H)\in\PiRNS$,
  where $\act_h:\cS_h\to\Delta(\cA)$. For a fixed MDP
  $M\in\ModelSpace$ and $\pi\in\Pi$, the observation $o\sim{}M(\pi)$
  is a trajectory $(s_1,a_1,r_1),\ldots,(s_H,a_H,r_H)$ that is
  generated through the following process, beginning from
  $s_1\sim{}d_1$. For $h=1,\ldots,H$:
  \begin{itemize}
  \item $a_h\sim\act_h(s_h)$.
  \item $r_h\sim\Rm_h(s_h,a_h)$ and $s_{h+1}\sim{}P_h\sups{M}(\cdot\mid{}s_h,a_h)$.
  \end{itemize}
  So the obseravtion space $\cO = \cS_1\times \cA \times \bbR \times \dots \times \cS_H\times \cA\times \bbR$.
  For notational convenience, we take $s_{H+1}$ to be a deterministic terminal
state. We use $\bbP^{\sss{M},\act}$ and $
\En^{\sss{M},\act}[\cdot]$ to denote the probability law and
expectation over trajectories induced by $M(\decn)$. In addition, we
define $\PbarM_h(\cdot\mid{}s_h,a_h)$ as the conditional
distribution on $s_{h+1},r_h$ given $s_h,a_h$ under $M$ for $h\in [H]$. 

The guarantees we provide apply to any loss that has a particular
\emph{layer-wise} structure tailored to reinforcement learning.
\begin{definition}[Layer-wise loss]
  \label{def:layer-wise-loss}
For any sequence of \jqedit{losses $\set{\divhSymbol}_{h\in [H]}$ bounded by $[0,1]$,
where $\divhSymbol:\Delta(\cS_h\times \bbR)\times \Delta(\cS_h\times \bbR) \to [0,1]$ for all $h\in [H]$}, we define the layer-wise loss $\divRLSymbol$ on $\Delta(\cO)$ as \footnote{ \jqedit{The ordering of $M$ and $M'$ on the right-hand side of the definition is due to the following technical reason: \cref{thm:coverability-mdp} only works when the expectation on the right-hand side to be taken with respect to $\Mstar(\pi)$, which shows up as the second argument in the offline oracle guarantee (\cref{eq:offline_constraint_simulation}).}}
  \begin{align*}
    \divRL{M(\pi)}{M'(\pi)} = \sum\limits_{h=1}^{H} \En^{\sss{M',\pi}} \brk*{ \divh{ \PbarM[M'](\cdot\mid{}s_h,a_h)}{ \PbarM(\cdot\mid{}s_h,a_h)}},
  \end{align*}
  for any pair of MDPs $M,M'\in\cM$ and policy
  $\pi\in\Pi$.\footnote{For the results in this section, it will be
    useful to work with asymmetric losses, and in this case we use the
    notation $\divSymbol(\cdot\dmid\cdot)$ instead of $\div{\cdot}{\cdot}$.}
\end{definition}

Examples of the layer-wise loss are \jqedit{scaled reverse KL-divergence (which is
bounded by $[0,1]$ whenever the density ratios under consideration are upper and lower
bounded with an appropriate scaling)} \citep{foster2021statistical} and the squared
Bellman error \citep{foster2023model}. Another useful example is the
the sum of layer-wise squared Hellinger distances given by
\begin{align}
  \label{eq:hellinger_layer}
    \divRLHels{M(\pi)}{M'(\pi)} = \sum\limits_{h=1}^{H} \En^{\sss{M',\pi}}  \brk*{ \Dhels{ \PbarM(\cdot\mid{}s_h,a_h)}{ \PbarM[M'](\cdot\mid{}s_h,a_h)}}.
  \end{align}
This loss coincides with the global squared Hellinger distance
$\Dhels{M(\pi)}{M'(\pi)}$ up to an $\bigoh(H)$ factor.

\paragraph{Coverability}
We provide memoryless oracle-efficient algorithms for online estimation for any
layer-wise loss $\divRLSymbol$ when the underlying MDP $\Mstar$
has bounded \emph{coverability} \citep{xie2022role}.
\begin{definition}[Coverability]
\label{def:coverability}
For an MDP $\Mstar$ and a policy $\act$, we define $d_h^{\act}(s,a) \equiv \En^{\sMstar,\act}\brk*{ \mathbbm{1}(s_h,a_h = s,a) } $.
The coverability coefficient $C_\cov$ for a policy class $\ActSpace$ for the MDP $\Mstar$ is given by
\begin{align*}
C_\cov(\Mstar) \ldef \inf_{\nu_1,\dots,\nu_H\in \Delta(\cS\times \cA)} \sup_{\act\in \ActSpace, h\in [H]} \nrm*{\frac{d_h^\act}{\nu_h}}_\infty.
\end{align*}
\end{definition}

It is immediate to see that $\Ccov\leq\abs{\Pi}$, but in general it
can be much smaller.
Examples of MDP classes with low coverability include Block MDPs, Low-Rank MDPs, and exogenous block MDPs \citep{xie2022role,amortila2024harnessing}.

\paragraph{Offline-to-online conversion under coverability}

Our main result shows that under coverability, the outputs of any
offline estimation oracle $\Orc$ satisfy an online estimation
guarantee \emph{as-is}.

\begin{restatable}[Offline-to-online conversion under coverability]{theorem}{coverabilityupper}
  \label{thm:coverability-mdp}
  For any layer-wise loss $\divRLSymbol$ and MDP class $(\cM,\Pi,\cO)$ and $\Mstar\in\cM$, the sequence
  of estimators $(\Mhat\ind{1},\dots,\Mhat\ind{T})$ produced by any
  offline estimation oracle $\Orc$ for $\divRLSymbol$ with parameter $\offgan$ satisfy
  \begin{align*}
    \sum\limits_{t=1}^{T}  \divRL{\Mhat\ind{t}(\decn\ind{t})}{\Mstar(\decn\ind{t})} \leq \bigoh\prn*{\sqrt{HC_\cov(\Mstar) \offgan T \log T} + HC_\cov(\Mstar)}.
  \end{align*}
\end{restatable}
This result is based on a variant of the proof technique in Theorem 1 of
\citet{xie2022role}. An important application of the result, which can be applied in
tandem with the guarantees in \citet{foster2021statistical}, concerns
squared Hellinger distance.

\begin{restatable}{corollary}{coverabilitymdp}
  \label{cor:coverability-mdp-version}
  For any MDP class $(\cM,\Pi,\cO)$ and $\Mstar\in\cM$, the sequence
  of estimators $(\Mhat\ind{1},\dots,\Mhat\ind{T})$ produced by any
  offline estimation oracle $\Orc$ for squared Hellinger distance
  $\Dhelshort^2$ with parameter $\offgan$ satisfy
\begin{align*}
  \EstOnHels(T) = \sum\limits_{t=1}^{T}\Dhels{\Mhat\ind{t}(\decn\ind{t})}{\Mstar(\decn\ind{t})}\leq \bigoh\prn*{H\sqrt{C_\cov(\Mstar) \offgan T \log T  }+H^2C_\cov(\Mstar)}
\end{align*}
\end{restatable}
This result follows by using that the layer-wise squared Hellinger distance in
\cref{eq:hellinger_layer} is equivalent to $\Dhels{M(\pi)}{M'(\pi)}$
up to $\bigoh(H)$ factors.

\paragraph{Application to interactive decision making}
We apply \cref{thm:coverability-mdp} to decision making via \cref{alg:etd-off}.
\begin{corollary}
  \label{cor:application-memless}
  Consider any layer-wise loss $\divSymbol = \divRLSymbol$ and MDP
  class $(\cM,\Pi,\cO)$, and let
  $\Ccov\ldef{}\sup_{M\in\cM}\Ccov(M)$. \cref{alg:etd-off} with 
  exploration parameter $\gamma >0$ and 
  $\AlgRed$ chosen to be the identity map ensures that
\begin{align*}
  \En[\RegDM] \leq  \bigoh(\log{}T)\cdot{}\max\crl[\bigg]{\sup_{\mu\in\Delta(\cM)}\compgen(\cM,\mu)\cdot{}T,\;
  \gamma\cdot{}\prn*{\sqrt{HC_\cov \offgan T \log T} + HC_\cov}},
\end{align*}
for any offline estimation oracle $\Orc$ with parameter $\offgan$.
\end{corollary}

\paragraph{Contextual bandits and optimality of offline-to-online conversion}

Another implication for \pref{thm:coverability-mdp} concerns the
special case of contextual bandits (that is, MDPs with horizon
one). For the contextual bandit setting we abbreviate $\cS=\cS_1$, and refer to $d_1\in\Delta(\cS)$ as the \emph{context distribution}. We define $\gm(s,a) = \En_{r\sim R_1^M(s,a)}\brk*{r}$ as the expected reward function under a model $M$, and following \citet{foster2020beyond}, use the squared error between mean reward functions as our divergence:
\begin{align}
\label{eq:cb_div}
\divCB{M(\pi)}{M'(\pi)} \ldef \En_{s\sim d_1, a\sim \pi(s)}\brk*{ \Dsq{\gm(s,a)}{g^{\sss{M'}}(s,a)}}.
\end{align}
For this setting, the coverability coefficient $C_\cov$ is
always bounded by the number of actions $|\cA|$, which leads to the following corollary.

\begin{corollary}
  \label{cor:coverability-contextual-version}
    For any contextual bandit class $(\cM,\Pi,\cO)$ and $\Mstar\in\cM$, the sequence
  of estimators $(\Mhat\ind{1},\dots,\Mhat\ind{T})$ produced by any
  offline estimation oracle $\Orc$ for $\divCBSymbol$ with parameter $\offgan$ satisfy
\begin{align*}
  \EstOnD(T) = \sum\limits_{t=1}^{T}
  \divCB{\Mhat\ind{t}(\decn\ind{t})}{\Mstar(\decn\ind{t})} \leq
  \bigoh\prn[\big]{\sqrt{|\cA|T \offgan\log T} + |\cA|}, 
\end{align*}
\end{corollary}
Recall that \citet{foster2020beyond} show that any algorithm for
online estimation with the divergence in \cref{eq:cb_div} with $\EstOnD(T)\leq\ongan$
can be lifted to a contextual bandit algorithm with regret
$\bigoh\prn[\big]{\sqrt{\abs{\cA}T\cdot\ongan}}$ via the \emph{inverse
  gap weighting strategy}, even if contexts are chosen adversarially. Subsequent work of
\citet{simchi2020bypassing} shows that for stochastic contexts, the
inverse gap weighting strategy also yields regret
$\bigoh\prn[\big]{\sqrt{\abs{\cA}T\cdot\offgan}}$ given access to an
offline oracle with parameter
$\offgan$. On the other hand, combining
\cref{cor:coverability-contextual-version} with the guarantee from
\citet{foster2020beyond} gives regret
$\bigoh\prn[\big]{\abs{\cA}^{1/4}T^{3/4}\offgan^{1/4}}$. This does not
recover the result from \citet{simchi2020bypassing}, but nonetheless gives an alternative proof that
sublinear offline estimation error suffices for sublinear regret.

The guarantee in \cref{thm:coverability-mdp} leads to a degradation in
rate from $\offgan$ to $\sqrt{T\offgan}$ (suppressing
problem-dependent parameters). Our next result shows that this is tight
in general.
\begin{restatable}[Tightness of offline-to-online conversion]{proposition}{coverabilitylower}
  \label{thm:coverability-lower-bound}
  For any integer $T\geq 1$ \dfedit{and $\offgan>0$}, there exists a contextual bandit class
  $(\ModelSpace,\Pi=\cA^{\cS},\cO)$
  with $|\cA|= 2$, 
  a distribution $d_1\in \Delta(\cS)$, a sequence 
  $(\pi\ind{1},\dots,\pi\ind{T})$ and an offline oracle $\Orc$ for $\divCBSymbol$
  with parameter $\offgan$ such that the oracle's outputs
  $(\Mhat\ind{1},\ldots,\Mhat\ind{T})$ satisfy
  \begin{align*}
    \EstOnD(T) = \sum\limits_{t=1}^{T} \divCB{\Mhat\ind{t}(\decn\ind{t})}{\Mstar(\decn\ind{t})} \geq \bigom\prn*{\sqrt{T\offgan}}.
  \end{align*}
\end{restatable}

}

\section{Discussion}
\label{sec:discussion}

Our work introduces the Oracle-Efficient Online Estimation protocol as an information-theoretic framework to study the relative power of online and offline estimators and gives a nearly complete characterization of the statistical and computational complexity of learning in this framework. In what follows, we discuss broader implications for our information-abstraction of oracle-efficiency.

\arxiv{\subsection{Oracle-Efficient Learning as a Framework for Algorithm Design and Analysis}}

\colt{\paragraph{Oracle-efficient learning as a general framework for analysis of algorithms}}

One of the most important contributions of this work is to formalize oracle-efficient algorithms as mappings that act upon a sequence of estimators but do not directly act on historical outcomes. While the computational lower bounds we provide for oracle-efficient learning are somewhat disappointing, we are optimistic that---similar to statistical query complexity in TCS and information-based complexity in optimization---our abstraction can (i) aid in designing computationally efficient algorithms for learning settings beyond online estimation, and (ii) continue to serve as a tool to formalize lower bounds against natural classes of algorithms, for estimation and beyond. That is, we envision \emph{oracle-efficient learning} as a more general framework to study oracle-based algorithms in any type of interactive learning problem. \dfedit{We remark that one need not restrict to offline oracles; it is natural to study oracle-efficient algorithms based on online estimation oracles or other types of oracles through our information-theoretic abstraction as well.} For concreteness, let us mention a couple of natural settings where our information-theoretic abstraction can be applied. %

\paragraphi{Oracle-efficient interactive decision making} For interactive decision making problems like bandits and reinforcement learning (more broadly, the DMSO framework described in \cref{subsec:application-to-decision-making}), it is natural to formalize \emph{oracle-effiicent} algorithms as algorithms that do not directly observe rewards (bandits) or trajectories (reinforcement learning), and instead must select their decision based on an (online or offline) estimator (e.g., regression for bandits or conditional density estimation for RL). \citet{foster2020beyond} and \citet{foster2021statistical} \arxiv{and follow-up work}\colt{et seq.} provide algorithms with this property for contextual bandits and RL, respectively, but the power of offline oracles in this context is not well understood.\loose

\paragraphi{Oracle-efficient active learning} For active learning, it is natural to consider algorithms that decide whether to query the label for a point in an oracle-efficient fashion (e.g., \citet{krishnamurthy2017active}). For concreteness, consider pool-based active learning \citep{hanneke2014theory}. Suppose the learner is given a pool $\cP = \set{x_1,...,x_n}$ of covariates and a \FunSpaceName $\cF$. The learner can repeatedly choose $x\ind{t}\in\cP$ and call the offline oracle to obtain an estimator $\fhat\ind{t}$ such
$\EstOffD(t) \ldef \sum_{i<t} \div{\fhat\ind{t}(x\ind{i})}{\fstar(x\ind{i})} \leq \offgan$ (in contrast to an unrestricted algorithm that observes $y\ind{t}=\fstar(x\ind{t})$). The aim is to learn a hypothesis with low classification error using the smallest number of queries possible. Can we design oracle-efficient algorithms that do so with near-optimal label complexity?

    \arxiv{
  \subsection{Further Directions}

  We close with some additional directions for future research.
  
\paragraph{Refined notions of estimation oracles}
This work considers generic offline estimation algorithms that satisfy the statistical guarantee in \cref{def:offline-oracle} but can otherwise be arbitrary. Understanding the power of offline estimators that satisfy more refined (e.g., problem-dependent) guarantees is an interesting direction for future research.

\paragraph{Open questions for proper versus improper learning} Our results leave some interesting gaps in the power of proper versus improper oracles. First, the computational lower bounds
in \cref{subsec:comp-hardness}, leave open the possibility of attaining $\poly (\log |\cM_n|)\cdot o(T)$ online estimation error with runtime $\poly (\log |\cM_n|)$ given access to a \emph{proper} offline estimation oracle with parameter $\offgan=0$. Second, our results in \cref{subsec:impossibility-results} leave open the possibility of
bypassing the $\Omega(|\CovarSpace|(\offgan+1))$ lower bound for memoryless algorithms under the assumption that the offline oracle is proper.
}

\bibliography{refs}

\clearpage

\renewcommand{\contentsname}{Contents of Appendix}
\addtocontents{toc}{\protect\setcounter{tocdepth}{2}}
{\hypersetup{hidelinks}
\tableofcontents
}

\clearpage

\appendix  

\arxiv{
\part{Additional Discussion and Examples}

\section{Additional Related Work}
\label{sec:related}

In this section we discuss related work not already covered in detail.

\paragraph{Computational lower bounds for online learning}

Beyond \citet{blum1994separating}, another work that considers computational lower bounds for online learning is \citet{hazan2016computational}. This work proves lower bounds for online learning in a model where the learner has access to an ERM oracle that can minimize the training loss for an arbitrary dataset $(x\ind{1},y\ind{1}),\ldots,(x\ind{T},y\ind{T})$. Their lower bound does not fit in our computational model due to details around the way description length is formalized. In particular, the main focus of \cite{hazan2016computational} is to obtain a lower bound on the \emph{number of oracle calls} any online learning algorithm must make to an ERM oracle. 

Similar to the setup for \cref{thm:computation-lower-bound}, \citet{hazan2016computational} consider a sequence of classification instances with $\cX_n=\crl*{0,1}^{n}$ and classes $\cF_n$ of the size of $\Omega(2^{\sqrt{|\cX_n|}})$, and show that any online learning algorithm requires $\Omega(\sqrt{\abs{\cX_n}})$ oracle calls to achieve low regret for this class. However, the estimators $f\in\cF_n$ returned by the oracle in their construction have $\Omega(\sqrt{\abs{\cX_n}})$ description length themselves, meaning that they do not satisfy the $\poly(n)$-description length required by the model described in \cref{subsec:comp-hardness} (in other words, the result is not meaningful as a lower bound on runtime, because simply reading in the output of the offline oracle takes exponential time).
For completeness, we restate the example proposed by \citet[Theorem 22]{hazan2016computational} in our framework below.

\paragraphi{Hard case from \cite{hazan2016computational}}
For any integer $n\geq 1$, consider a binary classification problem with
  $\CovarSpace_{2n} = \set{0,1}^{2n}$, $\ValSpace =\ObsSpace =\set{0,1}$,
  $\divSymbol = \Dbinshort$ and $\Kernel(\val)=\indic_\val$.
  Let $N \ldef 2^{2n}$, and \dfedit{let $\cS$ be the collection of all sets} $S = \set{s_1,\dots,s_{2^n}}\subset \set{0,1}^{2n}$ where \jqedit{$\set{0,1}^{2n}$ is also treated as the integer set of $\set{0,\dots,2^{2n}-1}$ in left-to-right order and} $s_i\in \set{2^n(i-1),\dots, 2^ni-1}$ for each $i=1,\dots,2^n$. We define a class $\cF_{2n} = \set{f_{S,\tau} :S\in\cS,\tau\leq{}2^n}$, where 
  \begin{align*}
  f_{S,\tau}(x) = 
  \begin{cases}
  0 \quad \text{if }x \in S \text{ and } x \geq s_\tau,\\
  1 \quad \text{otherwise.}
  \end{cases}
  \end{align*}
  For this class, 
  we reduce back to the Theorem 22 of \citet{hazan2016computational} which states that any algorithm with runtime $o(\sqrt{N})$ has to suffer online estimation error at least $t/2$ for all $1\leq t\leq 2^n$.
  The issue with this example for our computation model is that $|\cF_{2n}| = \Omega(2^{\sqrt{|\cX_{2n}|}})$. Any sufficient description for this \FunSpaceName in bit strings (or, e.g., boolean circuits) will scale with $\Omega(\sqrt{|\cX_{2n}|})$.  
  Thus, the description length required to return $\fhat\ind{t}$ is too large (already larger than the lower bound obtained).

  \paragraph{Online learning with memory constraints}  A number of recent works focus on memory-regret tradeoffs in online learning \citep{srinivas2022memory,peng2023online,peng2023near,aamand2023improved,woodruff2023streaming}. Here, the learner can observe the full data stream $(x\ind{1},y\ind{1}),\ldots,(x\ind{T},y\ind{T})$, but is constrained to $B$ bits of memory. This framework is incomparable to the \framework framework, but it would be interesting to explore whether there are deeper connections (e.g., any memoryless \framework algorithm inherently has memory no larger than that of the offline oracle).\loose

\paragraph{Gaps between offline and online}
A long line of work aims to characterize the optimal regret for online learning, developing complexity measures (Littlestone dimension, sequential Rademacher complexity) that parallel classical complexity measures like VC dimension and Rademacher complexity for offline learning and estimation \citep{ben2009agnostic,rakhlin2010online,rakhlin2014online,attias2023optimal}. It is well known that in general, the optimal rates for online learning can be significantly worse than those for offline learning. Our work primarily focuses on finite classes $\cF$, where there is no gap, but for infinite classes, any conversion from offline to online estimation will inevitably lead to a loss in the estimation error rate that scales with appropriate complexity measures for online learning (cf. \cref{sec:infinite}).

\paragraph{Other restricted computational models}
Our information-theoretic formulation of oracle-efficiency is inspired by statistical query complexity in theoretical computer science and information complexity in optimization, both of which can be viewed as restricted computational models with an information-theoretic flavor. The statistical query model is a framework in which the learner can only access the environment through an oracle that outputs noise estimates (``statistical queries'') for a \truefunName of interest \citep{blum1994weakly,kearns1998efficient,feldman2012complete,feldman2017general,feldman2017statistical}. Information complexity in optimization is a model in which algorithms can only access the \funName of interest through (potentially noisy) local queries to gradients or other information \citep{nemirovski1983problem,traub1988information,raginsky2011information,agarwal2012information,arjevani2023lower}.\loose

\section{Additional Results}
\label{app:additional}
\subsection{The Role of Randomization}
\label{app:randomization}

A natural question is to what extent our results change if the offline oracle in the \framework framework is allowed to be \emph{randomized}, in the sense that the estimator $\Mhat\ind{t}$ is drawn from a distribution $\nu\ind{t}\in\Delta(\cM)$ that is chosen based on data. Concretely, we assume that for all $t\in [T]$, there is a randomized oracle output $\nu\ind{t}\in\Delta(\cM)$ such that 
\begin{align}
  \label{eq:offline_randomized}
    \sum\limits_{s=1}^{t-1}\En_{\Mhat\sim \nu\ind{t}}\brk*{ \div{\Mhat(\covar\ind{s})}{ \Mstar(\covar\ind{s}) } } \leq \offgan,
    \end{align}

    We observe that if the loss is convex, then it suffices to restrict our attention \emph{improper, non-randomized} offline oracles algorithms; this can be seen by noting that the mixture $\En_{\Mhat\sim \nu\ind{t}}\Mhat$ also satisfies the offline guarantee in \cref{eq:offline_randomized}.

    Next, we observe that when the loss $\divSymbol$ is convex, it is possible to replace any improper estimator by a proper one via projection. 
    \begin{lemma}
        \label{lem:randomized-proper}
        For any model class $(\cM, \CovarSpace,\cO)$, any convex \mlike loss $\divSymbol$, any distribution $p$ on the covariate space and any \textbf{possibly improper} $\Mhat$, let 
        \[\Mhat' \ldef \argmin_{M\in \cM} \En_{\covar\sim p}\brk*{\div{\Mhat(\covar)}{M(\covar)}},\] 
        where $\argmin_{M\in\cM}$ breaks ties with a predefined ordering, we have
    \begin{align*}
        \En_{x\sim p}\brk*{ \div{\Mhat'(\covar)}{\Mstar(\covar)} } \leq  2C_\divSymbol\cdot \En_{x\sim p}\brk*{ \div{\Mhat(\covar)}{\Mstar(\covar)} } .
    \end{align*}
    \end{lemma}
    \begin{proof}[\pfref{lem:randomized-proper}]
    \begin{align*}
        \En_{x\sim p}\brk*{ \div{\Mhat'(\covar)}{\Mstar(\covar)} } &\leq  C_\divSymbol \En_{x\sim p}\brk*{ \div{\Mhat'(\covar)}{\Mhat(\covar)} } + C_\divSymbol\En_{x\sim p}\brk*{ \div{\Mhat(\covar)}{\Mstar(\covar)} }\\
        &\leq 2C_\divSymbol\cdot \En_{x\sim p}\brk*{ \div{\Mhat(\covar)}{\Mstar(\covar)} },
    \end{align*}
    where the second inequality is by the definition of $\Mhat'$.
  \end{proof}
  Hence, given a potentially improper estimator $\Mhat$ with $\sum_{s=1}^{t-1}\div{\Mhat(\covar\ind{s})}{ \Mstar(\covar\ind{s}) }  \leq \offgan$, the estimator $\Mhat'=\argmin_{M\in\cM}\sum_{s=1}^{t-1}\div{\Mhat(\covar\ind{s})}{M(\covar\ind{s}) }$ is proper and satisfies
  \[
\sum_{s=1}^{t-1}\div{\Mhat'(\covar\ind{s})}{ \Mstar(\covar\ind{s}) }  \leq 2\Closs\offgan.
\]
  Note however that this reduction can only be applied in the general \framework framework, and is not implementable for memoryless algorithms as-is.

\section{Examples of Offline Oracles}
\label{app:background}
\newcommand{\MLE}{$\mathrm{MLE}$\xspace}
\newcommand{\ERM}{$\mathrm{ERM}$\xspace}

For completeness, below we prove the offline estimation (fixed design) guarantees for
square loss empirical risk minimization and maximum likelihood
estimation mentioned \colt{in the prequel}\arxiv{in \cref{sec:examples}}.

\paragraph{Empirical risk minimization for square loss regression}

For any square loss regression instance $(\CovarSpace,\ObsSpace, \ValSpace,\Kernel,\FunSpace)$ defined as in \cref{sec:examples}.
The Empirical Risk Minimizer (\ERM) estimator $\fhat$ is defined as
\begin{align*}
    \fhat = \argmin_{f\in \FunSpace}\sum\limits_{t=1}^{T} (f(x\ind{t}) - \obs\ind{t})^2.
\end{align*}
We have the following bounds on the offline estimation error for the squared loss.
\begin{lemma}
\label{lem:erm-guarantee}
For any \truefunName $\fstar\in \FunSpace$, with probability at least $1-\delta$, we have,
\begin{align*}
    \sum\limits_{t=1}^{T} \prn*{\fhat(\covar\ind{t}) - \fstar(\covar\ind{t}) }^2\leq 8\log (\abs*{\FunSpace}/\delta). 
\end{align*} 
\end{lemma}
\begin{proof}[\pfref{lem:erm-guarantee}]
  Observe that the \ERM estimator satisfies
\begin{align}
    \sum_{t=1}^{T} (\fstar(\covar\ind{t}) - \fhat(\covar\ind{t}))^2 &= \sum_{t=1}^{T} \Big((\fhat(\covar\ind{t}) - \obs\ind{t})^2 - (f^\star(\covar\ind{t}) - \obs\ind{t})^2 + 2(\fhat(\covar\ind{t}) - f^\star(\covar\ind{t}))(\obs\ind{t}-f^\star(\covar\ind{t}))\Big) \notag\\
    &\leq 2\sum_{t=1}^{T} (\fhat(\covar\ind{t}) - f^\star(\covar\ind{t}))\veps\ind{t}, \label{ineq:oracle-ineq}
\end{align}
where the last inequality is by the definition of $\fhat$, and $\veps\ind{t} \ldef \obs\ind{t}-f^\star(\covar\ind{t})$ is 
a standard normal distribution
for all $t\in [T]$.
Then by the Gaussian tail bound, we have with probability at least $1-\delta$,
\begin{align*}
    \sum_{t=1}^{T} (\fhat(\covar\ind{t}) - f^\star(\covar\ind{t}))\veps\ind{t} &\leq \sqrt{2\sum_{t=1}^{T}   (\fhat(\covar\ind{t}) - f^\star(\covar\ind{t}))^2 \log(|\FunSpace|/\delta)}.
\end{align*}
Plug the above inequality back into \cref{ineq:oracle-ineq} and reorganize, we obtain the desired bound of 
\begin{align*}
    \sum_{t=1}^{T} (\fhat(\covar\ind{t}) - \fstar(\covar\ind{t}))^2 \leq 8 \log(|\FunSpace|/\delta). 
\end{align*}

\end{proof}

\paragraph{Maximum likelihood estimation for conditional density estimation}
For any conditional density estimation instance $(\CovarSpace,\ObsSpace, \ValSpace,\Kernel,\FunSpace)$ defined as in \cref{sec:examples}.
The Maximum Likelihood Estimator (\MLE) $\fhat$  
is defined as
\begin{align*}
\fhat = \argmax_{\fun\in \FunSpace}\sum\limits_{t=1}^{T}\log \fun ( \obs\ind{t} \mid{}\covar\ind{t} ). 
\end{align*}
We have the following bounds on the offline estimation error for squared Hellinger distance.
\begin{lemma}
\label{lem:mle-guarantee}
For any \truefunName $\fstar\in \FunSpace$, with probability at least $1-\delta$, we have,
\begin{align*}
    \sum\limits_{t=1}^{T} \Dhels{\fhat(\covar\ind{t})}{\fstar(\covar\ind{t})}\leq \log (\abs*{\FunSpace}/\delta). 
\end{align*} 
\end{lemma}

\begin{proof}[\pfref{lem:mle-guarantee}]
For any \funName $f\in \FunSpace$, define 
\begin{align*}
Z_t = - \frac{1}{2} (\log \fstar ( \obs\ind{t} \mid{}\covar\ind{t} ) - f (\obs\ind{t} \mid{} \covar\ind{t} )).
\end{align*}
Then for any $f\in \cF$, by Lemma A.4 of \citet{foster2021statistical}, with probability at least $1-\delta/|\FunSpace|$, we have,
\begin{align*}
\sum\limits_{t=1}^{T} - \frac{1}{2} (\log \fstar ( \obs\ind{t} \mid{}\covar\ind{t} ) - \log \fun (\obs\ind{t} \mid{} \covar\ind{t} ) ) \leq \sum\limits_{t=1}^{T} \log \prn*{\En\brk*{  \sqrt{\frac{\fun ( \obs\ind{t} \mid{}\covar\ind{t} )}{\fstar (\obs\ind{t} \mid{} \covar\ind{t} )}  } }}   + \log (|\FunSpace|/\delta).
\end{align*}
We further have by the inequality of $\log(1+x)\leq x$ that
\begin{align*}
    \log \prn*{\En\brk*{  \sqrt{\frac{\fun (\obs\ind{t} \mid{} \covar\ind{t} )}{\fstar ( \obs\ind{t} \mid{}\covar\ind{t} )}  } } }\leq  \En \brk*{  \sqrt{\frac{\fun ( \obs\ind{t} \mid{}\covar\ind{t} )}{\fstar ( \obs\ind{t} \mid{}\covar\ind{t} )}  } -1}. 
\end{align*}
Since $\obs\ind{t}\sim \fstar ( \covar\ind{t} )$, we have by standard calculus and the definition of the squared Hellinger distance that 
\begin{align*}
    \En \brk*{  \sqrt{\frac{\fun ( \obs\ind{t} \mid{}\covar\ind{t} )}{\fstar ( \obs\ind{t} \mid{}\covar\ind{t} )}  } -1} = -\Dhels{\fstar(\covar\ind{t})}{\fun(\covar\ind{t})}. 
\end{align*}
Altogether, we have obtained for any $f\in \cF$, with probability at least $1-\delta/|\FunSpace|$
\begin{align*}
    \sum\limits_{t=1}^{T} - \frac{1}{2} (\log \fstar ( \obs\ind{t} \mid{}\covar\ind{t} ) - \log \fun (\obs\ind{t} \mid{} \covar\ind{t} )) \leq - \sum\limits_{t=1}^{T} \Dhels{\fstar(\covar\ind{t})}{\fun(\covar\ind{t})}     +  \log (|\FunSpace|/\delta). 
\end{align*}
Thus, by union bound, the above inequality holds for all $\fun\in \FunSpace$ with probability at least $1-\delta$.
Thus the MLE $\fhat$ satisfies
\begin{align*}
    \sum\limits_{t=1}^{T} \Dhels{\fstar(\covar\ind{t})}{\fhat(\covar\ind{t})}    &\leq \sum\limits_{t=1}^{T}  \frac{1}{2} (\log\fstar (  \obs\ind{t} \mid{}\covar\ind{t} ) - \log \fhat (  \obs\ind{t} \mid{}\covar\ind{t} ) )   + \log (|\FunSpace|/\delta) \\
    &\leq \log (|\FunSpace|/\delta),
\end{align*}
where the second inequality is by the defintion of MLE.
\end{proof}

}

\colt{
\part{Additional Discussion and Examples}

\section{Additional Related Work}
\label{sec:related}

\section{Examples of Estimation Problems and Loss Functions}
\label{sec:examples}

\subsection{Examples of Offline Oracles}
\label{app:background}

}

\colt{
  \clearpage
  \part{Omitted Results}

\section{General Reductions for Oracle-Efficient Online Estimation}
\label{sec:delayed}

    \section{Application to Interactive Decision Making}
\label{sec:application}

  }

\clearpage

\part{Proofs}

\section{Technical Tools}
\label{app:technical}

\begin{lemma}
  \label{lem:potential}
For any non-increasing sequence
  $x_1\geq{}x_2\geq{}\cdots{}\geq{}x_{T+1}\geq{}1$,
  \begin{align*}
    \sum_{t=1}^{T}\frac{x_t- x_{t+1}}{x_{t}} \leq \log(x_1).
  \end{align*}
\end{lemma}

\begin{proof}[\pfref{lem:potential}]
  Since $\log(1+a)\leq a$ for all $a>-1$, for any $t\in [T]$, we have
  \begin{align*}
  - \log(x_{t}/x_{t+1}) = \log \prn*{ 1 + \prn*{\frac{x_{t+1}}{x_{t}} -1}   } \leq \prn*{\frac{x_{t+1}}{x_{t}} -1}.
  \end{align*}
  Summing up over $t\in[T]$, we obtain
  \begin{align*}
    \sum_{t=1}^{T}\frac{x_t- x_{t+1}}{x_{t}}  = \sum_{t=1}^{T}1-\frac{x_{t+1}}{x_{t}}
    \leq{}\sum_{t=1}^{T}\log(x_{t}/x_{t+1})=\log(x_1/x_{T+1})\leq\log(x_1).
  \end{align*}
\end{proof}

The following lemma gives an improvement to
  Lemma A.13 of \citet{foster2021statistical} that removes a
  logarithmic factor. This shows that up to an absolute constant,
  squared Hellinger distance obeys a one-sided version of the chain rule for
  KL divergence.
  \begin{lemma}[Subadditivity for squared Hellinger distance]
  \label{lem:hellinger_chain_rule}
  Let $(\cX_1,\filt_1),\ldots,(\cX_n,\filt_n)$ be a sequence of
  measurable spaces, and let $\cX\ind{i}=\prod_{i=t}^{i}\cX_t$ and
  $\filt\ind{i}=\bigotimes_{t=1}^{i}\filt_t$. For each $i$, let
  $\bbP\ind{i}(\cdot\mid{}\cdot)$ and $\bbQ\ind{i}(\cdot\mid{}\cdot)$ be probability kernels from
  $(\cX\ind{i-1},\filt\ind{i-1})$ to $(\cX_i,\filt_i)$. Let $\bbP$ and
  $\bbQ$ be
  the laws of $X_1,\ldots,X_n$ under
  $X_i\sim{}\bbP\ind{i}(\cdot\mid{}X_{1:i-1})$ and
  $X_i\sim{}\bbQ\ind{i}(\cdot\mid{}X_{1:i-1})$ respectively. Then it
  holds that
\begin{align*}
  \Dhels{\bbP}{\bbQ}
  &\leq{}
7\cdot\En_{\bbP}\brk*{\sum_{i=1}^{n}\Dhels{\bbP\ind{i}(\cdot\mid{}X_{1:i-1})}{\bbQ\ind{i}(\cdot\mid{}X_{1:i-1})}}.
\end{align*}
\end{lemma}

\begin{proof}[\pfref{lem:hellinger_chain_rule}]
We appeal to the \emph{cut-and-paste property} of
\cite{jayram2009hellinger}, defining a collection of distributions
indexed by a hypercube $\crl{0,1}^{n}$ with the property that the
vertices $(0,\ldots,0)$ and $(1,\ldots,1)$ correspond to the distribution
 $\bbP$ and $\bbQ$. 
Concretely, for any vertex $v\in \set{0,1}^n$ of the hypercube, we define a probability distribution 
\begin{align*}
\mathfrak{P}_v \ldef \prod_{i\in [n]} R_{v_i}(\cdot\mid{}X_{1:i-1}), \quad\quad \text{where~} R_{v_i}(\cdot\mid{}X_{1:i-1}) = 
\begin{cases}
\bbP\ind{i}(\cdot\mid{}X_{1:i-1})  & \text{if }v_i = 0,\\
\bbQ\ind{i}(\cdot\mid{}X_{1:i-1})  & \text{if }v_i = 1.
\end{cases}
\end{align*}
Observe that $\mathfrak{P}_{(0,...,0)} = \bbP$ and $\mathfrak{P}_{(1,...,1)} = \bbQ$. 
Now, consider any four vertices $a,b,c,d\in \set{0,1}^n$ with the
property that $\set{a_i, b_i} =
\set{c_i, d_i}$ for each $i\in [n]$ (with $\crl{\cdot}$ interpreted as
a multi-set). Then for any measure \jqedit{$\nu\ldef{}\prod_{i=1}^n
  \nu_i(\cdot|X_{1:i-1})$, where $\nu_i(\cdot|X_{1:i-1})$ is any common dominating conditional measure\footnote{For example, we can take $\nu_i(\cdot|X_{1:i}) = (\bbP\ind{i}(\cdot|X_{1:i-1})+\bbQ\ind{i}(\cdot|X_{1:i-1}))/2$. for
  $\bbP\ind{i}$ and $\bbQ\ind{i}$. The result is independent of the choice of $\nu$.}} for $\bbP\ind{i}(\cdot|X_{1:i-1})$ and $\bbQ\ind{i}(\cdot|X_{1:i-1})$, 
by the definition of squared Hellinger distance we have
\begin{align}
\Dhels{\mathfrak{P}_a }{\mathfrak{P}_b } &= 1- \int \sqrt{ \prod_{i=1}^n  \frac{dR_{a_i}(\cdot\mid{}X_{1:i-1})}{d\nu_i} \frac{d R_{b_i}(\cdot\mid{}X_{1:i-1})}{d\nu_i} } d\nu \notag\\
&= 1- \int \sqrt{ \prod_{i=1}^n  \frac{dR_{c_i}(\cdot\mid{}X_{1:i-1})}{d\nu_i} \frac{d R_{d_i}(\cdot\mid{}X_{1:i-1})}{d\nu_i} } d\nu \notag\\
&=\Dhels{\mathfrak{P}_c }{\mathfrak{P}_d }.\label{eq:cutpaste}
\end{align}

Let $k$ be the maximum integer such that $2^k\leq n$. Then since
\cref{eq:cutpaste} holds, by Theorem 7 of \citet{jayram2009hellinger}
applied with the pairwise disjoint collection $A_j = \set{ i\mid{} i\mod 2^k = j}$ for all $j \in [2^k]$, we have
\begin{align*}
  \Dhels{\bbP}{\bbQ}\cdot\prod_{i=1}^{n} (1-1/2^i)
  &\leq{} \sum_{j=1}^{2^k}  \Dhels{\bbP}{ \prod_{l\in A_j}\bbQ\ind{l}(\cdot\mid{}X_{1:l-1}) \prod_{l'\notin A_j}\bbP\ind{l'}(\cdot\mid{}X_{1:l'-1})}\\
  &\leq{} 2\En_{\bbP}\brk*{\sum_{i=1}^{n}\Dhels{\bbP\ind{i}(\cdot\mid{}X_{1:i-1})}{\bbQ\ind{i}(\cdot\mid{}X_{1:i-1})}}.
\end{align*}

To conclude, we note that $\prod_{i=1}^{n} (1-1/2^i)>2/7$.

\end{proof}

\section{Proofs from \creftitle{sec:stats-results}}
\label{app:stats-results}

\subsection{Proofs from \creftitle{subsec:opt-err-bound}}

\EstMemUpperBound*
\begin{proof}[\pfref{thm:woe-upper-bound}]
  Our main technical result is the following lemma, which is proven in
  the sequel.
\begin{lemma}
\label{lem:general-version-space}
Consider any instance $(\CovarSpace,\ObsSpace,
\ValSpace,\Kernel,\FunSpace)$ and a \mlike loss\footnote{\jqedit{For
    this lemma, $\divSymbol$ need not be a metric-like loss; it
    suffices that $\divSymbol$ is bounded
and has $\divSymbol(z,z)=0$ for all $z\in \cZ$.}} $\divSymbol$ on $\ValSpace$.
Let $\fstar\in\cF$ be the target parameter, and consider a sequence of
sets $\FunSpace = \FunSpace_1\supseteq \FunSpace_2 \supseteq \cdots
\supseteq \FunSpace_T \supseteq \set{\fstar}$ and sequence of
covariates $\covar\ind{1},\dots,\covar\ind{T} \in \CovarSpace$ with
the property that for all $t\in\brk{T}$, all $f\in\cF_t$ satisfy the
following offline estimation guarantee:
\begin{align}
  \label{ineq:local-set}
\sum\limits_{s=1}^{t-1} \div{\fun(\covar\ind{s})}{\fstar(\covar\ind{s})} \leq \offgan.
\end{align}  
Then, by defining $\mu\ind{t} = \unif(\FunSpace_t)$, we have that
\begin{align*}
  \sum\limits_{t=1}^{T} \En_{\fun\sim\mu\ind{t}}\brk*{ \div{\fun(\covar\ind{t})}{\fstar(\covar\ind{t})} } \leq \bigoh\prn*{  (\offgan+1)\cdot \min \set{ \log |\FunSpace|,  |\CovarSpace|\log T}}.
\end{align*}
\end{lemma}
To invoke \cref{lem:general-version-space}, we observe that the
version space construction in \cref{alg:weighted-maj-vote} ensures
that for all $t\in\brk{T}$, all $f\in\cF_t$ satisfy
\begin{align*}
  \sum\limits_{s=1}^{t-1}  \div{\fun(\covar\ind{s})}{\fstar(\covar\ind{s})}  \leq \sum\limits_{s=1}^{t-1} C_\divSymbol \prn*{ \div{\fun(\covar\ind{s})}{\fhat\ind{t}(\covar\ind{s})} } +  \div{\fhat\ind{t}(\covar\ind{s})}{\fstar(\covar\ind{s})} 
  \leq 2C_\divSymbol\offgan.
\end{align*}
In addition, it is immediate to see that $\FunSpace = \FunSpace_1\supseteq \FunSpace_2 \supseteq \cdots
\supseteq \FunSpace_T$. Thus, by invoking \cref{lem:general-version-space} with parameter
$\offgan'=2C_\divSymbol\offgan$, we have that
\begin{align*}
  \EstOnD(T) = \sum\limits_{t=1}^{T}\En_{\fun\sim \mu\ind{t}}\brk*{\loss{\fun(\covar\ind{t})}{\fstar(\covar\ind{t})}}
  &\leq   
  \bigoh\prn*{  (C_\divSymbol \offgan+1)\cdot \min \set{ \log |\FunSpace|,  |\CovarSpace|\log T}}.
\end{align*}
To simplify, we note that $(C_\divSymbol
\offgan+1)\leq{}C_\divSymbol(\offgan+1)$, since $C_\divSymbol \geq 1$.

\end{proof}

\begin{proof}[\pfref{lem:general-version-space}]
We begin by proving that $\EstOnD(T)\leq  
\bigoh\prn*{ C_\divSymbol \cdot{}(\offgan+1)\cdot\log
    |\FunSpace|}$. Let us adopt the convention that $\FunSpace_{T+1} = \set{\fstar}$, so
that $\FunSpace_{T+1}\subseteq \FunSpace_{T}\subseteq \dots \subseteq \FunSpace_1$. 
For any parameter $\fun\in   \FunSpace\setminus\set{\fstar}$, define
$t_{\fun} \ldef \min \set{t:\fun\notin \FunSpace_{t+1}}$. It is
immediate to see that for all $\fun\in \FunSpace\setminus\set{\fstar}$, $ 1\leq t_{\fun} \leq T$ and for all $t\in [T]$, $\abs*{\set{\fun:t_{\fun} = t}} = |\FunSpace_{t}\setminus \FunSpace_{t+1}|$.
Using that $\loss{\fstar(\covar)}{\fstar(\covar)} = 0$ for all
$\covar\in \CovarSpace$, we re-write the online estimation error as
\begin{align*}
  \sum\limits_{t=1}^{T} \En_{\fun\sim\mu\ind{t}}\brk*{
  \div{\fun(\covar\ind{t})}{\fstar(\covar\ind{t})} }  =   \sum\limits_{t=1}^{T} \frac{1}{\abs*{\FunSpace_{t}}}\sum\limits_{\fun\in \FunSpace_{t}} \loss{\fun(\covar\ind{t})}{\fstar(\covar\ind{t})} &= \sum\limits_{t=1}^{T} \frac{1}{\abs*{\FunSpace_{t}}}\sum\limits_{\fun\in \FunSpace_{t},\fun\neq \fstar} \loss{\fun(\covar\ind{t})}{\fstar(\covar\ind{t})} \\
  &\dfedit{= \sum\limits_{t=1}^{T} \sum\limits_{\fun:t_{\fun} = t}
    \sum\limits_{s\leq t}\frac{1}{\abs{\FunSpace_{s}}} 
    \loss{\fun(\covar\ind{s})}{\fstar(\covar\ind{s})}}
\end{align*}
\dfedit{Using that
  $\FunSpace_T\subseteq\FunSpace_{T-1}\subseteq\cdots\subseteq\FunSpace_1$,
  we can upper bound this quantity by}
\begin{align*}
  \sum\limits_{t=1}^{T} \sum\limits_{\fun:t_{\fun} = t}
  \sum\limits_{s\leq t}\frac{1}{\abs{\FunSpace_{s}}} 
  \loss{\fun(\covar\ind{s})}{\fstar(\covar\ind{s})} &\leq  \sum\limits_{t=1}^{T} \sum\limits_{\fun:t_{\fun} = t} \frac{1}{\abs{\FunSpace_{t}}} \sum\limits_{s\leq t}
    \loss{\fun(\covar\ind{s})}{\fstar(\covar\ind{s})}.
\end{align*}
To proceed, observe that for any function $f\in\cF$, if $t_f = t$,
then $f\in \cF_{t}$. It follows from the assumed bound in
\cref{ineq:local-set} that if $t_f=t$, then
\begin{align*}
  \sum\limits_{s\leq t}
  \loss{\fun(\covar\ind{s})}{\fstar(\covar\ind{s})} &= \loss{\fun(\covar\ind{t})}{\fstar(\covar\ind{t})} + \sum\limits_{s\leq t-1}
  \loss{\fun(\covar\ind{s})}{\fstar(\covar\ind{s})} \\
  &\leq 1+\offgan,
\end{align*} 
where we have used the fact that the loss $\divSymbol$ is bounded by $1$.
Using this fact, and recalling that for all $t\in [T]$,
$\abs*{\set{\fun:t_{\fun} = t}} = |\FunSpace_{t}\setminus
\FunSpace_{t+1}|$, we bound
\begin{align*}
\sum\limits_{t=1}^{T} \sum\limits_{\fun:t_{\fun} = t} \frac{1}{\abs{\FunSpace_{t}}} \sum\limits_{s\leq t}
  \loss{\fun(\covar\ind{s})}{\fstar(\covar\ind{s})}
  \leq (\offgan+1)\sum\limits_{t=1}^{T} \sum\limits_{\fun:t_{\fun} = t} \frac{1}{\abs{\FunSpace_{t}}} \leq (\offgan + 1)\sum\limits_{t=1}^{T} \frac{| \FunSpace_{t}\setminus \FunSpace_{t+1}|}{|\FunSpace_{t}|} .
\end{align*}
Finally, by \cref{lem:potential}, we have that 
\begin{align*}
  \sum\limits_{t=1}^{T} \frac{| \FunSpace_{t}\setminus \FunSpace_{t+1}|}{|\FunSpace_{t}|} = \sum\limits_{t=1}^{T} \frac{\abs{\FunSpace_{t}}-\abs{\FunSpace_{t+1}}}{|\FunSpace_{t}|} \leq \log |\FunSpace_1| = \log |\FunSpace|.
\end{align*}

We conclude that
\begin{align*}
  \sum\limits_{t=1}^{T}\En_{\fun\sim \mu\ind{t}}\brk*{\loss{\fun(\covar\ind{t})}{\fstar(\covar\ind{t})}} \leq (\offgan + 1)\log |\FunSpace|.
\end{align*}

We now prove the bound $\EstOnD(T)\leq  
\bigoh\prn*{ C_\divSymbol \cdot{}(\offgan+1)\cdot |\CovarSpace|\log
  T}$. For each $x\in\cX$, define $N_{t-1}(\covar) =
\sum\limits_{s=1}^{t-1} \mathbbm{1}(\covar\ind{s} = \covar)$. Then we
can write the online estimation error as
\begin{align*}
  \sum\limits_{t=1}^{T}\En_{\fun\sim \mu\ind{t}}\brk*{\loss{\fun(\covar\ind{t})}{\fstar(\covar\ind{t})}} = \sum\limits_{t=1}^{T} \sum\limits_{\covar\in \CovarSpace} \frac{\indic(\covar=\covar\ind{t})}{N_{t-1}(\covar)\vee 1}\cdot \En_{\fun\sim \mu\ind{t}}\brk*{(N_{t-1}(\covar)\vee 1) \loss{\fun(\covar)}{\fstar(\covar)}}.
\end{align*}
From the definition of $\mu\ind{t}$, we have that
\begin{align*}
  \En_{\fun\sim \mu\ind{t}}\brk*{(N_{t-1}(\covar)\vee 1) \loss{\fun(\covar)}{\fstar(\covar)}} &= \frac{1}{\abs*{\FunSpace_{t}}}\sum\limits_{\fun\in \FunSpace_{t}} (N_{t-1}(\covar)\vee 1) \loss{\fun(\covar)}{\fstar(\covar)}\\
  &\leq \frac{1}{\abs*{\FunSpace_{t}}}\sum\limits_{\fun\in \FunSpace_{t}} (N_{t-1}(\covar) + 1) \loss{\fun(\covar)}{\fstar(\covar)}.
\end{align*}
Now, from the offline guarantee assumed in \cref{ineq:local-set}, we
have for all $x\in\cX$ and $f\in\cF_t$,
\begin{align*}
  (N_{t-1}(\covar)+1)\loss{\fun(\covar)}{\fstar(\covar)} \leq \sum\limits_{s\leq t-1}
  \loss{\fun(\covar\ind{s})}{\fstar(\covar\ind{s})} +  \loss{\fun(\covar)}{\fstar(\covar)} \leq \offgan+1.
\end{align*}
Combining these observations, we have
\begin{align*}
  \sum\limits_{t=1}^{T}\En_{\fun\sim \mu\ind{t}}\brk*{\loss{\fun(\covar\ind{t})}{\fstar(\covar\ind{t})}} \leq (\offgan +1) \sum\limits_{t=1}^{T} \sum\limits_{\covar\in \CovarSpace} \frac{\indic(\covar=\covar\ind{t})}{N_{t-1}(\covar)\vee 1}.
\end{align*}
Now, for each $\covar\in \CovarSpace$, define $t_\covar =
\min\set{t\leq T: x\ind{t} =x }$ if this set is not empty, and set
$t_\covar = T$ otherwise. From this definition 
and the fact that $1+1/2+\dots+1/T\leq 1+\log T$, we have that 
\begin{align*}
  \sum\limits_{\covar\in \CovarSpace} \sum\limits_{t=1}^{T} \frac{\indic(\covar=\covar\ind{t})}{N_{t-1}(\covar)\vee 1} &= \sum\limits_{\covar\in \CovarSpace}  \prn*{\sum\limits_{t=1}^{t_\covar} \frac{\indic(\covar=\covar\ind{t})}{N_{t-1}(\covar)\vee 1} +  \sum\limits_{t=t_\covar+1}^{T} \frac{\indic(\covar=\covar\ind{t})}{N_{t-1}(\covar)\vee 1} }\\
  &\leq \sum\limits_{\covar\in \CovarSpace} \prn*{ 1 +  \sum_{i=1}^{N_{T-1}(x)} \frac{1}{i} } \\
    &\leq 2 |\CovarSpace| + |\CovarSpace|\log T\leq 3 |\CovarSpace|\log T.
\end{align*}
We conclude that
\begin{align*}
  \sum\limits_{t=1}^{T}\En_{\fun\sim \mu\ind{t}}\brk*{\loss{\fun(\covar\ind{t})}{\fstar(\covar\ind{t})}} &\leq  3(\offgan+1)\cdot|\CovarSpace|\log T .
\end{align*}
\end{proof}

\EstMemLowerBound*

\begin{proof}[\pfref{thm:woe-lower-bound}]
Let $N\geq 1$ be given, a consider the model class where $\CovarSpace
= \set{\covar_1,\dots,\covar_N}$ is an arbitrary discrete set,
$\ValSpace = \ObsSpace = \set{0,1}$, $\FunSpace =
\set{0,1}^\CovarSpace$, $\loss{\val_1}{\val_2} =
\Dbin{\val_1}{\val_2}=\indic\crl{z_1\neq{}z_2}$, and
$\Kernel(\val)=\indic_\val$.

We first specify the offline estimation oracle, then specify an
adversarially chosen covariate sequence. Fix $T\in\bbN$, and for any
$1\leq t\leq T$ and sequence of covariates
$\covar\ind{1},\dots,\covar\ind{t}$ define $N_t(\covar) \ldef{}
\sum_{s=1}^{t} \mathbbm{1}(\covar\ind{s}=\covar)$. For any target
parameter $\fstar\in \FunSpace$ and offline estimation parameter
$\offgan > 0$, we consider the oracle $\Orc\ind{t}(\cdot;\fstar)$ for the sequence $\covar\ind{1},\dots,\covar\ind{t}$ that returns
\begin{align*}
\fhat\ind{t}(\covar) =
\begin{cases}
0 &\text{if }N_{t-1}(\covar)< \offgan,\\
\fstar(\covar) &\text{otherwise}.
\end{cases}
\end{align*}

To complete the construction, we consider sequence
$(\covar\ind{1},\dots,\covar\ind{T})$ in which \[\covar\ind{t} =
  \covar_{\min\set{\lceil t/\lceil \offgan \rceil \rceil,N}}.\]
Equivalently, and perhaps more intuitively, we set
\[(\covar\ind{1},\dots,\covar\ind{T}) =
  (\underbrace{\covar_1,\dots,\covar_1}_{\lceil\offgan\rceil},\underbrace{\covar_2,\dots,\covar_2}_{\lceil\offgan\rceil},\dots,\underbrace{\covar_N,\dots,\covar_N}_{\lceil\offgan\rceil},\covar_N,\covar_N,\dots),\]
stopping earlier if $T \leq N \lceil\offgan \rceil$.

For any $\fstar$, we now show that
$\Orc\ind{t}(\cdot;\fstar)$ is an offline oracle with parameter
$\offgan$ on
the sequence $x\ind{1},\ldots,x\ind{T}$. This is because for any $i< \min\set{\lceil t/\lceil \offgan \rceil \rceil,N}$, the covariate $x_i$ is repeated $\lceil \offgan \rceil \geq \offgan$ times. Thus $\fhat\ind{t}(\covar_i) = \fstar(\covar_i)$. This implies for $t\leq  N\cdot \lceil \offgan \rceil$ that
\begin{align*}
  \sum\limits_{s=1}^{t-1}\Dbin{\fhat\ind{t}(\covar\ind{s})}{\fstar(\covar\ind{s})} = \sum\limits_{s= t- \lceil t/\lceil \offgan \rceil \rceil \cdot\lceil \offgan \rceil }^{t-1} \Dbin{\fhat\ind{t}(\covar\ind{s})}{\fstar(\covar\ind{s})} \leq \lceil \offgan \rceil -1 \leq \offgan.
\end{align*}
If $t>  N\cdot \lceil \offgan \rceil$, then all the covariates are repeated for more than $\lceil\offgan\rceil$ times, thus $\fhat\ind{t}=\fstar$. Overall, we have shown that
$\Orc\ind{t}(\cdot;\fstar)$ is an offline oracle with parameter
$\offgan$.

Now, fix any oracle-efficient online estimation
algorithm, and consider the expected regret under $\fstar\sim
\unif(\FunSpace)$ (with $\Orc\ind{t}(\cdot;\fstar)$ as the
oracle). If $T\geq{}N\ceil{\offgan}$, then regardless of how the algorithm chooses $\mu\ind{t}$, since for any block of $x_i$, $\fstar(x_i)$ is
independent of $\mu\ind{t}(x_i)$, the
expected regret is lower bounded as  
\begin{align*}
  &\sum\limits_{t=1}^{T}\En_{\fstar\sim \unif(\FunSpace)}\En_{\fbar\sim \mu\ind{t}}\brk*{\Dbin{\fbar(\covar\ind{t})}{\fstar(\covar\ind{t})}} \\
  &\geq \sum\limits_{i=1}^{N} \sum\limits_{j=1}^{\lceil\offgan\rceil} \En_{\fstar\sim \unif(\FunSpace)}\En_{\fbar\sim \mu\ind{i\lceil\offgan\rceil+j}}\brk*{\Dbin{\fbar(\covar_i)}{\fstar(\covar_i)}} \\
  &= \frac{1}{2}N \lceil \offgan \rceil \geq \Omega(N(\offgan+1)) = \Omega(\min\set{(\offgan+1)\log |\FunSpace|, (\offgan+1)|\CovarSpace|} ) . 
\end{align*}

\end{proof}

\subsection{Proofs from \creftitle{sec:memoryless}}

\SOELowerBoundImp*

\begin{proof}[\pfref{prop:separation_memoryless_improper}]
  \newcommand{\jhat}{\wh{j}}%
  Given a parameter $\offgan>0$ and an integer $N$, assume without loss of generality that $K\ldef{}T/(\lfloor \offgan \rfloor +1)$ is an integer. 
  Consider the instance
  $(\CovarSpace,\ObsSpace, \ValSpace,\Kernel,\FunSpace)$ with
  $\CovarSpace = [N]$, $\ValSpace =\ObsSpace =\set{0,1}$, $\divSymbol
  = \Dbinshort$, $\Kernel(\val)=\indic_\val$, and \FunSpaceName $\FunSpace = \set{\fun_i}_{i\in [N]}$ is defined as 
  \[
    f_i(\covar) = \indic\crl{\covar=i}.
  \]

  \jqedit{We consider a sequence of
    covariates $(\covar\ind{1},\dots,\covar\ind{T})$ divided into
    $K$ blocks, each with length
    $\lfloor \offgan \rfloor + 1$. In each block, the covariates will be chosen to be the same, i.e., $x\ind{1} = \dots = x\ind{\lfloor \offgan \rfloor + 1}$, $x\ind{\lfloor \offgan \rfloor + 2} = \dots = x\ind{2\lfloor
      \offgan \rfloor + 2}$.}\dfedit{ We define $\tau_t = \lceil  t/(
      \lfloor \offgan \rfloor + 1 )  \rceil$ as the index of the block
      the step $t$ belongs to, and we adopt the convention that
      $x_\tau\in\cX$ is value of the covariates for block $\tau$,
      i.e., $x\ind{t}=x_{\tau_t}$ for all $t$. We leave the precise
      choice for $x_1,\ldots,x_K$ as a free parameter for now.}

  Fix any memoryless oracle-efficient online estimation algorithm
  defined by a sequence of maps $\crl*{\Lrn\ind{t}}_{t\in\brk{T}}$
  (cf. \cref{def:mem-protocol}). We set the true target parameter to be
  $\fstar=\fun_{\istar}$, where the index $\istar\in\brk{N}$ will be chosen
  later in the proof (in an adversarial
  fashion based on the algorithm under consideration); for now, we
  leave $\istar\in\brk{N}$ as a free parameter.

  We first specify the offline estimation oracle under $\fun_{\istar}$. For each block index $\tau =1,\dots, K$, define
  \[
    \CovarSpace\ind{\tau}=\CovarSpace\setminus(\crl*{\covar\ind{s}}_{\jqedit{s\leq (\tau-1)(\lfloor\offgan\rfloor+1)}}\cup\crl{\istar})
  \]
  as the set of covariates in $\cX\setminus\crl{\istar}$ that have not been observed before block $\tau$, and let $\CovarSpacebar\ind{\tau}\ldef\CovarSpace\ind{\tau}\cup\crl{\istar}$. We define
  \[
    \fhat\ind{t}(\covar) = \indic\crl{\covar\in\CovarSpacebar\ind{\tau_t}}
  \]
  as the estimator returned by the oracle at round $t$.
  It is immediate from this construction that regardless of
  how $\covar\ind{1},\ldots,\covar\ind{T}$ are chosen, the offline
  estimation error is bounded by
  \begin{align*}
\jqedit{\forall t\in [T],~ \sum\limits_{s<t} \Bloss{\fstar(\covar\ind{s})}{\fhat\ind{t}(\covar\ind{s})}  \leq  \sum\limits_{s=(\tau_t-1)(\lfloor\offgan\rfloor+1) + 1}^{t-1} \Bloss{\fstar(\covar\ind{s})}{\fhat\ind{t}(\covar\ind{s})} \leq \lfloor\offgan\rfloor \leq \offgan},
  \end{align*}
  since the value of $\fhat\ind{t}$ differs from $\fun_{\istar}$ only for
  covariates that have not been observed before block $\tau_t$.

  It remains to lower bound the algorithm's online estimation
  error. \jqedit{To start, note that $\Dbinshort$ coincides with
    $\Dsqshort$ on the set $\set{0,1}$ and $\Dsqshort$ is convex.}
  Hence, by Jensen's inequality, it suffices to lower bound
\begin{align*}
  \sum_{t=1}^{T}(\jqedit{  f\ind{t}(\covar\ind{t}) }-f_{\istar}(\covar\ind{t}))^2,
\end{align*}
where $\jqedit{  f\ind{t}(\covar\ind{t}) }$ is the mean under
\jqedit{$\fbar\ind{t}(\covar\ind{t})$ where $\fbar\ind{t}\sim
  \mu\ind{t} = \Lrn\ind{t}(\fhat\ind{1},\ldots,\fhat\ind{t-1},\fhat\ind{t})$}. 

To proceed, we specify the sequence $x\ind{1},\ldots,x\ind{T}$, choosing
$\covar\ind{t}$ as a measurable function of $\CovarSpacebar\ind{t}$
and $\istar$ (recall that $\istar$ itself has yet to be chosen).  Fix a round $t$, and suppose that $\CovarSpace\ind{t}\neq\emptyset$. We choose
$\covar\ind{t}$ to lower bound the estimation error by considering two
cases. In the process, we will also define a function
$\jqedit{\jhat\ind{\tau}}:(2^{\brk{N}})^{\otimes \tau}\to\brk{N}\cup\crl{\perp}$ for
each $\tau\in [K]$, where $(2^{\brk{N}})^{\otimes \tau} = \underbrace{2^{\brk{N}} \times \dots \times 2^{\brk{N}}}_{\tau \text{~copies}}$. Let $\tau\in\brk{K}$ be fixed.
\begin{itemize} 
\item If \jqedit{$\sum\limits_{p=1}^{\lfloor\offgan\rfloor+1} \mathbbm{1}(f\ind{\tau(\lfloor\offgan\rfloor+1)+p}(\covar) \leq 1/2) \geq \frac{\lfloor\offgan\rfloor+1}{2}$ for all $\covar\in \CovarSpacebar\ind{\tau}$}, 
we set $x_\tau=\istar$ (equivalently, 
  $\covar\ind{\tau(\lfloor\offgan\rfloor+1)+p}=\istar$ for $p=1,\dots,\lfloor\offgan\rfloor+1$), so that
  \[
    \jqedit{\sum\limits_{t= \tau(\lfloor\offgan\rfloor+1)+1}^{(\tau+1)(\lfloor\offgan\rfloor+1)}\max_{\covar\ind{t}}
      (f\ind{t}(\covar\ind{t})-f_{\istar}(\covar\ind{t}))^2
    \geq{} \frac{\lfloor\offgan\rfloor+1}{8} .}
  \]
  In this case, we define $\jhat\ind{\tau}(\CovarSpacebar\ind{1},\dots,\CovarSpacebar\ind{\tau})=\perp$.
\item If there exists $j\in\CovarSpacebar\ind{\tau}$ such that
\jqedit{$\sum\limits_{p=1}^{\lfloor\offgan\rfloor+1} \mathbbm{1}(f\ind{\tau(\lfloor\offgan\rfloor+1)+p}(j) \leq 1/2) < \frac{\lfloor\offgan\rfloor+1}{2}$},
  we set $x_\tau=j$ (equivalently, 
  $\covar\ind{\tau(\lfloor\offgan\rfloor+1)+p}=j$ for
  $p=1,\dots,\lfloor\offgan\rfloor+1$) \dfedit{for the least such $j$}, so that
  \begin{align*}
    \sum\limits_{t= \tau(\lfloor\offgan\rfloor+1)+1}^{(\tau+1)(\lfloor\offgan\rfloor+1)}\max_{\covar\ind{t}}
      (f\ind{t}(\covar\ind{t})-f_{\istar}(\covar\ind{t}))^2
    \geq{} \frac{\lfloor\offgan\rfloor+1}{8}\indic\crl*{\istar\neq{}j} .
  \end{align*}
  In this case, we define
  $\jhat\ind{\tau}(\CovarSpacebar\ind{1},\dots,\CovarSpacebar\ind{\tau})=j$ as well.
\end{itemize}
Note that since $\jqedit{f\ind{t}}$ is a measurable function of
$\CovarSpacebar\ind{1}$, $\jhat\ind{1}$,\dots,$\CovarSpacebar\ind{\tau_t}$, $\jhat\ind{\tau_t}$ is well-defined.

Combining these cases, it follows that for any choice of $\istar$,
choosing $x_1,\ldots,x_K$ in the fashion described above ensures that
\begin{align*}
  \sum_{t=1}^{T}(\jqedit{f\ind{t}}(\covar\ind{t})-f_{\istar}(\covar\ind{t}))^2
  \geq \frac{\lfloor\offgan\rfloor+1}{8}\sum_{\tau=1}^{K}\indic\crl*{\jhat\ind{\tau}(\CovarSpacebar\ind{1},\dots, \CovarSpacebar\ind{\tau})\neq\istar}.
\end{align*}
We now state and prove the following technical lemma, which asserts
that there exists a choice of $\istar$ for which the right-hand side
above is large.
    \begin{lemma}
      \label{prop:posterior_uniform}
      For any algorithm, there exists a choice for $\istar\in\brk*{N}$ such that
      \[
        \min\crl*{\tau : \jhat\ind{\tau}(\CovarSpacebar\ind{1},\dots, \CovarSpacebar\ind{\tau})=\istar} \geq{} \bigom(N).
      \]
    \end{lemma}
  \begin{proof}[\pfref{prop:posterior_uniform}]
    Consider a more abstract process, which we claim captures the
    evolution of $\CovarSpacebar\ind{\tau}$. Let $\istar\in\brk*{N}$,
    and let $A\ind{1}=\brk{N}$. We consider a sequence of sets $\crl*{A\ind{\tau}}_{\tau\geq{}1}$ evolving according
    to the following process, parameterized by a sequence of functions
    $\crl*{g\ind{\tau}:(2^{\brk{N}})^{\otimes \tau}\to\brk{N}\cup\crl{\perp}}_{\tau\geq{}1}$ and index
    $\istar\in\brk{N}$.

    For $\tau\geq{}1$:
    \begin{itemize}
    \item If $g\ind{\tau}(A\ind{1},\dots,A\ind{\tau})=\perp$, $A\ind{\tau+1}\gets{}A\ind{\tau}$.
    \item If $g\ind{\tau}(A\ind{1},\dots,A\ind{\tau}) \neq \perp$, let
      $A\ind{\tau+1}=A\ind{\tau}\setminus\crl{g\ind{\tau}(A\ind{1},\dots,A\ind{\tau})}$ if
      $g\ind{\tau}(A\ind{1},\dots,A\ind{\tau})\neq\istar$, and let
      $A\ind{\tau+1}\gets{}A\ind{\tau}$ otherwise.
    \end{itemize}
    We claim that there exists $\istar\in\brk{N}$ such that
    $g\ind{s}(A\ind{1},\dots,A\ind{s})\neq\istar$ for all $s<N$ under this process. To see
    this define a set $X\ind{\tau}$ inductively: Starting from
    $X\ind{1}=\brk{N}$, set
    $X\ind{\tau+1}\gets{}X\ind{\tau}\setminus{}g\ind{\tau}(X\ind{1},\dots,X\ind{\tau})$ if
    $g\ind{\tau}(X\ind{1},\dots,X\ind{\tau})\neq\perp$, and set $X\ind{\tau+1}\gets{}X\ind{\tau}$
    otherwise; note that this process does not depend on the choice $\istar$,
    since $g\ind{\tau}$ itself does not depend on $\istar$.

For any $\tau$, observe that for any $i\in{}X\ind{\tau}$, if we set
$\istar=i$, then $g\ind{s}(A\ind{1},\dots,A\ind{s})\neq{}\istar$ for all $s<\tau$, and so
$A\ind{\tau}=X\ind{\tau}$. It follows that as long as
$X\ind{\tau}\neq\emptyset$, we can choose $\istar$ such that
$g\ind{s}(A\ind{1},\dots,A\ind{s})\neq{}\istar$ for all $s<\tau$. Since $X\ind{\tau}$
shrinks by at most one element per iteration, it follows that this is
possible for all $\tau<N$.

  \end{proof}
    
    It follows immediately from \cref{prop:posterior_uniform} that by choosing $\istar$ as guaranteed by
    the lemma, we have
  \begin{align*}
    \sum_{t=1}^{T}(\jqedit{f\ind{t}}(\covar\ind{t})-f_{\istar}(\covar\ind{t}))^2
    &\geq \frac{\lfloor\offgan\rfloor+1}{8} \cdot \bigom(\min\set{N,K})\\
    &\geq \Omega ( \min \set{N(\offgan+1) ,T}).
  \end{align*}

\end{proof}

\SOELowerBound*

  \begin{proof}[Proof of \refmemorylessvar]%
   Given a parameter $N\in\bbN$, we consider the instance
   $(\CovarSpace,\ObsSpace, \ValSpace,\Kernel,\FunSpace)$ given by
   $\CovarSpace = \set{\covar_i}_{i\in [N]}$, $\ValSpace =\ObsSpace
   =\set{0,1}$, $\divSymbol = \Dbinshort$ and $\Kernel(\val)=\indic_\val$,
   with \FunSpaceName $\FunSpace \ldef{}\set{\fun_i}_{i\in [N]}$ given by
\begin{align*}
\fun_i(\covar_j) = \mathbbm{1}(j\geq i).
\end{align*}

Let any memoryless oracle-efficient algorithm defined by prediction
map $\Lrn\ind{1} = \dots = \Lrn\ind{T}=\Lrn$ be given. We lower
bound the algorithm's online estimation error by considering two cases.

\textbf{Case 1:} \emph{There exists a \funName $\fun_i$ such that the distribution $\mu_i = \Lrn(\fun_i)$ satisfies $\mu_i(\fun_i) < 1/2$.}
We consider two sub-cases of Case 1.
The first subcase is where $\mu_i\prn*{ \set{ \fun_j: j>i} } > 1/4$.
In this case, we choose the sequence of covariates as
$\covar\ind{1}=\dots=\covar\ind{T} = \covar_{i}$, set $\fstar=\fun_i$,
and choose $\AlgOff$ to be the offline estimation oracle that sets
$\fhat\ind{1}=\dots=\fhat\ind{T}=\fun_i$. With this choice, the
offline estimation error for the oracle satisfies
\begin{align*}
\forall t\in [T],~ \sum\limits_{s< t} \Dbin{\fstar(\covar\ind{s})}{\fhat\ind{t}(\covar\ind{s})}  =\sum\limits_{s< t} \Dbin{\fun_i(\covar\ind{s})}{\fun_i(\covar\ind{s})} = 0.
\end{align*} 
However, the online estimation error satisfies 
\begin{align*}
    \sum\limits_{t=1}^{T}\En_{\fbar\sim \mu\ind{t}}\brk*{\Dbin{\fstar(\covar\ind{t})}{\fbar(\covar\ind{t})}} &= \sum\limits_{t=1}^{T}\En_{\fbar\sim \mu_i}\brk*{\Dbin{f_i(x_i)}{\fbar(x_i)}} \\
    &\geq \sum\limits_{t=1}^{T} \sum\limits_{j>i} \mu_i\prn*{ \fun_j }  \Dbin{f_i(x_i)}{f_j(\covar_i)}\\
    &=  \sum\limits_{t=1}^{T} \sum\limits_{j>i} \mu_i\prn*{ \fun_j }  \Dbin{f_i(x_i)}{f_{i+1}(\covar_i)}\\
    & \geq \frac{1}{4} \sum\limits_{t=1}^{T}\Dbin{\fun_i(\covar_i)}{\fun_{i+1}(\covar_i)} = T/4.
\end{align*}

The second sub-case of Case 1 is where $\mu_i\prn*{ \set{ \fun_j: j>i}
} \leq 1/4$. This, combined with the fact that $\mu_i(\fun_i) < 1/2$,
gives $\mu_i \prn*{  \set{\fun_j: j<i}  } > 1/4$. In this sub-case, we
choose the sequence of covariates as
$\covar\ind{1}=\dots=\covar\ind{T} = \covar_{i-1}$, set
$\fstar=\fun_i$, and choose $\AlgOff$ to be the offline estimation
oracle that sets $\fun\ind{1}=\dots=\fun\ind{T}=\fun_i$. In this case,
the offline estimation error is zero:
\begin{align*}
\forall t\in [T],~ \sum\limits_{s< t} \Dbin{\fstar(\covar\ind{s})}{\fhat\ind{t}(\covar\ind{s})}  =\sum\limits_{s< t} \Dbin{\fun_i(\covar\ind{s})}{\fun_i(\covar\ind{s})} = 0.
\end{align*} 
In addition, the online estimation error is lower bounded by
\begin{align*}
    \sum\limits_{t=1}^{T}\En_{\fbar\sim \mu\ind{t}}\brk*{\Dbin{\fstar(\covar\ind{t})}{\fbar(\covar\ind{t})}} 
    &= \sum\limits_{t=1}^{T}\En_{\fbar\sim \mu_i}\brk*{\Dbin{f_i(x_{i-1})}{\fbar(x_{i-1})}} \\
    &\geq \sum\limits_{t=1}^{T} \sum\limits_{j<i} \mu_i\prn*{ \fun_j }  \Dbin{f_i(x_{i-1})}{f_j(\covar_{i-1})}\\
    &=  \sum\limits_{t=1}^{T} \sum\limits_{j<i} \mu_i\prn*{ \fun_j }  \Dbin{f_i(x_{i-1})}{f_{i-1}(\covar_{i-1})}\\
    &\geq \frac{1}{4} \sum\limits_{t=1}^{T}\Dbin{\fun_i(\covar_{i-1})}{\fun_{i-1}(\covar_{i-1})} = T/4.
\end{align*}

\textbf{Case 2:} \emph{For all \funName $\fun_i\in \FunSpace$, the distribution $\mu_i = \Lrn(\fun_i)$ has $\mu_i(\fun_i) \geq 1/2$.}

In this case, we choose $x\ind{1},\ldots,x\ind{T}$ by repeating each
of the covariates $\lfloor \offgan \rfloor + 1$ number of times in
increasing order by their index, and choose the offline estimation
oracle $\AlgOff$ to return $f_1,\ldots,f_N$ in the same block-wise but
with the index offset by $1$.

Formally, let $\tau_t = \lceil  t/( \lfloor \offgan \rfloor + 1 )
\rceil$ denote the index of the block that round $t$ belongs to, so
that $\tau_1 = \dots = \tau_{\lfloor \offgan \rfloor + 1} = 1$,
$\tau_{\lfloor \offgan \rfloor + 2} = \dots = \tau_{2\lfloor \offgan
  \rfloor + 2} = 2 $ and so on. We choose $\covar\ind{t} =
\covar_{\min\set{\tau_t,N}}$, and choose the offline oracle $\AlgOff$ to
set $\fhat\ind{t} = \fun_{\min\set{\tau_{t},N}}$. Finally, we set
$\fstar=\fun_N$.

We have that for all $t$, the offline estimation error of the oracle
is bounded as.
\begin{align*}
    \sum\limits_{s< t} \Dbin{\fhat\ind{t}(\covar\ind{s})}{\fstar(\covar\ind{s})}  &=\sum\limits_{s<t} \Dbin{\fun_N(\covar_{\min\set{\tau_s,N}})}{\fun_{\min\set{\tau_{t},N}}(\covar_{\min\set{\tau_s,k}})} \\
    &= \sum\limits_{s<t, \tau_s = \tau_{t}} 1\leq  \lfloor \offgan \rfloor \leq \offgan.
\end{align*} 

However, the online estimation error is lower bounded by
\begin{align*}
  \sum\limits_{t=1}^{T} \En_{\fbar\sim \mu\ind{t}} \brk*{\Dbin{\fstar(\covar\ind{t})}{\fbar(\covar\ind{t})}} &\geq \frac{1}{2} \sum\limits_{t=1}^{T}\Dbin{\fun_N(\covar_{\min\set{\tau_t,N}})}{\fun_{\min\set{\tau_t,N}}(\covar_{\min\set{\tau_t,N}})} \\
    &\geq \Omega(\min\set{T, N(\offgan + 1)}).
\end{align*}

\end{proof}

\soeupperbound*

\begin{proof}[\pfref{prop:soe-upper-bound}]
The proof is very similar to the second part of the proof of
\cref{lem:general-version-space}. For each $x\in\cX$, define $N_{t-1}(\covar) =
\sum\limits_{s=1}^{t-1} \mathbbm{1}(\covar\ind{s} = \covar)$. Then we
can write the online estimation error as
\begin{align*}
    \sum\limits_{t=1}^{T}\loss{\fhat\ind{t}(\covar\ind{t})}{\fstar(\covar\ind{t})} = \sum\limits_{t=1}^{T} \sum\limits_{\covar\in \CovarSpace} \frac{\indic(\covar=\covar\ind{t})}{N_{t-1}(\covar)\vee 1}\cdot (N_{t-1}(\covar)\vee 1) \cdot\loss{\fhat\ind{t}(\covar)}{\fstar(\covar)},
\end{align*}
As a consequence of the offline estimation guarantee for $\fhat\ind{t}$, we have that
\begin{align*}
  (N_{t-1}(\covar)\vee 1) \loss{\fhat\ind{t}(\covar)}{\fstar(\covar)} \leq   \prn*{\sum\limits_{s=1}^{t-1} \loss{\fhat\ind{t}(\covar\ind{s})}{\fstar(\covar\ind{s})}} \vee 1\leq \offgan + 1.
\end{align*}
Combining this with the preceding inequality gives
\begin{align*}
  \sum\limits_{t=1}^{T}\loss{\fhat\ind{t}(\covar\ind{t})}{\fstar(\covar\ind{t})} \leq (\offgan +1) \sum\limits_{t=1}^{T} \sum\limits_{\covar\in \CovarSpace} \frac{\indic(\covar=\covar\ind{t})}{N_{t-1}(\covar)\vee 1}.
\end{align*}
Now, for any $\covar\in \CovarSpace$, define $t_\covar \ldef
\min\set{t\leq T: x\ind{t}=x}$ if this set is not empty, and let
$t_\covar = T$ otherwise. From this definition 
and the fact that $1+1/2+\dots+1/T\leq 1+\log T$, we have that 
\begin{align*}
  \sum\limits_{\covar\in \CovarSpace} \sum\limits_{t=1}^{T} \frac{\indic(\covar=\covar\ind{t})}{N_{t-1}(\covar)\vee 1} &= \sum\limits_{\covar\in \CovarSpace}  \prn*{\sum\limits_{t=1}^{t_\covar} \frac{\indic(\covar=\covar\ind{t})}{N_{t-1}(\covar)\vee 1} +  \sum\limits_{t=t_\covar+1}^{T} \frac{\indic(\covar=\covar\ind{t})}{N_{t-1}(\covar)\vee 1} }\\
  &\leq \sum\limits_{\covar\in \CovarSpace} \prn*{ 1  +    \sum_{i=1}^{N_{T-1}(x)} \frac{1}{i} } \\
    &\leq 2 |\CovarSpace| + |\CovarSpace|\log T\leq 3 |\CovarSpace|\log T.
\end{align*}
We conclude that
\begin{align*}
  \sum\limits_{t=1}^{T}\loss{\fstar(\covar\ind{t})}{\fhat\ind{t}(\covar\ind{t})}   &\leq  3(\offgan+1)\cdot  |\CovarSpace|\log T.
\end{align*}

\end{proof}

\subsection{Proofs from \creftitle{sec:delayed}}
\label{app:delayed}
\OELtoDOL*

\begin{proof}[\pfref{thm:delayed_reduction}]
  Using the metric-like loss property, we can bound the online
  estimation error of \cref{alg:reduction-to-delayed-OL} by
  \begin{align*}
&\sum_{t=1}^{T}\En_{\fbar\sim{}\mu\ind{t}}\brk*{\loss{\fbar(\covar\ind{t})}{\fstar(\covar\ind{t})}}
    \\
&\leq
C_\divSymbol\cdot \sum_{t=1}^{T}\En_{\fbar\sim{}\mu\ind{t}}\brk*{\loss{\fbar(\covar\ind{t})}{\ftil\ind{t}(\covar\ind{t})}}
         + C_\divSymbol\cdot\sum_{t=1}^{T}\loss{\ftil\ind{t}(\covar\ind{t})}{\fstar(\covar\ind{t})}.
  \end{align*}
By the regret guarantee for the delayed online learning algorithm $\AlgDOL$, we have
\begin{align*}
  \sum_{t=1}^{T}\En_{\fbar\sim{}\mu\ind{t}}\brk*{\loss{\fbar(\covar\ind{t})}{\ftil\ind{t}(\covar\ind{t})}} \leq \gamma\cdot{} \sum_{t=1}^{T}\loss{\ftil\ind{t}(\covar\ind{t})}{\fstar(\covar\ind{t})} + \RDOL(T,N,\gamma)
\end{align*}
since $\fstar\in\cF$. Combining these observations, we have that
\begin{align*}
\sum_{t=1}^{T}\En_{\fbar\sim{}\mu\ind{t}}\brk*{\loss{\fbar(\covar\ind{t})}{\fstar(\covar\ind{t})}}
  \leq{}
  C_\divSymbol(\gamma +1)\sum_{t=1}^{T}\loss{\ftil\ind{t}(\covar\ind{t})}{\fstar(\covar\ind{t})}
  +  \RDOL(T,N,\gamma).
\end{align*}
Finally, from the definition of $\ftil\ind{t}$ and the convexity of
the loss $\divSymbol$, we have 
  \begin{align*}
    \sum_{t=1}^{T}\loss{\ftil\ind{t}(\covar\ind{t})}{\fstar(\covar\ind{t})}
    &\leq{} N +
\frac{1}{N}\sum_{t=1}^{T-N}\sum_{i=t+1}^{t+N}\loss{\fhat\ind{i}(\covar\ind{t})}{\fstar(\covar\ind{t})}
    \\
    &=
      N+ \frac{1}{N}\sum_{t=2}^{T}\sum_{i<t}\loss{\fhat\ind{t}(\covar\ind{i})}{\fstar(\covar\ind{i})}\\
    &\leq{} N+ \frac{\offgan{}T}{N},
  \end{align*}
  where the final line uses the offline estimation guarantee for
  $\AlgOff$. This completes the proof.

\end{proof}

\DOLToOL*

\begin{proof}[\pfref{prop:delayed}]
  This result follows using \cref{lem:delayed-to-ol} with $\gamma=2$,
  choosing $\AlgOL$ to be the exponential weights
  algorithm described in Corollary 2.3 of
  \citet{cesa2006prediction}, which has
  \begin{align*}
    \ROL(T,\gamma) \leq \bigoh(\log |\FunSpace|)
  \end{align*}
for all $T\in\bbN$ and $\gamma\geq{}1$.

\end{proof}

\BinaryLearnability*

\begin{proof}[\pfref{thm:binary-loss-learnability}]
For the lower bound we recall that for $\offgan = 0$, Lemma 21.6 of
\cite{shalev2014understanding} states that any algorithm
(oracle-efficient or not) has to suffer $\Omega(\Ldim(\FunSpace))$
online estimation error in the worst case.

For the remainder of the proof, we focus on establishing the upper
bound. \dfedit{For any set of \funNames
  $\FunSpace:\cX\to\Delta(\set{0,1}) $, define the majority vote function
  $\Major(\FunSpace)$ for a class $\FunSpace$ via
  \[\Major(\FunSpace)(x) = \mathbbm{1}\crl*{ \sum\limits_{\fun\in
        \FunSpace} \fun(1\mid\covar) \geq \sum\limits_{\fun\in
        \FunSpace} \fun(0\mid\covar) }\]
  for all $\covar\in \CovarSpace$.} We will show that
\cref{alg:reduction-to-delayed-OL-4-binary} (a variant of
\cref{alg:reduction-to-delayed-OL} that replaces averaging with a
majority vote), with a properly chosen
delayed online learning algorithm $\AlgDOL$, can obtain
\[\bigoh\prn*{\sqrt{\offgan{} \Ldim(\FunSpace) \cdot T\log T } +
  \Ldim(\FunSpace)\log T}\] online estimation error. 

\begin{algorithm}[t]
\caption{Reduction to delayed online learning for binary loss}
\label{alg:reduction-to-delayed-OL-4-binary}
\begin{algorithmic}[1]
  \State \textbf{input:} Offline estimation oracle $\Orc$ with
  parameter $\offgan\geq{}0$, delay parameter $N\in\bbN$, delayed
  online learning algorithm $\AlgDOL$ for class $\cF$.
\For{$t=1,2,\dots,T$}
\State Receive $\fhat\ind{t}$ from offline estimation algorithm.
\If{$t> N$}
\State Let $\ftil\ind{t-N} = \Major\prn[\big]{\crl{\fhat\ind{i}}_{i=t-N+1}^{t}}$. \algcommentlight{This is the only different step compared to \cref{alg:reduction-to-delayed-OL}.}
\State \multiline{Let 
$\ell\ind{t-N}(\fun) =  \Bloss{\ftil\ind{t-N}(\covar\ind{t-N})}{\fun(\covar\ind{t-N})} $ and pass $\ell\ind{t-N}(\cdot)$ to $\AlgDOL$ as the delayed feedback.}
\EndIf
\State \multiline{Let
$\mu\ind{t}=\AlgDOL\ind{t}(\ell\ind{1},..,\ell\ind{t-N})$ be the
delayed online learner's prediction distribution.}
\State Predict with $\dfedit{\fbar\ind{t}\sim}\mu\ind{t}$ and receive $\covar\ind{t}$.
\EndFor
\end{algorithmic}
\end{algorithm}

Let $\gamma\geq 1$ be fixed, and consider any delayed online learning
algorithm $\AlgDOL$ that achieves
    \begin{align*}
    \sum_{t=1}^{T}\En_{\fbar\ind{t}\sim{}\mu\ind{t}}\brk*{\ell\ind{t}(\fbar\ind{t})}
    - \gamma\cdot{}\min_{\fun\in\FunSpace}\sum_{t=1}^{T}\ell\ind{t}(\fun) \leq{} \RDOL(T,N,\gamma).
    \end{align*}
    for any sequence of losses in the delayed online learning setting
    with delay $N$ (i.e., where we receive loss $\ell\ind{t}$
  at time $t+N$ for some $N\geq{}0$). 

  We proceed to bound the regret of
  \cref{alg:reduction-to-delayed-OL-4-binary}. Since the loss $\Dbinshort$ is \mlike, the online estimation error is
upper bounded by 
  \begin{align*}
\sum_{t=1}^{T}\En_{\fbar\sim\mu\ind{t}}\brk*{\Bloss{\fbar(\covar\ind{t})}{\fstar(\covar\ind{t})}} \leq \sum_{t=1}^{T}\En_{\fbar\sim\mu\ind{t}}\brk*{\Bloss{\fbar(\covar\ind{t})}{\ftil\ind{t}(\covar\ind{t})}}
         + \sum_{t=1}^{T}\Bloss{\ftil\ind{t}(\covar\ind{t})}{\fstar(\covar\ind{t})}.
  \end{align*}
Next, the guarantee of $\AlgDOL$ ensures that 
  \begin{align*}
\sum_{t=1}^{T}\En_{\fbar\sim\mu\ind{t}}\brk*{\Bloss{\fbar(\covar\ind{t})}{\ftil\ind{t}(\covar\ind{t})}}
    \leq{} 
    \gamma\sum_{t=1}^{T}\Bloss{\ftil\ind{t}(\covar\ind{t})}{\fstar(\covar\ind{t})}
    +\RDOL(T,N,\gamma),
  \end{align*}
  since $\fstar\in\cF$.
  Combining these observations, we have that
  \begin{align*}
\sum_{t=1}^{T}\En_{\fbar\sim\mu\ind{t}}\brk*{\Bloss{\fbar(\covar\ind{t})}{\fstar(\covar\ind{t})}}
    \leq{}
    (\gamma+1)\sum_{t=1}^{T}\Bloss{\ftil\ind{t}(\covar\ind{t})}{\fstar(\covar\ind{t})}
    +\RDOL(T,N,\gamma).
  \end{align*}
  Finally, we observe that for each step $t$, if
  $\Bloss{\ftil\ind{t}(\covar\ind{t})}{\fstar(\covar\ind{t})} = 1$, it
  means that least $N/2$ of the predictors
  $\fhat\ind{t+1},\ldots,\fhat\ind{t+N}$ must have predicted
  $\fstar(x\ind{t})$ incorrectly. This implies that
  \begin{align*}
    \Bloss{\ftil\ind{t}(\covar\ind{t})}{\fstar(\covar\ind{t})} \leq  \frac{2}{N}\sum\limits_{i=t+1}^{t+N} \Bloss{\fhat\ind{i}(\covar\ind{t})}{\fstar(\covar\ind{t})}.
  \end{align*} 
  But since the offline estimation
  assumption states that
  \[
    \sum_{i<t}\Bloss{\fhat\ind{t}(\covar\ind{i})}{\fstar(\covar\ind{i})}\leq\offgan,
  \]
  this implies that
  \begin{align*}
    \sum_{t=1}^{T}\Bloss{\ftil\ind{t}(\covar\ind{t})}{\fstar(\covar\ind{t})} &\leq{} N +   \frac{2}{N}\sum\limits_{t=1}^{T-N}\sum\limits_{i=t+1}^{t+N} \Bloss{\fhat\ind{i}(\covar\ind{t})}{\fstar(\covar\ind{t})}\\
    &= N+ \frac{2}{N} \sum\limits_{t=2}^{T}\sum\limits_{i<t} \Bloss{\fhat\ind{t}(\covar\ind{i})}{\fstar(\covar\ind{t})}  \\
    &\leq{} N+ \frac{2\offgan{}T}{N}.
  \end{align*}
We conclude that
\begin{align*}
  \sum_{t=1}^{T}\En_{\fbar\ind{t}\sim{}\mu\ind{t}}\brk*{\Bloss{\fbar\ind{t}(\covar\ind{t})}{\fstar(\covar\ind{t})}}
  \leq \bigoh\prn*{
   \gamma(N+ \offgan{}T/N)+\RDOL(T,N,\gamma)}.
\end{align*}

To complete the proof, we set $\gamma=2$ and choose $\AlgOL$ to be the
algorithm described in Theorem 21.10 of
\citet{shalev2014understanding}, which (by incorporating the same
technique\footnote{\jqedit{The algorithm described in Theorem 21.10 of
    \citet{shalev2014understanding} applies the exponential weights
    algorithm to a specialized class of experts, and the guarantee
    obtained is for $\ROL(T, 1)$. The analysis from Corollary 2.3 of
    \citet{cesa2006prediction} shows that for $\gamma > e/(e-1)$, the
    same algorithm obtains $\ROL(T, 2)$ scaling with
    $\bigoh(\Ldim(\FunSpace)\log T)$. We omit the details
    here since it is a standard argument.}
} as Corollary 2.3 of \citet{cesa2006prediction}), ensures
that 
\begin{align*}
\ROL(T/N, 2) \leq \bigoh(\Ldim(\FunSpace)\log T).
\end{align*}
Then by \cref{lem:delayed-to-ol} with $\gamma=2$, we have
\begin{align*}
  \RDOL(T,N,2)\leq   \bigoh(N\cdot\Ldim(\FunSpace)\log T).
\end{align*}
Setting $N=\sqrt{\offgan T/(\Ldim(\FunSpace)\log T)} \vee 1$, this yields
\begin{align*}
  \sum_{t=1}^{T}\En_{\fbar\ind{t}\sim{}\mu\ind{t}}\brk*{\loss{\fbar\ind{t}(\covar\ind{t})}{\fstar(\covar\ind{t})}}
  \leq \bigoh\prn*{\sqrt{\offgan{} \Ldim(\FunSpace) \cdot T\log T } + \Ldim(\FunSpace)\log T}.
\end{align*}

\end{proof}

\subsubsection{Supporting Lemmas}

\begin{algorithm}[htp]
\caption{Reduction from delayed online learning to non-delayed online learning}
\label{alg:reduction-to-OL}
\begin{algorithmic}[1]
  \State \textbf{input:} Delay parameter $N\in\bbN$, base online
  learning algorithm $\AlgOL$.
  \State Initialize $N$ copies $\AlgOL^1,\dots,\AlgOL^N$ of the base algorithm.
\For{$t=1,\dots,T$}
\If{$t\leq N$}
\State Let $\mu\ind{t} = \AlgOL^t(\emptyset)$.
\Else
\State Let $i\equiv t \mod N$ where $i\in [N]$.
\State Receive loss $\ell\ind{t-N}$.
\State Feed $\ell\ind{t-N}$ to $\AlgOL^i$.
\State Let $\mu\ind{t}=\AlgOL^i( \ell\ind{i},\ell\ind{i+N},\dots,\ell\ind{t-N})$.
\EndIf
\State Play $\fbar\ind{t}\sim \mu\ind{t}$.
\EndFor
\end{algorithmic}
\end{algorithm}

The following lemma is a standard result
\citep{weinberger2002delayed,mesterharm2007improving,joulani2013online,quanrud2015online}
which shows that the delayed online learning problem setting in
\cref{sec:delayed} can be generically reduced to non-delayed online
learning. The idea behind the reduction, which is displayed in
\cref{alg:reduction-to-OL}, is as follows. Given a delay parameter
$N\in \bbN$, we run $N$ copies $\AlgOL^1,\dots,\AlgOL^N$ of a given
``base'' online learning algorithm $\AlgOL$ for a class $\cF$ over disjoint subsequences
of rounds. The following lemma gives a guarantee for this reduction

\begin{lemma}[Delayed online learning reduction]
  \label{lem:delayed-to-ol}
  Let 
  $\AlgOL$ be a base online learning algorithm for the class $\cF$ with the property that for any sequence
  of losses $\ell\ind{1},\ldots,\ell\ind{T}$ in the non-delayed online
  learning setting and any $\gamma\geq{}1$,
  \[
    \sum_{t=1}^{T}\En_{\fbar\ind{t}\sim\mu\ind{t}}\brk*{\ell\ind{t}(\fbar\ind{t})}
    - \gamma\cdot\min_{f\in\cF}\sum_{t=1}^{T}\ell\ind{t}(f)
    \leq \ROL(T,\gamma).
  \]
If we run \cref{alg:reduction-to-OL}
  with delay
  parameter $N\in\bbN$, then for all $\gamma\geq 1$, the algorithm
  ensures that
  \begin{align*}
    \RDOL(T,N,\gamma) &\leq{}
\sum_{i=1}^N \prn*{\sum_{j=1}^{T/N}\En_{\fbar\ind{i+N\cdot j}\sim{}\mu\ind{i+N\cdot j}}\brk*{\ell\ind{i+N\cdot j}(\fbar\ind{i+N\cdot j})}
    -
    \gamma\cdot{}\min_{\fun\in\FunSpace}\sum_{j=1}^{T/N}\ell\ind{i+N\cdot
    j}(\fun)}\\
&\leq{}    
    N\cdot{}\ROL(T/N,\gamma)
  \end{align*}
  for online learning with delay $N$.
\end{lemma}

\begin{proof}[\pfref{lem:delayed-to-ol}]
By the guarantee of $\AlgOL$, we have that for all $i\in [N]$,
\begin{align*}
  \sum_{j=1}^{T/N}\En_{\fbar\ind{i+N\cdot j}\sim{}\mu\ind{i+N\cdot j}}\brk*{\ell\ind{i+N\cdot j}(\fbar\ind{i+N\cdot j})}
  - \gamma\cdot{}\min_{\fun\in\FunSpace}\sum_{j=1}^{T/N}\ell\ind{i+N\cdot j}(\fun) \leq{} \ROL(T/N,\gamma).
\end{align*}
Summing up over all $i\in [N]$, we obtain
\begin{align*}
  N\cdot \ROL(T/N,\gamma)&= \sum\limits_{i=1}^{N} \sum_{j=1}^{T/N}\En_{\fbar\ind{i+N\cdot j}\sim{}\mu\ind{i+N\cdot j}}\brk*{\ell\ind{i+N\cdot j}(\fbar\ind{i+N\cdot j})}
 -\gamma\cdot \sum\limits_{i=1}^{N} \min_{\fun\in\FunSpace}\sum_{j=1}^{T/N}\ell\ind{i+N\cdot j}(\fun) \\
 &\geq \sum\limits_{i=1}^{N} \sum_{j=1}^{T/N}\En_{\fbar\ind{i+N\cdot j}\sim{}\mu\ind{i+N\cdot j}}\brk*{\ell\ind{i+N\cdot j}(\fbar\ind{i+N\cdot j})}
 -\gamma\cdot \min_{\fun\in\FunSpace}\sum\limits_{i=1}^{N} \sum_{j=1}^{T/N}\ell\ind{i+N\cdot j}(\fun) \\
 &= \sum_{t=1}^{T}\En_{\fbar\ind{t}\sim{}\mu\ind{t}}\brk*{\ell\ind{t}(\fbar\ind{t})}
 -\gamma\cdot \min_{\fun\in\FunSpace}\sum_{t=1}^{T}\ell\ind{t}(\fun).
\end{align*}

\end{proof}

\section{Proofs from \creftitle{sec:comp-results}}
\label{app:comp-results}

\subsection{Proofs from \creftitle{subsec:comp-hardness}}
\label{app:comp-hardness}

\complbimproper*

\begin{proof}[Proof of \refcomplbvar]
We frame the example proposed by \cite{blum1994separating} in their
Theorem 3.2 (see also \cite{bun2020computational}) in our setting. For
any integer $n\geq 1$, let the covariate space $\CovarSpace_n$ be
$\CovarSpace_n = \set{0,1}^n$, and set $\ValSpace =\ObsSpace =\set{0,1}$ and $\Kernel(\val)=\indic_\val$. We define a class $\FunSpace_n = \set{\fun_s:s\in \set{0,1}^{\sqrt{n}}}$ with
\begin{align*}
\fun_s(\covar) \ldef
\begin{cases}
1 & \text{if~}\covar \in c_s,\\
0 & \text{otherwise,}
\end{cases}
\end{align*}
for a certain collection of subsets $\crl*{c_s\in\cX_n}_{s\in\crl{0,1}^{\sqrt{n}}}$ defined in Definition 2 of
\cite{blum1994separating}, which is constructed based on cryptographic
functions using the assumption of existence
of one-way functions.
The precise definition will not be important. The properties we will use are:
\begin{enumerate}
\item The value $f_s(x)$ can be computed in $\poly(n)$ time for any $x\in\cX_n$.
\item For
  any polynomials $p(n),q(n)$, any (possibly randomized) online
  estimation algorithm \dfedit{(oracle-efficient or not)} which runs
  in time $p(n)$, and any time step $T\leq q(n)$, for 
  sufficiently large $n$ \jqedit{where $q(n) \ll 2^{\sqrt{n}}$},\footnote{\jqedit{The argument is essentially asymptotic, since the choice of $n$ is determined by the power of the one-way function.}}
  there exists $s\in \set{0,1}^{\sqrt{n}}$ and a sequence 
  $x_s^1 ,\ldots,x_s^{2^{\sqrt{n}}-1}, x_s^{2^{\sqrt{n}}}$ (the specific
  definition of this sequence can be found in
  \cite{blum1994separating}) such that the online estimation error
  \dfedit{under this sequence} when $\fstar=f_s$ is at least $T/4$ in expectation. Our lower bound construction for any oracle efficient online estimation algorithm with runtime bounded by $p(n)$ in time step bounded by $1\leq T\leq q(n)$ will choose the aforementioned covariate sequence as the covariates revealed with the aforementioned function as the true parameter, i.e., $x\ind{\tau} = x_s\ind{\tau}$ for $\tau\in [T]$ and $\fstar = f_s$.
\end{enumerate}

It is straightforward to see that the sequence
$(\cF_1,\cF_2...,\cF_n,...)$ admits polynomial description length as
claimed, since $\log |\cF_n| = \sqrt{n}$. %
We are left to verify that there is an offline oracle that achieves
$\offgan=0$ with $\poly(n)$-output description length, yet does not
provide any information not already available to the learner in the
setting of \citet{blum1994separating} (recall that in
the protocol of \citet{blum1994separating}, the learner gets to see
the covariates $x\ind{1}$,\dots,$x\ind{t}$ and the true labels
$y\ind{1}$,\dots,$y\ind{t-1}$ at time step $t$ before making their
prediction, but does not receive any other feedback).

Consider the following offlines oracle
$\Orc\ind{t}(x\ind{1},\dots,x\ind{t-1},y\ind{1},\dots,y\ind{t-1})$. The
oracle output $\fhat\ind{t}$ is a circuit that, on input $x$, compares
$x$ sequentially with $x\ind{1},\dots,x\ind{t}$. If $x$ is ever equal
to $x\ind{\tau}$ for some $\tau\in [t-1]$, the circuit will output
$y\ind{\tau}$. If $x$ is not equal to any $x\ind{\tau}$ for $\tau\in [t-1]$,
the circuit outputs $0$. Such a Boolean circuit can be constructed
with polynomial size in $n$ because each $x\ind{\tau}$ has length $n$ for
$\tau\in[t-1]$ and $t\leq T\leq  q(n)$ by assumption. It is easy to see
that this oracle achieves $\offgan=0$, yet does not provide any
additional information about \truefunName $\fstar$ beyond what is
available in the model of \citet{blum1994separating}. Combining all the above, we complete our lower bound proof.

Lastly, we observe that since the setting we consider is an instance of
noiseless binary classification, the classical halving algorithm
achieves an online estimation error bound of $\bigoh(\log|\cF_n|) =
\bigoh(\sqrt{n})$ \citep{cesa2006prediction}.

\end{proof}

  \subsection{Proofs from \creftitle{subsec:reductions}}
  \label{app:hellinger_reduction}

In this section, we prove \cref{thm:reduction-to-ODE} through four
layers of reductions through different variants of the online
estimation setting. In
\hyperref[app:subsec:reduction-to-ODE]{Appendix~\ref{app:subsec:reduction-to-ODE}},
we first introduce the relevant settings and the describe reductions
through them. We then combine these reductions to prove \cref{thm:reduction-to-ODE}. 
Finally, in \hyperref[app:subsec:reductions]{Appendix~\ref{app:subsec:reductions}}, we prove each of the four reduction results.

\subsubsection{Proof of \creftitle{thm:reduction-to-ODE}}
\label{app:subsec:reduction-to-ODE}

\RedToODE*

\begin{proof}[\pfref{thm:reduction-to-ODE}]%
\colt{The proof of \cref{thm:reduction-to-ODE} is algorithmic, and is based on several layers of reductions.
\begin{itemize}
\item First, using the scheme in \cref{sec:delayed}, we reduce the problem of oracle-efficient online estimation to delayed online learning with the loss function $\ell\ind{t-N}(f) =   \divX{\big}{\ftil\ind{t-N}(\covar\ind{t-N})}{f(\covar\ind{t-N})}$ defined in \cref{alg:reduction-to-delayed-OL}, where $\ftil\ind{t-N} = \frac{1}{N} \sum_{i=t-N+1}^{t} \fhat\ind{i}$ is an average of offline estimators and $N\in\bbN$ is a delay parameter.
\item Then, using a standard reduction \citep{weinberger2002delayed,mesterharm2007improving,joulani2013online,quanrud2015online}, we reduce the delayed online learning problem above to a sequence of $N$ non-delayed online learning problems, with the same sequence of loss functions; both this and the preceding step are computationally efficient.
\item To complete the reduction, we argue that the base algorithm can be used to solve the online learning problem above in an oracle-efficient fashion. To do this, we simulate interaction with the environment by sampling fictitious outcomes $y\ind{t}\sim{}\ftil\ind{t}(x\ind{t})$ from the averaged offline estimators and passing them into the base algorithm. We argue that the fictitious outcomes approximate the true outcomes well through a change-of-measure argument.
\end{itemize}
Combining the above, we conclude that given any base algorithm that efficiently performs online estimation with outcomes sampled from the \truefunName $\fstar$, we can efficiently construct a computationally efficient and oracle-efficient algorithm.
In more detail, we introduce four layers of reduction in reverse order from CDE to OEOE.
}
\arxiv{For this proof, we introduce four layers of reductions from
  OEOE to CDE. We introduce the reductions in reverse order, beginning
  with CDE and building up to OEOE.
}

\paragraph{Conditional
 Density Estimation with Reference Outcomes (CDEwRO)}

The bottom-most reduction we consider is from a setting we refer to as \emph{Conditional
 Density Estimation with Reference Outcomes (CDEwRO)} to
the (realizable) CDE setting. CDEwRO is similar to CDE, but with the
following difference. Instead of receiving \obsNames
$\obs\ind{1},\ldots,\obs\ind{T}$ sampled from the true model $\fstar(\covar\ind{t})$
directly, in CDEwRO, the outcome is sampled from a \emph{reference
  \funName} $\ftil\ind{t}(\covar\ind{t})$ which is guaranteed to be
close to $\fstar$ in a certain sense. Moreover, the covariates and the
reference parameters $\ftil\ind{1}$,\dots,$\ftil\ind{T}$ are selected
obliviously (i.e. the entire sequence is chosen by the adversary
before the online learning protocol begins).

\begin{algorithm}[htp]
\caption{Reduction from CDEwRO to CDE}
\label{alg:reduction-to-CDE}
\begin{algorithmic}[1]
\State\textbf{input:} Time $T\in\bbN$, base algorithm $\AlgCDE$.
\State Nature selects $T$ covariates
$\covar\ind{1},\dots,\covar\ind{T}$ along with the reference \funNames
$\ftil\ind{1},\dots,\ftil\ind{T}$. 
\For{$t=1,\dots,T$} 
\State \multiline{Learner predicts  $\fbar\ind{t} \sim\mu\ind{t} =  \AlgCDE(x\ind{1},\ldots,x\ind{t-1}, y\ind{1},\ldots,y\ind{t-1})$.}
\State Outcome $y\ind{t}\sim{}\ftil\ind{t}(\covar\ind{t})$ is sampled and revealed to the learner with the covariate $x\ind{t}$. 
\EndFor
\end{algorithmic}
\end{algorithm}

Our reduction from CDEwRO to CDE is given in
\cref{alg:reduction-to-CDE}. The main guarantee for this reduction is
as follows.

\begin{lemma}
\label{lem:reduction-to-CDE}
For any fixed $\zeta \geq 0$, suppose $\sum\limits_{t=1}^{T}\Dhels{\ftil\ind{t}(\covar\ind{t})}{\fstar(\covar\ind{t})} \leq \zeta$. Let 
\begin{equation}
  \label{eq:reduction-to-CDE}
  \RCDEwRO(T,\zeta) \ldef{} 3C_\cF \log V \cdot \zeta +  \RCDE(T)+ 2C_{\FunSpace}\cdot \log ( C_\cF T),
\end{equation}
where $\RCDE(T)$ is defined as in \cref{eq:online_base} by the assumption on $\AlgCDE$.
Then 
\cref{alg:reduction-to-CDE} achieves 
an expected online estimation error upper bound of
\[\sum\limits_{t=1}^{T}\En\brk*{\Dhels{\fbar\ind{t}(\covar\ind{t})}{\fstar(\covar\ind{t})}} \leq \RCDEwRO(T,\zeta)\] 
in the CDEwRO setting, and has runtime $\poly(\Time(\cF,T),T)$ .

\end{lemma}
The key technique in the proof of this lemma is a change of measure
argument based on Donsker-Varadhan \citep{polyanskiy2014lecture}.

\paragraph{Conditional Density Estimation with Reference \FunNames (CDEwRP)}

The next reduction in our stack is from a setting we refer to as
\emph{Conditional Density Estimation with Reference \FunNames
  (CDEwRP)} to the CDEwRO setting above. CDEwRP is identical to
CDEwRO, except that in the former setting, the learner directly
observes the reference \funName $\ftil\ind{t}$ instead of observing
$\obs\ind{t}\sim \ftil\ind{t}(x\ind{t})$ as in CDEwRO.

A second difference is that we allow the adversary in the CDEwRP
setting to be adaptive \dfedit{in the choice of $x\ind{t}$ and $\ftil\ind{t}$}, while our definition of the CDEwRO setting
only allows for oblivious adversaries. Thus, the reduction we consider
serves two purposes:
\begin{itemize}
\item Simulating the CDEwRO feedback model through sampling.
\item Reducing the adaptive adversary to an oblivious one.
\end{itemize}
The reduction from adaptive adversaries to oblivious follows and improves upon the result from \cite{gonen2019private}, and may be of independent interest.

\begin{algorithm}[htp]
\caption{Reduction from CDEwRP to CDEwRO}
\label{alg:reduction-to-CDEwRO}
\begin{algorithmic}[1]
\State\textbf{input:} Time $T\in\bbN$, accuracy parameter $\veps>0$, algorithm $\AlgCDEwRO$.
\State Let  $L = O(T\log (T|\FunSpace||\CovarSpace|)/\veps) $.
\For{$t=1,\dots,T$}

\For{$i=1,\dots,L$}
\For{$s=1,\dots,t-1$}
\State Learner samples $\obs_i\ind{s,t}\sim \ftil\ind{s}(\covar\ind{s})$.
\EndFor
\State Learner computes $\fbar_i\ind{t} = \AlgCDEwRO(T;\covar_i\ind{1,t},\dots,\covar_i\ind{t-1,t},\obs_i\ind{1,t},\dots,\obs_i\ind{t-1,t})$.
\EndFor
\State Learner predicts via $\fbar\ind{t}\sim  \mu\ind{t} =  \unif(\set{\fbar_i\ind{t}}_{i\in [L]})$.
\State Nature selects and reveals the covariate $x\ind{t}$ and the reference \funName $\ftil\ind{t}$ based on
$\mu\ind{t}$. 
\EndFor
\end{algorithmic}
\end{algorithm}

Our reduction from CDEwRP to CDEwRO is displayed in
\cref{alg:reduction-to-CDE}, and takes as input an algorithm
$\AlgCDEwRO (\cdot;\cdot)$ for the CDEwRO setting, where $\AlgCDEwRO\ind{t}(T;\cdot)$ denotes the algorithm's output at round $t\leq{}T$ as a function of the history. The main guarantee
for the algorithm is as follows.

\begin{lemma}
  \label{lem:reduction-to-CDEwRO}
  For any fixed $\zeta \geq 0$, suppose $\sum\limits_{t=1}^{T}\Dhels{\ftil\ind{t}(\covar\ind{t})}{\fstar(\covar\ind{t})} \leq \zeta$. Let 
  \begin{equation}
    \label{eq:reduction-to-CDEwRO}
    \RCDEwRP(T,\zeta,\veps) \ldef{} 2 \RCDEwRO(T,\zeta) + \veps,
  \end{equation}
  where $\RCDEwRO(T,\zeta)$ is defined as in \cref{eq:reduction-to-CDE}.
  Then
  \cref{alg:reduction-to-CDEwRO} with parameter $\veps>0$ achieves 
  an expected online estimation error upper bound of
  \[
    \sum\limits_{t=1}^{T}\En\brk*{\Dhels{\fbar\ind{t}(\covar\ind{t})}{\fstar(\covar\ind{t})}} \leq \RCDEwRP(T,\zeta,\veps) \] 
  in the CDEwRP setting, and runs in time $\poly(\Time(\cF,T),T,\log|\cF|,\log|\cX|,1/\veps)$ .
  The \arxiv{randomization }distributions $\mu\ind{1},\ldots,\mu\ind{T}$ produced by the algorithm have support size $\poly(\log|\FunSpace|,\log|\CovarSpace|, T,1/\veps)$.

\end{lemma}

\paragraph{Conditional Density Estimation with Delayed Reference \FunNames (CDEwDRP)}

Our next reduction is from a setting we refer to as \emph{Conditional
  Density Estimation with Delayed Reference \FunNames (CDEwDRP)} to
the CDEwRP setting. CDEwDRP is identical to CDEwRP, except that the
reference function $\ftil\ind{t}$ is revealed only at round
$t+N$ instead of at round $t$, for a delay parameter $N\in\bbN$.

\begin{algorithm}[htp]
\caption{Reduction from CDEwDRP to CDEwRP}
\label{alg:reduction-to-CDEwRP}
\begin{algorithmic}[1]
\State\textbf{input:} Time $T\in\bbN$, delay time $N\in\bbN$, algorithm $\AlgCDEwRP$.
\State Initialize $N$ copies of the algorithm $\AlgCDEwRP$ as $\AlgCDEwRP^1,\dots,\AlgCDEwRP^N$.
\For{$t=1,\dots,T$}
\State \multiline{Learner predicts $\fbar\ind{t}\sim\mu\ind{t}=
  \AlgCDEwRP^i(T/N,1/N;\covar\ind{i},\covar\ind{i+N},\dots,\covar\ind{t-N},
  \ftil\ind{i},\ftil\ind{i+N},\dots,\ftil\ind{t-N})$
  where $i\equiv t \mod N$.}
\State Nature selects and reveals the covariate $x\ind{t}$ and the reference \funName $\ftil\ind{t-N}$ based on
$\mu\ind{t}$. 
\EndFor
\end{algorithmic}
\end{algorithm}

Our reduction from CDEwDRP to CDEwRP is displayed in
\cref{alg:reduction-to-CDEwRO}, and takes as input an algorithm
$\AlgCDEwRP(\cdot;\cdot)$ for the CDEwRP setting, where $\AlgCDEwRP\ind{t}(T,\veps;\cdot)$ denotes the algorithm's output at round $t\leq{}T$ with accuracy parameter $\veps>0$ (cf. \cref{alg:reduction-to-CDEwRO}), as a function of the history. The main guarantee
for the algorithm is as follows.

\begin{lemma}
\label{lem:reduction-to-CDEwRP}
For any fixed $\zeta \geq 0$, suppose $\sum\limits_{t=1}^{T}\Dhels{\ftil\ind{t}(\covar\ind{t})}{\fstar(\covar\ind{t})} \leq \zeta$. Let 
\begin{equation}
  \label{eq:reduction-to-CDEwRP}
  \RCDEwDRP(T,N,\zeta) \ldef \sup\limits_{\zeta_1,...,\zeta_N\geq 0, \sum\limits_{i=1}^{N}\zeta_i \leq \zeta} \sum\limits_{i=1}^{N}\RCDEwRP(T/N,\zeta_i, 1/N),
\end{equation}
where $\RCDEwRP(T/N,\zeta_i, 1/N)$ is defined as in \cref{eq:reduction-to-CDEwRO}.
Then
\cref{alg:reduction-to-CDEwRP} achieves 
an expected online estimation error upper bound of
\[
  \sum\limits_{t=1}^{T}\En\brk*{\Dhels{\fbar\ind{t}(\covar\ind{t})}{\fstar(\covar\ind{t})}} \leq \RCDEwDRP(T,N,\zeta) \] 
in the CDEwDRP setting, and has runtime $\poly(\Time(\cF,T),T,\log|\cF|,\log|\cX|)$.
The \arxiv{randomization }distributions $\mu\ind{1},\ldots,\mu\ind{T}$ produced by the algorithm have support size $\poly(\log|\FunSpace|,\log|\CovarSpace|, T)$.

\end{lemma}

\paragraph{Oracle-Efficient Online Estimation (OEOE)}

\begin{algorithm}[htp]
\caption{Reduction from OEOE to CDEwDRP}
\label{alg:reduction-to-CDEwDRP}
\begin{algorithmic}[1]
  \State\textbf{input:} Time $T\in\bbN$, offline estimation oracle $\Orc$ with
  parameter $\offgan\geq{}0$, delay parameter $N\in\bbN$, CDEwDRP algorithm $\AlgCDEwDRP$.
  \For{$t=1,\dots,T$}
  \State \dfedit{Receive $\fhat\ind{t}=\Orc\ind{t}(x\ind{1},\ldots,x\ind{t-1},y\ind{1},\ldots,y\ind{t-1})$.}
\State Learner computes reference \funName $\ftil\ind{t-N} = \frac{1}{N}\sum\limits_{i=t-N+1}^{t}\fhat\ind{i}$. 
\State \multiline{Learner predicts $\fbar\ind{t} \sim\mu\ind{t} = 
  \AlgCDEwDRP(T,N;\covar\ind{1},\covar\ind{2},\dots,\covar\ind{t},
  \ftil\ind{1},\ftil\ind{2},\dots,\ftil\ind{t-N})$.}
\State Nature selects and reveals the covariate $x\ind{t}$ based on
$\mu\ind{t}$.
\EndFor
\end{algorithmic}
\end{algorithm}

Our final reduction reduces the Oracle-Efficient Online Estimation
setting (OEOE) to the CDEwDRP setting described above. This reduction,
which is displayed in \cref{alg:reduction-to-CDEwRP}, is a variant of
the approach used in \cref{thm:delayed_reduction}. The reduction takes
as input a CDEwDRP algorithm $\AlgCDEwDRP(\cdot;\cdot)$, where
$\AlgCDEwDRP\ind{t}(T,N;\cdot)$ denotes the algorithm's output at
round $t\leq{}T$ with delay parameter $N$, as a function of the history.

\begin{lemma}
\label{lem:reduction-to-CDEwDRP}
\cref{alg:reduction-to-CDEwDRP} achieves 
an expected online estimation error upper bound of
\[
  \sum\limits_{t=1}^{T}\En\brk*{\Dhels{\fbar\ind{t}(\covar\ind{t})}{\fstar(\covar\ind{t})}} \leq \RCDEwDRP(T,N,N+\offgan T/N)\] 
 in the OEOE setting, and has runtime $\poly(\Time(\cF,T),T,\log|\cF|,\log|\cX|)$.
The \arxiv{randomization }distributions $\mu\ind{1},\ldots,\mu\ind{T}$ produced by the algorithm have support size $\poly(\log|\FunSpace|,\log|\CovarSpace|, T)$.

\end{lemma}

\paragraph{Completing the proof of \cref{thm:reduction-to-ODE}}
To prove \cref{thm:reduction-to-ODE}, we compose all of the preceding
reductions, with $N$ left as a free parameter temporarily. 
We first apply \cref{lem:reduction-to-CDEwDRP} to reduce from the OEOE setting to the CDEwDRP setting, with the guarantee that
\begin{align*}
  \sum\limits_{t=1}^{T}\En\brk*{\Dhels{\fbar\ind{t}(\covar\ind{t})}{\fstar(\covar\ind{t})}}  \leq \RCDEwDRP(T,N,N+\offgan T/N).
\end{align*}
Then by \cref{lem:reduction-to-CDEwRP}, we can reduce the CDEwDRP setting to the CDEwRP setting, with the guarantee by \cref{eq:reduction-to-CDEwRP} that 
\begin{align*}
  \RCDEwDRP(T,N,N+\offgan T/N) = \sup\limits_{\sum\limits_{i=1}^{N}\zeta_i \leq N+\offgan T/N} \sum\limits_{i=1}^{N}\RCDEwRP(T/N,\zeta_i,1/N).
\end{align*}
Then apply \cref{lem:reduction-to-CDEwRO} $N$ times with $T$, $\zeta$, and $\veps$ in the Lemma chosen to be $T/N$, $\zeta_i$, and $1/N$ respectively for each $i\in [N]$, we can reduce the CDEwRP setting to the CDEwRO setting  with guarantee by \cref{eq:reduction-to-CDEwRO} that
\begin{align*}
  \sup\limits_{\sum\limits_{i=1}^{N}\zeta_i \leq N+\offgan T/N} \sum\limits_{i=1}^{N}\RCDEwRP(T/N,\zeta_i,1/N) =   1 + \sup\limits_{\sum\limits_{i=1}^{N}\zeta_i \leq N+\offgan T/N} 2\sum\limits_{i=1}^{N}\RCDEwRO(T/N,\zeta_i).
\end{align*}
Consequently, apply \cref{lem:reduction-to-CDE} $N$ times with $T$ and $\zeta$ in the Lemma chosen to be $T/N$ and $\zeta_i$ for each $i\in [N]$, we can reduce the CDEwRO setting to the CDE setting  with guarantee by \cref{eq:reduction-to-CDE} that
\begin{align}
  \hspace{.7in}&\hspace{-.7in}\sup\limits_{\sum\limits_{i=1}^{N}\zeta_i \leq N+\offgan T/N} \sum\limits_{i=1}^{N}\RCDEwRO(T/N,\zeta_i)  \notag\\
  &\leq \sup\limits_{\sum\limits_{i=1}^{N}\zeta_i \leq N+\offgan T/N} \sum\limits_{i=1}^{N} \prn*{3C_{\FunSpace} \log V \cdot \zeta_i +  \RCDE(T)+ 2C_{\FunSpace}\cdot \log ( C_{\FunSpace} T)}\notag\\
  &\lesssim  C_{\FunSpace}\log V\cdot\prn*{N+ \frac{\offgan T}{N}} +  N \cdot\prn*{ \RCDE(T)+ C_{\FunSpace}\cdot \log ( C_{\FunSpace} T)}\notag \\
  &=  C_{\FunSpace}\log V\cdot\frac{\offgan T}{N} +  N \cdot\prn*{ \RCDE(T)+ C_{\FunSpace}\cdot \log ( VC_{\FunSpace} T)}. \label{ineq:bound}
\end{align}

By choosing $N= \sqrt{ \frac{C_{\FunSpace}\offgan T\cdot\log V}{\RCDE(T)+ C_{\FunSpace}\cdot \log (VC_{\FunSpace}T)} }\vee 1$, we can bound the expression in \eqref{ineq:bound} as
\begin{align*}
\En\brk*{\EstHel(T,\offgan)}  &\lesssim \sqrt{ C_{\FunSpace}\offgan T \prn*{ \RCDE(T)+ C_{\FunSpace}\cdot \log(VC_{\FunSpace} T)} \log V } +  \RCDE(T)+ C_{\FunSpace}\cdot \log (VC_{\cF}T)\\
&\leq \sqrt{C_{\FunSpace} \offgan T  \RCDE(T)\log V } + C_{\FunSpace}\sqrt{ \offgan T\log(V C_{\FunSpace}T) \log V} \\
&\qquad\qquad+ \RCDE(T)+ C_{\FunSpace}\cdot \log(VC_{\FunSpace} T).
\end{align*}

Finally, under the assumption that $\RCDE(T) \leq  C_{\FunSpace}' \log
T$, the bound above can be further simplified as
\begin{align*}
\En\brk*{\EstHel(T,\offgan)}   &\lesssim (C_{\FunSpace}(C_{\FunSpace} + C_{\FunSpace}') \offgan \log V \log (VC_{\FunSpace}T))^{1/2} T^{1/2} + (C_{\FunSpace} + C_{\FunSpace}') \log (VC_{\FunSpace}T).
\end{align*}
\end{proof}

\subsubsection{Proofs for Supporting Lemmas}
\label{app:subsec:reductions}

\begin{proof}[\pfref{lem:reduction-to-CDE}]%
For \cref{alg:reduction-to-CDE}, denote the randomness of the sequence
$(\obs\ind{1:T},\mu\ind{1:T})$ under
$\fstar\in \FunSpace$ and $(\ftil\ind{1:T})$ by $\bbP^{\fstar}$ and
$\bbP^{\ftil\ind{1:T}}$ respectively. The data generating process for
$(\covar\ind{1:T},\obs\ind{1:T})$ in the CDEwRO setting implies that
\begin{align*}
  \En_{\bbP^{\ftil\ind{1:T}}} \brk*{\sum\limits_{t=1}^{T}\En_{\fbar\ind{t}\sim\mu\ind{t}} \brk*{ \Dhels{\fbar\ind{t}(\covar\ind{t})}{\fstar(\covar\ind{t})}}} = \En\brk*{ \sum\limits_{t=1}^{T}\En_{\fbar\ind{t}\sim\mu\ind{t}} \brk*{ \Dhels{\fbar\ind{t}(\covar\ind{t})}{\fstar(\covar\ind{t})}} }.
\end{align*}
By Donsker-Varadhan \citep{polyanskiy2014lecture}, we have that
for all $\eta>0$
\begin{align}
  \frac{1}{\eta}\Dkl{\bbP^{\ftil\ind{1:T}}}{\bbP^{\fstar}} \hspace{.5in}&\hspace{-.5in}\geq \En_{\bbP^{\ftil\ind{1:T}}} \brk*{\sum\limits_{t=1}^{T} \En_{\fbar\ind{t}\sim\mu\ind{t}} \brk*{\Dhels{\fbar\ind{t}(\covar\ind{t})}{\fstar(\covar\ind{t})}}} \notag\\
&- \frac{1}{\eta}\log  \En_{\bbP^{\fstar}} \exp\crl*{\eta\sum\limits_{t=1}^{T} \En_{\fbar\ind{t}\sim\mu\ind{t}} \brk*{\Dhels{\fbar\ind{t}(\covar\ind{t})}{\fstar(\covar\ind{t})}}}. \label{ineq:DV}
\end{align}  
For any random variables $X,Y,Z$, we denote by $\Dkl{\bbP_X}{\bbP_Y\mid{}Z} = \En_{Z}[\Dkl{\bbP_{X|Z}}{\bbP_{Y|Z}}]$.
We further note that by the chain rule for KL divergence, 
\begin{align*}
\Dkl{\bbP^{\ftil\ind{1:T}}}{\bbP^{\fstar}} 
&= \sum\limits_{t=1}^{T}\En_{\bbP^{\ftil\ind{1:T}}}\brk*{\Dkl{\bbP^{\ftil\ind{1:T}}_{(\covar\ind{t},\obs\ind{t})}  }{\bbP^{\fstar}_{(\covar\ind{t},\obs\ind{t})}\mid{} \covar\ind{1:t-1}, \obs\ind{1:t-1} }}\\
&= \sum\limits_{t=1}^{T}\Dkl{\ftil\ind{t}(\covar\ind{t})}{\fstar(\covar\ind{t}) },
\end{align*}
where the second equality 
holds because $\obs\ind{t}$ follows $\ftil\ind{t}(\covar\ind{t})$ and
$\fstar(\covar\ind{t})$ respectively, and because the conditional
distribution of $x\ind{t}$ is identical under both laws \jqedit{due to the oblivious assumption of the covariates in the setting of CDEwRO}.
Then by the relation between KL and Hellinger (Lemma A.10 of \cite{foster2021statistical}) and that $1\leq \log V$, we have
\begin{align*}
  \sum\limits_{t=1}^{T}\Dkl{\ftil\ind{t}(\covar\ind{t})}{\fstar(\covar\ind{t}) } \leq (2+\log V)\cdot \sum\limits_{t=1}^{T}\Dhels{\ftil\ind{t}(\covar\ind{t})}{\fstar(\covar\ind{t}) } \leq 3\log V\cdot \sum\limits_{t=1}^{T}\Dhels{\ftil\ind{t}(\covar\ind{t})}{\fstar(\covar\ind{t}) },
\end{align*}
where the last inequality is by $V\geq e$.
Combining all of the results so far and using
\eqref{ineq:DV}, we have 
\begin{align*}
  \En\brk*{ \sum\limits_{t=1}^{T}\En_{\fbar\ind{t}\sim\mu\ind{t}} \brk*{ \Dhels{\fbar\ind{t}(\covar\ind{t})}{\fstar(\covar\ind{t})}} } &\leq  \frac{3\log V}{\eta}\cdot \sum\limits_{t=1}^{T}\Dhels{\ftil\ind{t}(\covar\ind{t})}{\fstar(\covar\ind{t}) } \\
  &\qquad +\frac{1}{\eta} \log  \En_{\bbP^{\fstar}} \exp\crl*{\eta\sum\limits_{t=1}^{T} \En_{\fbar\ind{t}\sim\mu\ind{t}} \brk*{\Dhels{\fbar\ind{t}(\covar\ind{t})}{\fstar(\covar\ind{t})}}}.
\end{align*}
To proceed, using that for any postive random variable $X$, $\En[X] = \int_{0}^{\infty} \bbP(X\geq t)\d t$, we have
\begin{align*}
  \hspace{.5in}&\hspace{-.5in}\frac{1}{\eta}\log  \En_{\bbP^{\fstar}} \exp\crl*{\eta\sum\limits_{t=1}^{T} \En_{\fbar\ind{t}\sim\mu\ind{t}} \brk*{\Dhels{\fbar\ind{t}(\covar\ind{t})}{\fstar(\covar\ind{t})}}} \\
  &=    \RCDE(T) + \frac{1}{\eta}\log  \En_{\bbP^{\fstar}} \exp\crl*{\eta \prn*{\sum\limits_{t=1}^{T} \En_{\fbar\ind{t}\sim\mu\ind{t}} \brk*{\Dhels{\fbar\ind{t}(\covar\ind{t})}{\fstar(\covar\ind{t})}}-\RCDE(T) }}\\
  &\leq \RCDE(T) + \frac{1}{\eta}\log  \En_{\bbP^{\fstar}} \exp\crl*{\eta \prn*{\sum\limits_{t=1}^{T} \En_{\fbar\ind{t}\sim\mu\ind{t}} \brk*{\Dhels{\fbar\ind{t}(\covar\ind{t})}{\fstar(\covar\ind{t})}}-\RCDE(T) }_+}\\
  &=  \RCDE(T) +  \frac{1}{\eta}\log  \int_0^\infty \bbP^{\fstar} \prn*{ \prn*{\sum\limits_{t=1}^{T} \En_{\fbar\ind{t}\sim\mu\ind{t}} \brk*{\Dhels{\fbar\ind{t}(\covar\ind{t})}{\fstar(\covar\ind{t})}} - \RCDE(T)}_+ \geq \frac{1}{\eta} \log t  } \d t.
\end{align*}
Recall the assumption
$\sum\limits_{t=1}^{T}\Dhels{\ftil\ind{t}(\covar\ind{t})}{\fstar(\covar\ind{t})}
\leq \zeta$ and let $\eta = \frac{1}{C_\cF}$. We have
\begin{align*}
  &\frac{3\log V}{\eta}\cdot \sum\limits_{t=1}^{T} \Dhels{\ftil\ind{t}(\covar\ind{t})}{\fstar(\covar\ind{t}) }  + \frac{1}{\eta}\log  \En_{\bbP^{\fstar}} \exp\crl*{\eta\sum\limits_{t=1}^{T} \En_{\fbar\ind{t}\sim\mu\ind{t}} \brk*{\Dhels{\fbar\ind{t}(\covar\ind{t})}{\fstar(\covar\ind{t})}}}\\
  &\leq   \frac{3\log V}{\eta}\cdot\zeta +   \RCDE(T)  \\
  &\qquad\qquad+ \frac{1}{\eta}\log  \int_0^\infty \bbP^{\fstar} \prn*{ \prn*{\sum\limits_{t=1}^{T} \En_{\fbar\ind{t}\sim\mu\ind{t}} \brk*{\Dhels{\fbar\ind{t}(\covar\ind{t})}{\fstar(\covar\ind{t})}} - \RCDE(T)}_+ \geq \frac{1}{\eta} \log t  } \d t\\
  &\leq 3C_{\FunSpace}\log V\cdot\zeta +   \RCDE(T)+ C_{\FunSpace} \log \prn*{ 1+ \int_1^{e^{C_\cF \cdot T}}   \frac{1}{t} \d t}\\
  &\leq 3C_\cF \log V \cdot \zeta +  \RCDE(T)+ 2C_{\FunSpace}\cdot \log ( C_\cF T).
\end{align*}
where the second inequality uses the assumption \cref{eq:online_base} on the algorithm.

\end{proof}

\begin{proof}[\pfref{lem:reduction-to-CDEwRO}]%
  \newcommand{\Ent}{\En^t}%
  Our result improves uses the proof technique from the
  adversarial-to-oblivious reduction in Lemma 11 of
  \citet{gonen2019private}, but improves the result by a $\bigoh(\log
  T)$ factor.

  Consider the CDEwRP setting. Let $\Ent[\cdot] \ldef \En[\cdot  \mid  \covar\ind{1:t-1},\ftil\ind{1:t-1}]$.
Let $\mu\ind{s,t}\ldef{} \En^t [\AlgCDEwRO(\obs_1\ind{1,t},\dots,\obs_1\ind{s-1,t}, \covar_1\ind{1,t},\dots,\covar_1\ind{s-1,t})]$ 
for all $1\leq s\leq t\leq T$ where the expectation is taken over all the random variables $\obs_1\ind{1,t},\dots,\obs_1\ind{s-1,t}, \covar_1\ind{1,t},\dots,\covar_1\ind{s-1,t}$. 

Then by Bernstein's concentration inequality applied to $\mu\ind{t}$
(interpreted as an empirical approximation to $\mu\ind{t,t}$), conditioned on
$\covar\ind{1:t-1},\ftil\ind{1:t-1}$, we have with probability at
least $1-\frac{\veps}{2T}$, for all $\covar'\in\CovarSpace$ and
$\fun'\in \FunSpace$,
\begin{align}
  \Ent_{f\sim \mu\ind{t}}\brk*{ \Dhels{f(\covar)}{\fun'(\covar)} } \leq 2 \Ent_{f\sim \mu\ind{t,t}}\brk*{ \Dhels{f(\covar)}{\fun'(\covar)} } + \veps/(2T). \label{ineq:concentration}
\end{align}

For any fixed $t\in [T]$ and all  $t'$ such that $t\leq t' \leq  T$,
the different trajectories $\covar_1\ind{1,t'},\dots,\covar_1\ind{t-1,t'},
\obs_1\ind{1,t'},\dots,\obs_1\ind{t-1,t'}$ are i.i.d. conditioned on $\covar\ind{1:t-1},\ftil\ind{1:t-1}$. Thus, we have
\begin{align}
  \label{eq:adv_expectation}
  \En\brk*{\En_{f\sim \mu\ind{t,t}}^t\brk*{ \Dhels{f(\covar\ind{t})}{\fstar(\covar\ind{t})} }} &=\En\brk*{\En_{f\sim \mu\ind{t,t+1}}^{t+1}\brk*{ \Dhels{f(\covar^t)}{\fstar(\covar^t)} }} \\ 
&=\dots\\
&= \En \brk*{ \En_{f\sim \mu\ind{t,T}}^{T}\brk*{ \Dhels{f(\covar\ind{t})}{\fstar(\covar\ind{t})} } }.
\end{align}

Finally, for \cref{alg:reduction-to-CDEwRO}, we have by the guarantee of
$\AlgCDEwRO$, 
\begin{align*}
  \sum\limits_{t=1}^{T} \En \brk*{\En_{f\sim \mu\ind{t,T}}^T\brk*{ \Dhels{f(\covar\ind{t})}{\fstar(\covar\ind{t})} }} \leq  \RCDEwRO(T,\zeta).
\end{align*}

Combining the three results above, we have
\begin{align*}
  \sum\limits_{t=1}^{T} \En \brk*{  \Ent_{f\sim \mu\ind{t}}\brk*{ \Dhels{f(\covar\ind{t})}{\fun(\covar\ind{t})} } }&\leq 2\sum\limits_{t=1}^{T} \En\brk*{\En_{f\sim \mu\ind{t,t}}^t\brk*{ \Dhels{f(\covar\ind{t})}{\fstar(\covar\ind{t})} }}  +  \veps\\
 &= 2\sum\limits_{t=1}^{T} \En\brk*{\En_{f\sim \mu\ind{t,T}}^T\brk*{ \Dhels{f(\covar\ind{t})}{\fstar(\covar\ind{t})} }}  +  \veps \\
&\leq 2\RCDEwRO(T,\zeta) + \veps,
\end{align*}
where the first equality is by \cref{ineq:concentration}, the second
equality is from \cref{eq:adv_expectation}, and the final inequality is by the guarantee of $\AlgCDEwRO$.

\end{proof}

\begin{proof}[\pfref{lem:reduction-to-CDEwRP}]
  Note that \cref{alg:reduction-to-CDEwRP} is a variant of the
  reduction in \cref{lem:delayed-to-ol}, specialized to squared
  Hellinger distance, and the proof here will use the same idea as \cref{lem:delayed-to-ol}.

For each $i\in\brk{N}$, let $\zeta_i = \sum_{j=1}^{T/N}
\Dhels{\ftil\ind{i+N\cdot j}(\covar\ind{i+N\cdot
    j})}{\fstar(\covar\ind{i+N\cdot j}) }  $. Then by the guarantee of
$\AlgCDEwRP$, we that for all $i\in [N]$, 
\begin{align*}
  \sum_{j=1}^{T/N}\En_{\fbar\sim \mu\ind{i+N\cdot j}}\brk*{\Dhels{\fbar(\covar\ind{i+N\cdot j})}{\fstar(\covar\ind{i+N\cdot j}) }} \leq  \RCDEwRP(T/N,\zeta_i,1/N),
\end{align*}
Summing up over all $i\in [N]$, we obtain
\begin{align*}
  \sum\limits_{t=1}^{T} \En_{\fbar\sim \mu\ind{t}}\brk*{\Dhels{\fbar(\covar\ind{t})}{\fstar(\covar\ind{t}) }} &=
  \sum\limits_{i=1}^{N}\sum_{j=1}^{T/N}\En_{\fbar\sim \mu\ind{i+N\cdot j}}\brk*{\Dhels{\fbar(\covar\ind{i+N\cdot j})}{\fstar(\covar\ind{i+N\cdot j}) }} \\
  &\leq \sum\limits_{i=1}^{N}\RCDEwRP(T/N,\zeta_i,1/N).
\end{align*}
By the assumption on $\sum\limits_{t=1}^{T}\Dhels{\ftil\ind{t}(\covar\ind{t})}{\fstar(\covar\ind{t}) }$, we have
\begin{align*}
\sum\limits_{i=1}^{N} \zeta_i  &= \sum\limits_{i=1}^{N}\sum_{j=1}^{T/N} \Dhels{\ftil\ind{i+N\cdot j}(\covar\ind{i+N\cdot j})}{\fstar(\covar\ind{i+N\cdot j}) } 
  = \sum\limits_{t=1}^{T}\Dhels{\ftil\ind{t}(\covar\ind{t})}{\fstar(\covar\ind{t}) }\leq \zeta.
\end{align*}
Finally, we conclude
\begin{align*}
  \sum\limits_{t=1}^{T} \En_{\fbar\sim \mu\ind{t}}\brk*{\Dhels{\fbar(\covar\ind{t})}{\fstar(\covar\ind{t}) }} \leq  \sup\limits_{\sum\limits_{i=1}^{N}\zeta_i \leq \zeta} \sum\limits_{i=1}^{N}\RCDEwRP(T/N,\zeta_i,1/N).
\end{align*}

\end{proof}

\begin{proof}[\pfref{lem:reduction-to-CDEwDRP}]
This reduction is arguably the most interesting one. This reduction is by noticing that averaging across the outputs of the offline oracle will generate reference \funNames (although delayed) with small online estimation errors as shown later in \cref{ineq:sliding-window-guarantee}.
By the guarantee of $\AlgCDEwDRP$, we have
\begin{align*}
  \sum\limits_{t=1}^{T} \En_{\fbar\sim \mu\ind{t}}\brk*{\Dhels{\fbar(\covar\ind{t})}{\fstar(\covar\ind{t}) }} \leq  \RCDEwDRP(T,N,\zeta),
\end{align*}
where $\zeta$ can be chosen to be any upper bound of $\sum\limits_{t=1}^{T} \Dhels{\ftil\ind{t}(\covar\ind{t})}{\fstar(\covar\ind{t}) }$ since it is unknown to the learner in the setup where \jqedit{we augment the sequence of $\fhat\ind{1}$,\dots,$\fhat\ind{T}$ by setting $\fhat\ind{T+s} = \fhat\ind{T}$ for all $s\in \bbN$ and define $\ftil\ind{t}\ldef{}\frac{1}{N}\sum_{i=t+1}^{t+N}\fhat\ind{i}$ for $t=T-N,T-N+1,\dots,T$.} 
Furthermore, by the definition of $\ftil\ind{t}$, we can obtain
\begin{align}
  \sum\limits_{t=1}^{T} \Dhels{\ftil\ind{t}(\covar\ind{t})}{\fstar(\covar\ind{t}) } &= N+ \frac{1}{N}\sum_{t=1}^{T-N}\sum_{i=t+1}^{t+N}  \Dhels{\fbar\ind{i}(\covar\ind{t})}{\fstar(\covar\ind{t}) } \notag\\
  &= N+ \frac{1}{N}\sum_{t=2}^{T}\sum_{i<t}\Dhels{\fbar\ind{t}(\covar\ind{i})}{\fstar(\covar\ind{i}) }\notag\\
  &\leq N+ \frac{\offgan T}{N}. \label{ineq:sliding-window-guarantee}
\end{align}
Thus, we conclude \cref{alg:reduction-to-CDEwDRP} obtains
\begin{align*}
  \sum\limits_{t=1}^{T} \En_{\fbar\sim \mu\ind{t}}\brk*{\Dhels{\fbar(\covar\ind{t})}{\fstar(\covar\ind{t}) }} \leq  \RCDEwDRP(T,N,N+\offgan T/N).
\end{align*}

\end{proof}

\section{Proofs from \creftitle{sec:application}}
\label{app:application-proof}

\newcommand{\Pmstar}[1][\Mstar]{\wb{P}^\sss{#1}}

\coverabilityupper*

\begin{proof}[\pfref{thm:coverability-mdp}]
  This proof closely follows the proof of Theorem 1 in
  \citet{xie2022role}. Define the shorthand
$d_h\ind{t}(s,a) \equiv  d_h^{\act\ind{t}}(s,a)$ \dfedit{and
  $\Ccov=\Ccov(\Mstar)$}, and define
\begin{align*}
\dtil_h\ind{t}(s,a) \ldef \sum\limits_{\tau=1}^{t-1} d_h\ind{\tau}(s,a),\quad \text{and}\quad \mustar_h \ldef \argmin_{\mu_h\in \Delta(\cS\times\cA)} \sup_{\act\in \Pi} \nrm*{\frac{d_h^\act}{\mu_h}}_\infty.
\end{align*}

From the definition of the layer-wise loss, we have
\begin{align*}
  \sum\limits_{t=1}^{T}\divRL{\Mhat\ind{t}(\act\ind{t})}{\Mstar(\act\ind{t})}  =  \sum\limits_{h=1}^{H} \sum\limits_{t=1}^{T} \sum\limits_{(s,a)\in\cS\times\cA}d_h\ind{t}(s,a) \divh{\Pmstar_h(s,a)}{\Pmstar[\Mhat\ind{t}]_h(s,a)}.
\end{align*}

We define a ``burn-in'' phase for each state-action pair $(s,a)\in \cS\times\cA$ by defining
\begin{align*}
\tau_h(s,a) = \min \set{t\mid{} \dtil_h\ind{t}(s,a)\geq C_\cov\cdot \mustar_h(s,a)}.
\end{align*}
Let $h\in\brk{H}$ be  With this definition, for we can write
\begin{align*}
  &\sum\limits_{t=1}^{T}
  \sum\limits_{(s,a)\in\cS\times\cA}d_h\ind{t}(s,a)
  \divh{\Pmstar_h(s,a)}{\Pmstar[\Mhat\ind{t}]_h(s,a)}\\
  &= 
  \sum\limits_{(s,a)\in\cS\times\cA}\sum\limits_{t<\tau_h(s,a)} d_h\ind{t}(s,a)
  \divh{\Pmstar_h(s,a)}{\Pmstar[\Mhat\ind{t}]_h(s,a)}
  + \sum\limits_{(s,a)\in\cS\times\cA}\sum_{t\geq{}\tau_h(s,a)}d_h\ind{t}(s,a) \divh{\Pmstar_h(s,a)}{\Pmstar[\Mhat\ind{t}]_h(s,a)}
\end{align*}
For the first term above, we have
\begin{align*}
\sum\limits_{(s,a)\in\cS\times\cA}\sum\limits_{t<\tau_h(s,a)} d_h\ind{t}(s,a)\divh{\Pmstar_h(s,a)}{\Pmstar[\Mhat\ind{t}]_h(s,a)}  \leq \sum\limits_{(s,a)\in\cS\times\cA}\dtil_h\ind{\tau_h(s,a)}(s,a) \leq 2 C_\cov \sum\limits_{(s,a)\in\cS\times\cA}\mustar_h(s,a) = 2C_\cov,
\end{align*}
where the last inequality holds because
\begin{align*}
  \sum\limits_{(s,a)\in\cS\times\cA}\dtil_h\ind{\tau_h(s,a)}(s,a) =  \sum\limits_{(s,a)\in\cS\times\cA}\dtil_h\ind{\tau_h(s,a)-1}(s,a) + d_h\ind{\tau_h(s,a)}(s,a) \leq 2C_\cov\cdot \mustar_h(s,a).
\end{align*}
The remaining term is 
\begin{align*}
  &\sum\limits_{h=1}^{H}\sum\limits_{(s,a)\in\cS\times\cA}\sum\limits_{t\geq\tau_h(s,a)} d_h\ind{t}(s,a)\divh{\Pmstar_h(s,a)}{\Pmstar[\Mhat\ind{t}]_h(s,a)} \\
  &= \sum\limits_{h=1}^{H} \sum\limits_{t=1}^{T} \sum\limits_{(s,a)\in\cS\times\cA}d_h\ind{t}(s,a) \prn*{ \frac{\dtil_h\ind{t}(s,a)}{\dtil_h\ind{t}(s,a)}  }^{1/2} \mathbbm{1}\crl*{t\geq \tau_h(s,a)} \divh{\Pmstar_h(s,a)}{\Pmstar[\Mhat\ind{t}]_h(s,a)} \\
  &\leq \sqrt{ \sum\limits_{h=1}^{H}\sum\limits_{t=1}^{T} \sum\limits_{(s,a)\in\cS\times\cA}  \frac{( \mathbbm{1}(t\geq \tau_h(s,a)) d_h\ind{t}(s,a)  )^2}{\dtil_h\ind{t}(s,a)}   } \cdot \sqrt{\sum\limits_{h=1}^{H}\sum\limits_{t=1}^{T} \sum\limits_{(s,a)\in\cS\times\cA}\dtil_h\ind{t}(s,a) \divh{\Pmstar_h(s,a)}{\Pmstar[\Mhat\ind{t}]_h(s,a)}  }.
\end{align*}
Following the derivation in Theorem 1 of \citet{xie2022role}, we have
\begin{align*}
  \sum\limits_{t=1}^{T} \sum\limits_{(s,a)\in\cS\times\cA}  \frac{( \mathbbm{1}(t\geq \tau_h(s,a)) d_h\ind{t}(s,a)  )^2}{\dtil_h\ind{t}(s,a)} &\leq 2 \sum\limits_{t=1}^{T} \sum\limits_{(s,a)\in\cS\times\cA} \frac{d_h\ind{t}(s,a)\cdot d_h\ind{t}(s,a)}{\dtil_h\ind{t}(s,a) + C_\cov \cdot  \mustar_h(s,a)}  \\
  &\leq 2 \sum\limits_{t=1}^{T} \sum\limits_{(s,a)\in\cS\times\cA}\max_{t'\in [T]} d_h\ind{t'}(s,a)\cdot \frac{ d_h\ind{t}(s,a)}{\dtil_h\ind{t}(s,a) + C_\cov \cdot  \mustar_h(s,a)} \\
  &\leq 2\prn*{\max_{s,a}\sum\limits_{t=1}^{T}\frac{ d_h\ind{t}(s,a)}{\dtil_h\ind{t}(s,a) + C_\cov \cdot  \mustar_h(s,a)} }\cdot \prn*{\sum\limits_{(s,a)\in\cS\times\cA} \max_{t\in [T]} d_h\ind{t}(s,a)}\\
  &\lesssim C_\cov \log T,
\end{align*}
where the last inequality follows from Lemmas 3 and 4 of \citet{xie2022role}.
For the second term, as a consequence of the offline estimation
assumption, we have that for all $t\in [T]$,
\begin{align*}
  \sum\limits_{h=1}^{H} \sum\limits_{(s,a)\in\cS\times\cA}\dtil_h\ind{t}(s,a) \divh{\Pmstar_h(s,a)}{\Pmstar[\Mhat\ind{t}]_h(s,a)} = \sum\limits_{s=1}^{t-1}\divRL{\Mhat\ind{t}(\act\ind{s})}{\Mstar(\act\ind{s})}  \leq \offgan.
\end{align*}

Altogether, we conclude that
\begin{align*}
  \sum\limits_{t=1}^{T}  \divRL{\Mhat\ind{t}(\decn\ind{t})}{\Mstar(\decn\ind{t})} \leq \bigoh\prn*{\sqrt{HC_\cov \offgan T \log T} + HC_\cov}.
\end{align*}

\end{proof}

\coverabilitymdp*

\begin{proof}[\pfref{cor:coverability-mdp-version}]
By \pref{lem:hellinger_chain_rule}, 
for any two MDP models $M$ and $M'$ and any  $\act\in \Delta(\PiRNS)$,
we have 
\begin{align}
  \Dhels{M(\act)}{M'(\act)} = \Dhels{M'(\act)}{M(\act)}\leq 7 \cdot \sum\limits_{h=1}^{H} \En^{\sM',\pi}\brk*{\Dhels{\Pmstar[M']_h(s_h,a_h)}{\Pmstar[M]_h(s_h,a_h)}}.\label{eq:hel1}
\end{align} 
On the other hand, for any $h\in [H]$, we have from Lemma A.9 of
\cite{foster2021statistical} that
\begin{align*}
  \En^{\sM',\pi}\brk*{\Dhels{\Pmstar[M']_h(s_h,a_h)}{\Pmstar[M]_h(s_h,a_h)}}
  &\leq 4 \Dhels{M(\act)}{M'(\act)},
\end{align*}
by choosing $X=(s_h,a_h)$ and \dfedit{$Y=(r_h,s_{h+1})$} in the
aforementioned lemma. Summing up over $h$, we conclude that
\begin{align}
  \sum\limits_{h=1}^{H}  \En^{\sM',\pi}\brk*{\Dhels{\Pmstar[M']_h(s_h,a_h)}{\Pmstar[M]_h(s_h,a_h)}}\leq 4H\cdot{} \Dhels{M(\act)}{M'(\act)}.\label{eq:hel2}
\end{align}

Consider any sequence of policies $\act\ind{1},\dots,\act\ind{T}$ and
outputs $\Mhat\ind{1},\dots,\Mhat\ind{T}$ from an offline oracle with 
parameter $\offgan$ for squared Hellinger distance. By \cref{eq:hel2},
we have that for all $t\in\brk{T}$,
\begin{align*}
  \sum\limits_{\tau=1}^{t-1}\sum\limits_{h=1}^{H}\En^{\sMstar,\act\ind{\tau}} \brk*{\Dhels{\Pmstar_h(s,a)}{\Pmstar[\Mhat\ind{\tau}]_h(s,a)}} 
  \leq 4H \sum\limits_{\tau=1}^{t-1} \Dhels{\Mhat\ind{\tau}(\act\ind{\tau})}{\Mstar(\act\ind{\tau})} \leq 4\offgan H.
\end{align*}
\pref{thm:coverability-mdp} thus implies that
\begin{align*}
\sum\limits_{t=1}^{T} \sum\limits_{h=1}^{H} \En^{\sMstar,\act\ind{t}} 
\brk*{\Dhels{\Pmstar_h(s,a)}{\Pmstar[\Mhat\ind{t}]_h(s,a)}} \lesssim H\sqrt{C_\cov \offgan T \log T} + H^2 C_\cov.
\end{align*}
Finally, using \cref{eq:hel1}, we can convert the inequality above
into a bound on the square Hellinger distance:
\begin{align*}
  \sum\limits_{t=1}^{T} \Dhels{\Mhat\ind{t}(\act\ind{t})}{\Mstar(\act\ind{t})} &\lesssim \sum\limits_{t=1}^{T} \sum\limits_{h=1}^{H}\En^{\sMstar,\act\ind{t}}  \brk*{\Dhels{\Pmstar_h(s,a)}{\Pmstar[\Mhat\ind{t}]_h(s,a)}} \\
  &\lesssim  H\sqrt{C_\cov \offgan T \log T  }+H^2C_\cov.
\end{align*}

\end{proof}

\coverabilitylower*

\begin{proof}[\pfref{thm:coverability-lower-bound}]
  Let $\offgan>0$ and $T$ be given, and let $N\in\bbN$ be chosen such that $T=N\cdot \lfloor N\offgan \rfloor$; we assume without loss of generality that $N$ is large enough such
that $N\offgan > 1$, which implies that $\lfloor N\offgan \rfloor \geq
N\offgan/2$). Define $\cA = \set{a_0,a_1}$, $\cS  = \set{s_0,\dots,s_{N-1}}$, and $\gmstar(s,a) \equiv 0$ for all $s\in \cS, a\in \cA$. 
Let $d_1= \unif(\cS)$ be the context distribution. For any $t\in\set{1,\dots,T}$, consider the deterministic policy $\act\ind{t}:\cS\to\cA$ and the offline estimator $\Mhat\ind{t}$ defined via
\begin{align*}
\act\ind{t}(s_n) = 
\begin{cases}
a_1, \quad\text{if }n = \lfloor t/ \lfloor N\offgan\rfloor\rfloor ,\\
a_0, \quad\text{otherwise.}
\end{cases}
\text{~and~~}
g^{\sss{\Mhat\ind{t}}} (s,a) =
\begin{cases}
1, & \text{if }s = s_{\lfloor t/\lfloor N\offgan\rfloor\rfloor }, a = a_1,\\
0, & \text{otherwise}.
\end{cases}
\end{align*} 
We first verify that $\Mhat\ind{1},\ldots,\Mhat\ind{T}$ satisfy the
offline oracle requirement. For any $t\in \set{1,\dots,T}$, we have
\begin{align*}
  &\sum\limits_{\tau=1}^{t-1} \En_{s\sim d_1}\brk*{ \Dsq{\gm(s,\pi^\tau(s))}{g^{\sss{\Mhat\ind{\tau}}}(s,\pi^\tau(s))}} \\
  &= \frac{1}{N} \sum\limits_{\tau=1}^{t-1} \mathbbm{1}\crl*{\act\ind{\tau}(s_{\lfloor t/N\rfloor })= a_1}\\
  &= \frac{1}{N}\cdot (t-\lfloor t/\lfloor N\offgan\rfloor\rfloor \cdot \lfloor N\offgan\rfloor ) \leq \offgan.
\end{align*}
However, the online error is 
\begin{align*}
  \sum\limits_{t=1}^{T}\En_{s\sim d_1}\brk*{ \Dsq{\gm(s,\pi^t(s))}{g^{\sss{\Mhat\ind{t}}}(s,\pi^t(s))}} 
  = \sum\limits_{t=1}^{T} \frac{1}{N} =  \lfloor N\offgan \rfloor \geq \frac{1}{2} \sqrt{T\offgan}.
\end{align*}

\end{proof}

\end{document}